\documentclass{article}

\usepackage{lineno,hyperref}
\modulolinenumbers[1]

\usepackage{mathtools}

\usepackage{pifont}
\usepackage{geometry}
\usepackage{fleqn}
\usepackage{hyperref}
\usepackage{amsfonts}

\usepackage{xspace}

\usepackage{latexsym}
\usepackage{amsmath}
\usepackage{amsthm}
\usepackage{amssymb}
\usepackage{amsfonts}
\usepackage{amstext}
\usepackage{mathrsfs} 
\usepackage{bbm} 
\usepackage{color}
\usepackage[svgnames]{xcolor}
\usepackage{graphicx}
\usepackage{tikz}
\usetikzlibrary{shapes.misc,automata,positioning,arrows}
\usepackage{dsfont}
\usepackage{url}
\usepackage{stmaryrd}

\usepackage[ruled, linesnumbered, vlined]{algorithm2e}

\usepackage{tikz}

\usetikzlibrary{arrows,decorations,decorations.markings,decorations.pathmorphing,backgrounds,positioning,fit,matrix}

\tikzset{
  strike through/.style={
    postaction=decorate,
    decoration={
      markings,
      mark=at position 0.5 with {
        \draw[-] (-3pt,-3pt) -- (3pt, 3pt);
      }
    }
  }
}

\tikzset{
  middleneg/.style={
    postaction=decorate,
    decoration={
      markings,
      mark=at position 0.5 with {
        \coordinate [label=below:$\neg$] (a) at (0,0); ;
      }L
    }
  }
}

\newcommand{\ie}{\mbox{i.e.}}
\newcommand{\eg}{\mbox{e.g.}}

\newcommand{\wrt}{\mbox{w.r.t.}}
\newcommand{\st}{\mbox{s.t.}}

\newcommand{\tuple}[1]{\mbox{$\langle #1 \rangle$}}

\newtheorem{definition}{Definition}
\newtheorem{theorem}{Theorem}
\newtheorem{lemma}[theorem]{Lemma}
\newtheorem{proposition}[theorem]{Proposition}
\newtheorem{corollary}[theorem]{Corollary}
\newtheorem{example}{Example}
\newtheorem{remark}{Remark}

\newcommand{\T}{\ensuremath{\mathcal{T}}} 
\newcommand{\KB}{\ensuremath{\mathcal{K}}} 
\newcommand{\A}{\ensuremath{\mathcal{A}}} 

\DeclareRobustCommand{\sat}{\models}

\newcommand{\entails}{\models}

\newcommand{\R}{\ensuremath{\mathfrak{R}}} 

\newcommand{\K}{{\sf K}}

\newcommand{\N}{\ensuremath{{\sf N}}} 
\newcommand{\CN}{\ensuremath{{\sf N}_{\mathscr{C}}}} 
\newcommand{\RN}{\ensuremath{{\sf N}_{\mathscr{R}}}} 
\newcommand{\Names}{\ensuremath{{\N_{\I}}}}  
\newcommand{\ON}{\ensuremath{{\sf N}_{\mathscr{O}}}} 
\newcommand{\dlAnd}{\sqcap}
\newcommand{\dlOr}{\sqcup}
\newcommand{\subs}{\sqsubseteq}

\newcommand{\I}{\ensuremath{\mathcal{I}}}

\newcommand{\ALC}{\ensuremath{\mathcal{ALC}}}
\newcommand{\alc}{\ensuremath{\mathcal{ALC}}}

\newcommand{\EL}{\ensuremath{\mathcal{EL}}}

\newcommand{\aboxent}{\models_{rc}}
\newcommand{\aboxentsec}{\models_{rc}'}

\newcommand{\usually}{\raisebox{-0.5ex}{$\:\stackrel{\sqsubset}{{\scriptstyle\sim}}\:$} }
\newcommand{\nusually}{\not\hspace{-1.7mm}\usually}

\newcommand{\rk}{\ensuremath{\text{\it rk}}} 

\newcommand{\dsrel}{\usually}

\newcommand{\D}{\ensuremath{\mathcal{D}}} 
\newcommand{\E}{\ensuremath{\mathcal{E}}}

\newcommand{\hide}[1]{}

\newif\iftodo
\todotrue  

\newlength{\RoundedBoxWidth}
\newsavebox{\WhiteRoundedBox}
   {\setlength{\RoundedBoxWidth}{\dimexpr#1}
    \begin{lrbox}{\WhiteRoundedBox}
       \begin{minipage}{\RoundedBoxWidth}}%
   {   \end{minipage}
    \end{lrbox}
    \begin{center}
    \begin{tikzpicture}%
       \draw node[draw=Orange,fill=white,rounded corners,line width=1.5pt,%
             inner sep=0.7ex,text width=\RoundedBoxWidth]%
             {\usebox{\WhiteRoundedBox}};
    \end{tikzpicture}
    \end{center}}

\newcommand{\ratent}{\vdash_{rc}}

\newcommand{\cass}[2]{\mbox{$#1$:$#2$}}
\newcommand{\rass}[3]{\mbox{$(#1,#2)$:$#3$}}

\newcommand{\ELbot}{\ensuremath{\mathcal{E}\mathcal{L}_\bot}}

\newcommand{\ELObot}{\ensuremath{\mathcal{E}\mathcal{L}\mathcal{O}_\bot}}
\newcommand{\defELObot}{\ensuremath{\mathcal{EL\defnom{O}}_\bot}}

\newcommand{\dl}[1] {{\sf  #1}}
\newcommand{\nd}{\noindent}

\usepackage{xcolor}

\newcommand{\ru}{R^\cup_\KB}
\newcommand{\rub}{R^\cup_{\KB^\bullet}}

\newcommand{\defnom}[1]{\mbox{$\langle #1 \rangle$}}
\newcommand{\typ}{\bullet}

\usepackage{authblk}

\title{A Polynomial Time Subsumption Algorithm for Nominal Safe \ELObot~under Rational Closure\footnote{This is a preprint version of a paper accepted for publication in \emph{Information Sciences}, DOI: \url{https://doi.org/10.1016/j.ins.2018.09.037}}}

\author[1]{Giovanni Casini}
\author[2]{Umberto Straccia}
\author[3]{Thomas Meyer}

\affil[1]{CSC, University of Luxembourg, Luxembourg. Email: \emph{{giovanni.casini@uni.lu}}}
\affil[2]{ISTI - CNR, Italy. Email: \emph{{umberto.straccia@isti.cnr.it}}}
\affil[3]{University of Cape Town and CAIR, South Africa. Email: \emph{{tmeyer@cs.uct.ac.za}}}

\date{}
\begin{document}
\maketitle

%
%
%
%
%
%
%

\begin{abstract}
Description Logics (DLs) under Rational Closure (RC) is a well-known framework for non-monotonic reasoning in DLs. 
In this paper, we address the concept subsumption decision problem under RC for  nominal safe \ELObot, a notable and practically important DL representative of the OWL 2 profile OWL 2 EL.

Our contribution here is to define a polynomial time subsumption procedure for nominal safe \ELObot~under RC that relies entirely on a series of classical, monotonic \ELbot~subsumption tests. Therefore, any existing classical monotonic  \ELbot~reasoner can be used as a black box to implement our method. 
We then also adapt the method to one of the known extensions of RC for DLs, namely Defeasible Inheritance-based DLs without losing the computational tractability.

\end{abstract}

\section{Introduction}

\nd Description logics~(DLs) provide the logical foundation of formal ontologies of the OWL family.\footnote{\url{https://www.w3.org/TR/owl2-overview/}}
Among the various extensions proposed to enhance the representational capabilities of DLs, endowing them with non-monotonic features is still a main issue, as documented by the past 20~years of technical development 
(see \eg~\cite{Bonatti15,Eiter11,Giordano13,Motik10} and references therein, and Section~\ref{related}).

We recall that a typical problem that can be addressed using non-monotonic formalisms is reasoning with ontologies in which some classes are exceptional \wrt~some properties of their super classes.

\begin{example}\label{ex01}
We know that  
\emph{avian red blood cells}, \emph{mammalian red blood cells}, and hence also \emph{bovine red blood cells} are \emph{vertebrate red blood cells}, and that vertebrate red blood cells normally \emph{have a cell membrane}. We also know that vertebrate red blood cells normally \emph{have a nucleus}, but that mammalian red blood cells normally don't. 
\qed
\end{example}
\nd A classical formalisation of the ontology above would imply that mammalian red blood cells do not exist, since, being a subclass of vertebrate red blood cells, they would have a nucleus, but in the meantime,  they are an atypical subclass that does not have a nucleus. Therefore, mammalian red blood cells would and would not have a nucleus at the same time. Unlike a classical approach,  the use of a non-monotonic formalism may allow us to deal with such exceptional classes.

Among the various proposals to inject non-monotonicity into DLs, the preferential approach has recently gained  
attention~\cite{BritzEtAl2008,BritzEtAl2015a,CasiniStraccia2010,CasiniStraccia13,CasiniEtAl2014,CasiniStraccia14,CasiniEtAl2015,Giordano2009,GiordanoEtAl2012b,Giordano18a,Giordano15,Giordano18,Moodley15}  as it is based on one of the major frameworks for non-monotonic reasoning in the propositional case, namely the {\em KLM approach}~\cite{KrausEtAl1990}. One of the main constructions in the preferential approach is \emph{Rational Closure} (RC) \cite{LehmannMagidor1992}. RC has some interesting properties: the conclusions are intuitive, and the decision procedure can be reduced to a series of classical decision problems, sometimes preserving the computational complexity of the underlying classical decision problem. 

In this paper, we address the concept subsumption decision problem under RC for  nominal safe \ELObot~\cite{Kazakov14}, a computationally tractable and practically important DL representative of the OWL 2 profile OWL 2 EL.\footnote{\url{https://www.w3.org/TR/owl2-profiles/}} In fact, 
(i) nominal safe  \ELObot~is the language $\EL$~\cite{Baader05a}, extended with the bottom concept (denoted $\bot$) and with the so-called \emph{ObjectHasValue} construct (also denoted as $\mathcal{B}$ in the DL literature), which is an existential quantification. Roughly, nominal safe \ELObot~is as \ELObot, except for the fact that a nominal may only occur in concept expressions of the form $\exists r.\{a\}$, or in inclusion axioms of the form $\{a\} \subs C$, stating that individual $a$ is an instance of concept $C$ (note that so-called \emph{role assertions} can be also expressed via $\{a\} \subs \exists r.\{b\}$); and (ii) in~\cite{Kazakov14} it is shown that many  OWL EL ontologies are nominal safe and that an \ELbot~reasoner is sufficient to decide the subsumption problem, decreasing the inference time significantly in practice.

In summary, our contributions are as follows.
\begin{enumerate}
\item We describe a subsumption decision procedure  under RC for nominal safe \ELObot~that runs in polynomial time.
%
Al feature of our approach is that our procedure relies entirely on a series of classical, monotonic \ELbot~subsumption tests, and, thus, any existing \ELbot~reasoner can be used as a black box to implement our method. Note that \eg, in~\cite{BritzEtAl2015a} it is shown that the use of a DL reasoners as a black box for RC under \ALC~is scalable in practice. We conjecture that this property holds with respect to RC under nominal safe \ELObot~as well.

\item We will also illustrate how to adapt our procedure to a relevant modification of RC for DLs, namely Defeasible Inheritance-based DLs~\cite{CasiniStraccia13}. We recall that Defeasible Inheritance-based DLs have been introduced to overcome  some inference limitations of RC~\cite{CasiniStraccia2010}: in fact, in~\cite[Appendix A]{CasiniStraccia13} it is shown that  Defeasible Inheritance-based DLs behave  better than RC \wrt~most of the ``benchmark'' examples illustrated there. A feature of our  proposed procedure is that it runs in polynomial time and maintains the advantage of the previous point.
\end{enumerate}
\nd In the following, we will proceed as follows: for the sake of completeness, Section~\ref{prel} introduces nominal safe \ELObot~and  recaps salient notions about Rational Closure for \alc~\cite{BritzEtAl2015a,CasiniStraccia2010,Moodley15}; 
Section~\ref{ratclosEL} we describe a polynomial time procedure to decide subsumption under RC for defeasible  \ELbot; in Section~\ref{extend}, we adapt our procedure for the Defeasible Inheritance-based \ELObot. In Section~\ref{rcelo} we address nominal safe \ELObot~and show how we polynomially reduce reasoning within it to defeasible  \ELbot, and thus inherit the computational complexity of reasoning of the subsumption decision problem from defeasible \ELbot~ and its RC extensions. Section~\ref{related} discusses related work and  Section~\ref{conclusions} concludes and addresses future work. 

All relevant proofs are in the appendixes.\footnote{A preliminary version of some results in Sections \ref{ratclosEL} and \ref{extend} appear in the Technical Report~\cite{CasiniEtAl2015TR}.}

%
\section{Preliminaries}\label{prel}
%
\nd To make the paper self-contained, in the following we briefly present the DLs \ELbot, \ELObot, and nominal safe \ELObot, and an exponential time procedure to decide subsumption in the  DL \ALC~under RC via series of classical \ALC~subsumption tests. The latter is important here as we will adapt it to decide  the subsumption problem for nominal safe \ELObot~in polynomial time.

\subsection{The DLs \ELbot, \ELObot, and nominal safe \ELObot}\label{EL}

\nd \ELbot~is the DL \EL~with the addition of the empty concept $\bot$~\cite{Baader05a}. It is a proper sublanguage of \alc. Note that considering  \EL~alone would not make sense in our case as \EL~ontologies are always concept-satisfiable, while the notion of defeasible reasoning is built over a notion of conflict (see Example \ref{ex01}) which needs to be expressible in the language.

\ELObot~is \ELbot~extended with so-called nominal concepts (denoted with the letter $\mathcal{O}$ in the DL literature, while nominal safe \ELObot~is \ELObot~with some restrictions on the occurrence of nominals.

\paragraph{Syntax.} The vocabulary is given by a set of \emph{atomic concepts} $\CN=\{A_1,\dots, A_n\}$, a set of \emph{atomic roles} $\RN=\{r_1,\ldots, r_m\}$ and a set of \emph{individuals}  $\ON=\{a,b, c,\dots\}$.
All these sets are assumed to be finite. \ELObot~\emph{concept expressions} $C,D,\ldots$ are built according to the following syntax:
\[
C,D \rightarrow A\mid\top\mid\bot\mid C\dlAnd D\mid \exists  r.C \mid \{a\} \ .
\]
\nd An \emph{ontology} $\T$  (or \emph{TBox}, or \emph{knowledge base})
is a finite set of \emph{Generalised Concept Inclusion} (GCI) axioms $C\subs D$ ($C$ \emph{is subsumed by} $D$), meaning that all the objects in the concept $C$ are also in the concept $D$. We use the expression $C = D$ as  shorthand for having both $C \subs D$ and $D \subs C$.

\paragraph{The DL \ELbot.} \nd A concept of the form $\{a\}$ is called a \emph{nominal}. \ELbot~is \ELObot~without nominals.

\paragraph{The DL nominal safe \ELObot.} Nominal safe \ELObot~is \ELObot~with some restrictions on the occurrence of nominals and is defined as follows~\cite{Kazakov14}. An \ELObot~concept $C$ is \emph{safe} if $C$ has only occurrences of nominals in subconcepts of the form $\exists r.\{a\}$;
 $C$ is \emph{negatively safe} (in short, \emph{n-safe}) if $C$ is either safe or a nominal. A GCI $C \subs D$ is \emph{safe} if $C$ is n-safe and $D$ is safe. An \ELObot~ontology is  \emph{nominal safe} if all its GCIs are safe. It is worth noting that nominal safeness is a quite commonly used pattern of nominals in OWL EL ontologies, as reported in~\cite{Kazakov14}.
%
\paragraph{Semantics.} 
An \emph{interpretation} is a pair $\I=\tuple{\Delta^\I,\cdot^\I}$, where $\Delta^\I$ is a non-empty set, called \emph{interpretation domain} and $\cdot^\I$ is an \emph{interpretation function} that
\begin{enumerate}
\item maps atomic concepts $A$ into a set $A^\I \subseteq  \Delta^\I$;
\item maps $\top$ (resp.~$\bot$) into a set  $\top^\I =  \Delta^\I$ (resp.~$\bot^\I =  \emptyset$);
\item maps roles $r$ into a set  $r^\I \subseteq  \Delta^\I \times \Delta^\I $;
\item maps individuals $a\in\ON$ into an object $a^\I \in \Delta^\I$.
\end{enumerate}
\nd The interpretation function $\cdot^\I$ is extended to complex concept expressions as follows:
\begin{eqnarray*}
(C \dlAnd D)^\I & = & C^\I \cap D^\I \\
(\exists r.C)^\I & = & \{ o \in \Delta^\I \mid \exists o' \in \Delta^\I \mbox{ s.~t. } \tuple{o, o'} \in r^\I \mbox{ and } o' \in C^\I \} \\
\{a\}^\I & = & \{a^\I \} \ .
\end{eqnarray*}
\nd An interpretation $\I$ \emph{satisfies (is a model of)} $C\subs D$ if $C^\I \subseteq D^\I$, denoted $\I\sat C\subs D$.
$\I$ \emph{satisfies  (is a model of)} an ontology $\T$  if it satisfies each axiom in it. An axiom $\alpha$ is \emph{entailed} by a $\T$ if every model of $\T$ is a model of $\alpha$, denoted as $\T \models \alpha$.
\begin{remark}\label{aboxA}
Note that a so-called \emph{concept assertion} $\cass{a}{C}$ ($a$ is an instance of concept $C$) and a \emph{role assertion} $\rass{a}{b}{r}$ ($a$ and $b$ are related via role $r$) can easily be represented in nominal safe \ELObot~via the mapping $\cass{a}{C} \mapsto \{a\} \subs C$ and $\rass{a}{b}{r} \mapsto \{a\} \subs \exists r.\{b\}$. 
\qed
\end{remark}
\nd In the following, we recap here some salient facts related to nominal safe \ELObot~\cite[Appendix A]{Kazakov14}, which we will use once we present our entailment decision algorithm for nominal safe \ELObot. Specifically,  we can replace nominals in a nominal safe \ELObot~ontology $\T$ with newly introduced concept names, yielding an $\ELbot$~ontology $\T'$, such that $\T'$ supports the same entailments as $\T$ \cite{Kazakov14}. Hence, an entailment decision procedure for $\ELbot$~suffices to decide entailment for nominal safe \ELObot~(but not for unrestricted \ELObot).

Consider an \ELObot~ontology $\T$. For each individual $a$ occurring in $\T$ consider a new atomic concept $N_a$. For $X$ an \ELObot~concept, GCI, or ontology, we define $N(X)$ to be the result of replacing each occurrence of each nominal $\{a\}$ in $X$ with $N_a$.\label{defN} The following proposition provides  a sufficient condition to check entailment.
\begin{proposition}[\cite{Kazakov14}, Lemma 5 and Corollary 2]\label{propelo}
Let $\T$ be an \ELObot ontology and  $\alpha$ an \ELObot~axiom that do not contain atomic concepts of the form $N_a$. Then
\begin{enumerate}
\item if $N(\T) \models N(\alpha)$ then $\T \models \alpha$;
\item if $N(\T) \models N_a \subs \bot$ for some $a$ then $\T$ is not satisfiable.\qed
\end{enumerate}
\end{proposition}
\nd 
The converse of Proposition~\ref{propelo} does not hold in general but holds for nominal safe  \ELObot.
\begin{proposition}[\cite{Kazakov14}, Theorem 4]\label{propeloB}
Let $\T$ be a nominal safe \ELObot ontology and  $\alpha$ a safe \ELObot~axiom that do not contain atomic concepts of the form $N_a$. Then
\begin{enumerate}
\item if $N(\T) \not\models N_a \subs \bot$ for all $a$ then  $\T$ is satisfiable;  
\item if $\T \models \alpha$ then $N(\T) \models N(\alpha)$.\qed
\end{enumerate}
\end{proposition}
\nd Note that Proposition~\ref{propeloB} fails if the use of nominals is not safe.
\begin{example}[\cite{Kazakov14}, Remark 2] \label{nosafe}
Consider 
\[
\T = \{A \subs \{a\}, B \subs \{a \}, A \subs \exists r.B \} \ .
\]
\nd It is easily verified that $\T$ is satisfiable and that $\T \models A \subs B$. However, for
\[
N(\T) = \{A \subs N_a, B \subs N_a, A \subs \exists r.B \} 
\]
\nd we have that $N(\T)$ is satisfiable, $N(A\subs B) = A \subs B$, but $N(\T) \not \models A \subs B$.
\qed
\end{example}
%
\subsection{Rational Closure in \ALC}\label{alc}
\nd We briefly recap RC for the DL \ALC~(see, \eg~\cite{BritzEtAl2015a}), which in turn is based on its original formulation for Propositional Logic \cite{LehmannMagidor1992}. 

The DL \ALC~is the DL $\ELbot$ extended with concept negation, \ie~concept expressions of the form $\neg C$ and semantics
$(\neg C)^\I  =  \Delta^\I \setminus C^\I$. Note that by using the negation $\neg$ and the conjunction $\dlAnd$ we can introduce also, \eg~the disjunction $\dlOr$, \ie~$C \dlOr D$ is a macro for $\neg (\neg C \dlAnd \neg D$), that is, it is interpreted as $(C \dlOr D)^\I  = C^\I \cup D^\I$.

A \emph{defeasible GCI} axiom is of the form $C\usually D$, that is read as `Typically, an instance of $C$ is also an instance of $D$'. We extend ontologies  with a $DBox$ $\D$, \ie~a finite set of defeasible GCIs and denote an ontology as $\KB=\tuple{\T,\D}$, where $\T$ is a TBox and $\D$ is a DBox.
\begin{example}\label{ex02}
We can  formalise the information in Example \ref{ex01} in \ELbot~and, thus, in \alc~with the following ontology $\KB=\tuple{\T,\D}$, with\footnote{The acronyms stand for: 
\dl{BRBC} - Bovine Red Blood Cells;
\dl{MRBC} - Mammalian Red Blood Cells;
\dl{ARBC} - Avian Red Blood Cells;
\dl{VRBC} - Vertebrate Red Blood Cells;
\dl{hasN} - has Nucleus;
\dl{hasCM} - has Cell Membrane.
}

\begin{eqnarray*}
\T  & = \{ &\dl{BRBC}\subs\dl{MRBC}, \\
&& \dl{ARBC}\subs\dl{VRBC}, \\
&& \dl{MRBC}\subs\dl{VRBC}, \\
&& \exists\dl{hasN}.\top\dlAnd \dl{NotN}\subs\bot \ \ \} \\ \\
 \D & = \{ & \dl{VRBC}\usually\exists\dl{hasCM}.\top,\\
 &&  \dl{VRBC}\usually\exists\dl{hasN}.\top,\\
 && \dl{MRBC}\usually \dl{NotN} \  \ \} \ . 
\end{eqnarray*}\qed
\end{example}
\nd Given a KB $\KB=\tuple{\T,\D}$,  RC satisfies some basic desiderata:  the axioms in $\T$ and $\D$ are included into the set of the derivable axioms, that moreover  is closed under  the following  properties.
\[
\begin{array}{ccccc}
(\text{Ref})&C \dsrel C&&(\text{\small LLE})&{\displaystyle \frac{\entails C= D,\ C \dsrel E}{D \dsrel E}}\\
\\
(\text{And})&{\displaystyle \frac{C \dsrel D,\ C \dsrel E}{C \dsrel D \dlAnd E}}&&(\text{Or})&{\displaystyle \frac{C \dsrel E,\ D \dsrel E}{C \dlOr D \dsrel E}}\\
\\
(\text{\small RW})&{\displaystyle \frac{C \dsrel D,\ \entails D \subs E}{C \dsrel E}}&&(\text{\small CM})&{\displaystyle \frac{C \dsrel D,\ C \dsrel E}{C\dlAnd D \dsrel E}}\\
\\
\text{({\small RM})}&{\displaystyle \frac{C\dsrel E, \ C\nusually\lnot D}{C\dlAnd D\dsrel E}}
\end{array}
\]
\nd Reflexivity (Ref), Left Logical Equivalence (LLE), Right Conjunction (And), Left Disjunction (Or), and Right Weakening (RW) are all properties that correspond to well-known properties of the classical subsumption relation $\subs$. Cautious Monotonicity (CM) and Rational Monotonicity (RM) are constrained forms of Monotonicity that are useful and desirable in modelling defeasible reasoning. (CM) guarantees that our inferences are cumulative, that is, whatever we can conclude about typical $C$s (\eg~that they are in $D$), we can add such information to $C$ ($C\dlAnd D$) and still derive all the information associated to typical $C$s ($C\dlAnd D \dsrel E$). The stronger principle (RM) is necessary to model the principle of \emph{presumption of typicality}, that is,
if  the typical elements of a class $C$ satisfy certain properties (\eg~$E$) and we are  not informed that the typical elements of $C$ do  not satisfy the properties of $D$ ($C\nusually \lnot D$), then we can assume that the typical elements of  $C\dlAnd D$ satisfy all the properties characterising the typical elements of $C$ ($C\dlAnd D\dsrel E$). 
%
%
We refer to~\cite{KrausEtAl1990,Makinson1994a}  for a deeper explanation of the meaning of such properties and why they are desirable for modelling defeasible reasoning.

RC is a form of inferential closure that satisfies all the  properties above; it is based on the semantic notion of ranked interpretation and on the directly connected notion of ranked entailment, which we illustrate next.
\begin{definition}[Ranked interpretation]\label{rankedint}
A \emph{ranked interpretation} is a triple $R=\tuple{\Delta^R,\cdot^R,\prec^R}$, where $\Delta^R$ and $\cdot^R$ are as in the classical DL interpretations, while $\prec^R$ is a modular preference relation over $\Delta^R$, that is, a strict partial order satisfying the following property:
\begin{description}
%
\item[\emph{Modularity}:] $\prec^R$ is modular if and only if there is a ranking function $\rk:\Delta^R\longrightarrow \mathbb{N}$ \st\
for every $o,p\in \Delta^R$, $o\prec^R p$ iff $\rk(o)< \rk(p)$. \qed
\end{description}
\end{definition}
%
%
%
%
%
%
\nd In the definition above, $o\prec^R p$ means that the object $o$ is considered  more typical than the object $p$. 
The order $\prec^R$ allows us to partition the domain $\Delta^R$ of a ranked interpretation $R$ into a sequence of \emph{layers},
$\tuple{L^R_0,\ldots,L^R_n,\dots} \ ,\label{layers}$ where for every object $o$, $o\in L^R_0$ iff $o\in\min_{\prec^R}(\Delta^R)$ and $o\in L^R_{i+1}$ iff $o\in\min_{\prec^R}(\Delta^R\setminus\bigcup_{0\leq j\leq i} L^R_j)$.\footnote{Given a set $X$ and the  order $\prec$ defined over $X$, $\min_{\prec}(X)=\{x\in X\mid \not\exists y\in X\ \text{\st~} y\prec x\}$} From this partition, we can define the \emph{height} of an individual $a$ as
\[
h_R(a)=i\text{ iff }a^R\in L^R_i \ .
\]
The lower the height, the more typical  the individual in the interpretation is taken to be. We can also extend this to a level of typicality for  concepts: the \emph{height} of a concept $C$ in an interpretation $R$, $h_R(C)$, as the lowest (most typical) layer in which the concept's  extension is non-empty: \ie
\[
h_R(C)=\min\{i\mid (C^R)\cap L^R_i\neq\emptyset\} \ .
\]
\nd If in a model $R$ there is no individual satisfying a concept $C$, we set $h_R(C)=\infty$.
\begin{definition}[Ranked model]\label{rankedsat}
An interpretation $R=\tuple{\Delta^R,\cdot^R,\prec^R}$ \emph{satisfies} (is \emph{a model of}) $C\subs D$ (denoted $R\sat C\subs D$) iff $C^R\subseteq D^R$, and \emph{satisfies} (is \emph{a model of}) $C\usually D$ iff $\min_{\prec^R}(C^R)\subseteq D^R$ (denoted $R\sat C\usually D$). $R$ \emph{satisfies} (is \emph{a model of})
 $\KB=\tuple{\T,\D}$ iff $R\sat\alpha$ for all  axioms $\alpha\in\T\cup\D$.
 \qed
\end{definition}
\nd Hence, $C\usually D$ is satisfied by $R$ iff all the most typical individuals in $C^R$ are also in $D^R$.
We  say that two ontologies are \emph{rank equivalent} iff they are satisfied by exactly the same ranked models, and that an ontology is \emph{rank satisfiable} iff there is at least a ranked model that satisfies it.
\begin{remark}\label{remark_strict}
Note that from the above definition of the satisfiability of an axiom $C\usually D$ 
we obtain the following correspondence: for every ranked model $R$,
\[
R\sat C\subs \bot \text{ iff } R\sat C\usually \bot \ .
\]
This allows for the translation of every classical axiom $C\subs D$ into a defeasible axiom $C\dlAnd\neg D\usually \bot$. Note also that such a translation is not feasible in \ELObot, as  $\neg$ is not supported in \ELbot.
\qed
\end{remark}
\nd Now, the definition of ranked entailment follows directly from the notion of a ranked model. So, let $\R^\KB$ be the class of the ranked models of an ontology $\KB$.
\begin{definition}[Ranked Entailment]\label{rankedent}
Given an ontology $\KB=\tuple{\T,\D}$ and a defeasible axiom $C\usually D$, \KB~\emph{rationally entails} $C\usually D$ (denoted $\KB\entails_\R C\usually D$) iff $\forall R\in\R^\KB\hspace{-0.1cm},\ R\sat C\usually D$.
\qed
\end{definition}
\nd The main drawbacks of ranked entailment are that it is too weak from the inferential point of view and does not satisfy the (RM) property \cite{BritzEtAl2015a,LehmannMagidor1992}. RC is a kind of entailment that extends Ranked Entailment, allowing us to overcome  these limitations.
It is based on a notion of \emph{exceptionality} that is built on Ranked Entailment.
\begin{definition}[Exceptionality]\label{exceptionality}
A concept $C$ is \emph{exceptional} \wrt~an ontology $\KB=\tuple{\T,\D}$~iff $\KB\entails_\R~\top\usually\neg C$. That is to say, $C$ is \emph{exceptional} \wrt~$\KB$ iff, for every ranked model $R\in\R^\KB$,

\[
C^{R}\cap\min_{\prec^R}(\Delta^{R})=\emptyset \ .
\]
\nd An axiom $C\usually D$ is \emph{exceptional} \wrt~\KB~iff $C$ is exceptional.
\qed
\end{definition}
\nd Intuitively, a concept is exceptional \wrt~an ontology iff it is not possible to have it satisfied by any typical individual (\ie, an individual in the layer $0$, that corresponds to $\min_{\prec^R}(\Delta^{R})$) in any ranked model of the ontology. Iteratively applied, the notion of exceptionality allows to associate to every concept $C$ a \emph{rank value} \wrt~an ontology $\KB=\tuple{\T,\D}$~in the following way (called \emph{RC ranking procedure}).

\begin{enumerate} \label{classicalRC}
\item A concept $C$ has rank $0$ ($r_\KB(C)=0$) iff it is not exceptional \wrt~\KB~(that is, $h_R(C)=0$ for some model $R$ of \KB). In this case we set $r_\KB(C\usually D)=0$ for every defeasible axiom having $C$ as antecedent. The set of the axioms in $\D$ with rank $0$ is denoted as $\D_0$.



\item For $i>0$, $C$ has rank $i$ iff it does not have rank $i-1$ and it is not exceptional
wrt~$\KB^i = \tuple{\T,\D\setminus\bigcup_{j=0}^{i-1}\D_{j}}$.
If $r_\KB(C)=i$, then we set $r_\KB(C\usually D)=i$. The set of the axioms in $\D$ with rank $i$ is denoted as $\D_i$.

\item By iterating the previous step a finite number of times, we finally reach a (possibly empty) subset $\E \subseteq\D$ \st~all the axioms in $\E$ are exceptional w.r.t. $\tuple{\T,\E}$.\footnote{Since $\D$ is finite, we must reach such a point.} If $\E\neq\emptyset$ we define the rank value of the axioms in $\E$ as $\infty$, and the set $\E$ is denoted as $\D_\infty$.
\end{enumerate}
As a consequence, according to the procedure above, \D~can be partitioned into a finite sequence 
$\langle \D_0, \ldots$, $\D_n,\D_\infty \rangle$ ($n\geq 0$), where $\D_\infty$ may be possibly empty. This semantic procedure allows us to give a rank value to every concept and every defeasible subsumption. Using the rank values, we can define the notion of RC as follows:
\begin{definition}[Rational Closure]\label{RCdef}
$C\usually D$ is in the RC of an ontology $\KB$ iff
\[
r_\KB(C\dlAnd D) < r_\KB(C\dlAnd\neg D) \mbox{ or } r_\KB(C)=\infty  \ .
\]
\qed
\end{definition}
\nd Informally, the above definition says that $C\usually D$ is in the rational closure of $\KB$ if the most typical instances of $C$ happen to be all instances of $C\dlAnd D$, and not  of $C\dlAnd \neg D$. 

Given an ontology $\KB=\tuple{\T,\D}$, distinct ways of defining models of $\KB$ characterising its RC can be found in \cite{BritzEtAl2015a} (summarised also here in~\ref{canmod} and presented also in \cite[Section 4.1]{Moodley15}) and in~\cite{Giordano15}.
Both such kinds of models can be described as \emph{minimal} models of the ontology $\tuple{\T,\D}$. Paraphrasing Definition 23 in~\cite{Giordano15}, we can define a preference relation $<$ among ranked interpretations in the following way.

\begin{definition}[\cite{Giordano15}, Definition 23]\label{def_minimal_giordano}
Let $R=\tuple{\Delta^R,\cdot^R,\prec^R}$ and $R'=\tuple{\Delta^{R'},\cdot^{R'},\prec^{R'}}$ be two ranked interpretations s.t. $\Delta^{R}=\Delta^{R'}$ and $C^{R}=C^{R'}$ for every concept $C$. $R$ is preferred to $R'$ ($R<R'$) iff for every $x\in\Delta^R$, $h_R(x)\leq h_{R'}(x)$, and there is a $y\in\Delta^R$ s.t. $h_R(y)< h_{R'}(y)$. An interpretation $R$ is minimal \wrt~an ontology $\KB$ if it is a model of $\KB$ and there is no model $R'$ of $\KB$ s.t. $R'<R$.
\end{definition}

The reason behind the use of minimal models in characterising RC is in the direct connection between minimality and the \emph{presumption of typicality}: in minimal models we \emph{maximise} the amount of typicality for each individual in the domain, modulo the satisfaction of the ontology. We will go back to the role of minimality in Section \ref{rcelo}.

The type of reasoning we are primarily interested in modelling is subsumption checking in \alc~under RC, that is, deciding whether a defeasible subsumption $C\usually D$ is or is not a consequence under RC of an ontology $\tuple{\T,\D}$. In~\cite{BritzEtAl2015a} a detailed decision procedure for subsumption checking in \alc~under RC is described, which we recap here.\footnote{The procedure is based on the one by Casini and Straccia~\cite{CasiniStraccia2010} and paired with a proper semantics; the latter needed to be  modified slightly since it does not always give back the expected result in case $\D^r_\infty\neq\emptyset$. The procedure presented in~\cite{BritzEtAl2015a} have been presented (and peer-reviewed) also in~\cite{Moodley15}.} This will be useful, as our subsumption decision procedure for defeasible \ELbot~will be a variant of it. 
The key step in translating the semantic procedure into a correspondent one, based on classical \alc~decision steps, is given by the following proposition.
\begin{proposition}\label{except}
For every concept $C$ and every ontology $\KB=\tuple{\T,\D}$,~if
\begin{equation}\label{cexcet}
\T \models \bigsqcap \{ \neg E \dlOr F \mid E \usually F \in \D\} \subs \neg C \ 
\end{equation}
then $C$ is exceptional \wrt~$\KB$.
\qed
\end{proposition}
\nd By Proposition~\ref{except}, checking exceptionality can be done by using a classical DL reasoner for \alc. Now, 
consider an \alc~defeasible ontology $\KB=\tuple{\T,\D}$ and a defeasible GCI $C\usually D$. In order to decide whether  $C\usually D$ is in the RC of an ontology $\KB$, we perform two steps: the first one is a \emph{ranking} procedure, that transforms the initial ontology $\KB=\tuple{\T,\D}$ into a rank equivalent ontology $\KB^\bullet=\tuple{\T^\bullet,\D^\bullet}$, where $\D^\bullet$ is partitioned into a sequence $\D^\bullet_0,\ldots,\D^\bullet_n$, with each $\D^\bullet_i$ containing the defeasible axioms with rank $i$; the second one uses $\KB^\bullet$ to decide whether an axiom $C\usually D$ is in the RC of $\KB$.

Specifically, define the function $e$ that, given any ontology $\tuple{\T',\D'}$, returns exceptional axioms as
\begin{equation} \label{eqE}
e(\T',\D') = \{ C \usually D \in \D' \mid \T' \models \bigsqcap \{ \neg E \dlOr F \mid E \usually F \in \D'\} \subs \neg C\} \ .
\end{equation}
\nd The function $e$ gives back axioms in $\D'$ that are exceptional \wrt~$\tuple{\T',\D'}$ (see also~\cite[Section 6]{BritzEtAl2015a}).

Now in order to decide whether  $C\usually D$ is in the RC of $\KB$, we execute the following two steps shown below, which we will call {\bf RC.Step 1} and  {\bf RC.Step 2}. Note that {\bf RC.Step 1}  will correspond to procedure  $\mathtt{ComputeRanking}$, while  {\bf RC.Step 2} is encoded in  procedure  $\mathtt{RationalClosure}$, both presented in Section~\ref{ratclosEL} later on. Also,  the execution of {\bf RC.Step 1}, \ie, procedure $\mathtt{ComputeRanking}$, can be followed \eg~in Example~\ref{exinfty}), which also illustrates why Steps 1.1 and 1.2 may need to be repeated more than once to extract all the needed information.
\begin{description} \label{algoRCalc}
\item[ {\bf RC.Step 1}] Let $\T^0 = \T, \D^0 = \D$ and $i=0$. Repeat Steps 1.1 and 1.2 until $\D_i^\infty = \emptyset$.
\begin{description}
\item[Step 1.1]  Given $\tuple{\T^i, \D^i}$, construct the sequence $\E_0, \E_1, \ldots$
\begin{eqnarray*}
\E_0 & = & \D^i \\
\E_{j+1} & = & e(\T^i, \E_j) \ .
\end{eqnarray*}
\nd Since $\D^i$ is finite, the iteration will terminate with (a possibly empty) fixed-point of $e(\T^i, \cdot)$. 
\item[Step 1.2] For a defeasible GCI $E \usually F \in \D^i$, define  the rank of $E \usually F$ and of concept $E$ as
\begin{eqnarray*}
r^i(E \usually F) & = &
\left \{
\begin{array}{ll}
j & \mbox{if } E \usually F \in \E_j \mbox{ and } E \usually F \not\in \E_{j+1} \\
\infty & \mbox{if } E \usually F \in \E_j \mbox{ for all } j \\ 
\end{array}
\right . \\ \\
r^i(E) & = & r^i(E \usually F) \ .
\end{eqnarray*}
\nd Define 
\begin{eqnarray*}
\D^i_j & = & \{ E \usually F \in \D^i \mid r^i(E \usually F) = j \} \ .
\end{eqnarray*}
It follows that $\D^i$ is partitioned into sets $\D^i_0, \ldots, \D^i_m, \D^i_\infty$, for some $m$, with $\D^i_\infty$ possibly empty. Finally, define
\begin{eqnarray*}
\T^{i+1} & = & \T^i \cup \{E \subs \bot \mid E \usually F \in \D^i_\infty \} \\
\D^{i+1} & = & \D^i \setminus \D^i_\infty\ .
\end{eqnarray*}
\end{description}
\nd {\bf RC.Step 1} terminates  after building a sequence of ontologies $\tuple{\T^0,\D^0}, \ldots, \tuple{\T^k,\D^k}$ and ranking functions $r^0, \ldots, r^k$, for some $k \leq |\D|$, once we reach the point where $\D^k_\infty = \emptyset$. Let $r^k$ partition $\D^k$ into $\D^k_1,\ldots,\D^k_n$ for some $n$.
Furthermore, let  $\T^\bullet=\T^k$, $\D^\bullet=\D^k$, and $\D^\bullet_i=\D^k_i$ for every $0\leq i\leq n$.

Once we have applied {\bf RC.Step 1}, Proposition \ref{except} holds also in the opposite direction.
\begin{proposition}\label{except2}
Given an ontology $\KB^\bullet=\tuple{\T^\bullet,\D^\bullet}$, obtained from the application of ${\bf RC.Step 1}$ to an ontology $\KB=\tuple{\T,\D}$, then for every concept $C$,
\[
\T^\bullet \models \bigsqcap \{ \neg E \dlOr F \mid E \usually F \in \D^\bullet\} \subs \neg C 
\]
if and only if  $C$ is exceptional \wrt~$\KB^\bullet$.
\qed
\end{proposition}

\item[{\bf RC.Step 2}] \label{step2}  So, let 
$\T^\bullet = \T^k, \D^\bullet  = \D^k$, $r^\bullet =r^k$, and $\D^\bullet$ be partitioned into $\D^\bullet_0, \ldots, \D^\bullet_n$. 
\begin{description}
\item[Step 2.1] For $0 \leq i \leq n$ define the concept 
\begin{eqnarray*}
H_i & = & \bigsqcap \{\neg E \dlOr F \mid E \usually F \in  \D^\bullet_i \cup \ldots \cup \D^\bullet_n \} \ .
\end{eqnarray*}
\nd Note that if $j < i$ then $\models H_j \subs H_i$. 

\item[Step 2.2] Finally,  given $C \usually D$, let $H_C$ be the first concept $H_i$ of the sequence $H_0, \ldots, H_n$  such that $\T^\bullet \not \models H_i \subs \neg C$. If there is no such $H_i$ let $H_C$ be $\top$.
Then, we say that $C\usually D$ is \emph{derivable} from $\KB$ iff $\T^\bullet \models C \dlAnd H_C \subs D$. 

$C\subs D$ is \emph{derivable} from $\KB$ iff $\T^\bullet \models C\dlAnd\neg D\usually \bot$ (see Remark \ref{remark_strict}). 
\end{description}
\end{description}
\nd With $\KB\ratent C\usually D$ (resp. $\KB\ratent C\subs D$) we will denote that $C\usually D$ (resp. $C\subs D$) is derivable from $\KB$ via {\bf RC.Step 2}.  In~\cite{BritzEtAl2015a,Moodley15} it is shown that {\bf RC.Step 1} is correct \wrt~the semantic definition of ranking, and that {\bf RC.Step 2} is correct \wrt~the semantic definition of RC (Definition \ref{RCdef}). That is, respectively,
\begin{proposition}[\cite{BritzEtAl2015a}, Proposition 7]\label{ranking_correct}
Given an ontology $\KB=\tuple{\T,\D}$ and a concept $C$, then $r_{\KB} (C)=r^\bullet(C)$ holds.
\qed
\end{proposition}
\begin{proposition}[\cite{BritzEtAl2015a}, Theorem 5]\label{mainratio}
Given an ontology $\KB=\tuple{\T,\D}$, and concepts $C,D$, then $C\usually D$ is in the RC of $\KB$ iff $\KB\ratent C\usually D$.
\qed
\end{proposition}
\begin{remark}\label{remalcrc}
Note that an indispensable requirement of the above described defeasible subsumption procedure for \alc~under RC is to have a classical DL subsumption decision procedure supporting the empty concept, concept conjunction, 
negation and disjunction.
\qed
\end{remark}
\nd From~\cite{BritzEtAl2015a}, the following propositions are immediate. 
\begin{proposition}\label{classgci}
A classical GCI $C\subs D$ is in the RC of $\KB= \tuple{\T,\D}$ iff $\T^\bullet \models C \subs D$, where $\T^\bullet$ has been computed using 
{\bf RC.Step 1}.
\qed
\end{proposition}
\begin{corollary}[\cite{BritzEtAl2015a}, Corollary 2]\label{consistency}
An ontology $\KB=\tuple{\T,\D}$ does not have a ranked model iff $\T^\bullet\entails \top\subs\bot$, where $\T^\bullet$ has been computed using 
{\bf RC.Step 1}.
\qed
\end{corollary}
%
\section{Rational Closure in \ELbot}\label{ratclosEL}
%
\nd We now present a subsumption decision procedure under RC for \ELbot~by adapting the procedure for \alc~under RC to \ELbot.  By Remark~\ref{remalcrc}, as \ELbot~does not support concept negation and disjunction, the main problem we have to address is to find a way to overcome this limitation. Concretely, we will define alternative ways  both to
\begin{enumerate}
\item  express whether an \ELbot~concept is exceptional using a classical \ELbot~subsumption problem only;
\item  express the subsumption problems in  Steps 2.1 and 2.2 above in terms of  \ELbot~subsumption problems only.
\end{enumerate}
%
\subsection{A Subsumption Decision Procedure for \ELbot~under RC}\label{procedure}
%
\nd Consider a defeasible \ELbot~ontology $\KB=\tuple{\T,\D}$.
As for \alc, we will define two procedures. The first one is a \emph{ranking} procedure that transforms the initial ontology $\KB=\tuple{\T,\D}$ into a rank equivalent ontology $\KB^*=\tuple{\T^*,\D^*}$, where $\D^*$ is partitioned into a sequence $\D_0,\ldots,\D_n$, with each $\D_i$ containing the defeasible axioms with rank $i$. The second one uses $\KB^*$ to decide whether an axiom $C\usually D$ is in the RC of $\KB$.
\paragraph{The Ranking Procedure.} \label{rankings}
Given an ontology $\tuple{\T,\D}$, the ranking procedure is defined  by means of two procedures: one for finding  exceptional axioms and one for determining the rank value of axioms, as defined in Section~\ref{alc}.

In the following, given an ontology $\tuple{\T,\E}$, and a new atomic concept $\delta_\E$ (with $\E$ indicating a set of defeasible subsumptions), we define $\T_{\delta_\E}$ as
\begin{equation} \label{Tdelta}
\T_{\delta_\E}=\T\cup\{E \dlAnd\delta_\E\subs F \mid E \usually F \in\E\} \ .
\end{equation}
\nd Informally, we introduce the atom $\delta_\E$ as a way of representing the information that characterises the lowest rank. Hence, its introduction is aimed at the formalisation of the typicality of the lowest layer: $C\dlAnd\delta_{\E}$ is introduced to represent the individuals in $C$ that are in the lowest layer.
\begin{remark}
The aim of  the definition for $\T_{\delta_\E}$ is to replace the \alc~subsumption test in Proposition~\ref{except} with the \ELbot~subsumption test
\[
\T_{\delta_\D}\entails C\dlAnd \delta_\D \subs\bot \ , 
\]
for an ontology $\tuple{\T,\D}$.
\qed
\end{remark}
\nd We obtain an analogue of Proposition~\ref{except}.
\begin{proposition}\label{exceptB}
For every concept $C$ and every ontology $\KB=\tuple{\T,\D}$,~if
\begin{equation}\label{cexcetB}
\T_{\delta_\D}\entails C\dlAnd \delta_\D \subs\bot \ ,
\end{equation}
\nd where $\delta_\D$ is a new atomic concept, then $C$ is exceptional \wrt~$\KB$.
\qed
\end{proposition}
\nd Procedure $\mathtt{Exceptional}$ illustrates how to compute the exceptional axioms.
\begin{procedure}[h]
\caption{Exceptional($\T,\E$)}\label{algex}
{\small
 \KwIn{$\T\text{ and }\E\subseteq{\D}$}
 \KwOut{$\E'\subseteq\E$ such that $\E'$ is a set  of exceptional axioms \wrt~$\tuple{\T,\E}$}
$\E':=\emptyset$\;
$\T_{\delta_\E}=\T\cup\{E \dlAnd\delta_\E\subs F\mid E\usually F\in\E\}$, where  $\delta_\E$ is a new atomic concept\;
 \ForEach{$E \usually F \in\E$}{	
   	\If{$\T_{\delta_\E}\entails E \dlAnd \delta_\E\subs\bot$}{$\mathcal{E'}$ := $\mathcal{E'}\cup\{E \usually F\}$\;}
}
\Return{$\mathcal{E'}$}
}
\end{procedure}

\nd The procedure $\mathtt{ComputeRanking}$ instead, shows how we implement {\bf RC.Step 1} in \ELbot, which we comment shortly on next. 
 \begin{procedure}[h]
\caption{ComputeRanking($\KB$)\label{algrank}}
{\small
\KwIn{Ontology $\KB=\tuple{\T,\D}$}
\KwOut{Ontology $\tuple{\T^*,\D^*}$, partitioning (ranking) $\R =\{\D_0,\ldots,\D_n\}$ of $\D^*$}
$\T^*$:=$\T$\;
$\D^*$:=$\D$\;
$\R$:=$\emptyset$\;
\Repeat {$\D_\infty=\emptyset$}
{$i$ := $0$\;
$\mathcal{E}_{0}$ := $\D^*$\;
$\mathcal{E}_{1}$ := $\mathtt{Exceptional}$($\T^*,\mathcal{E}_{0}$)\;
\While{$\E_{i+1}\neq\E_{i}$}{
$i$ := $i$ + 1\;
$\E_{i+1}$ := $\mathtt{Exceptional}$($\T^*,\E_{i}$)\;
}
$\D_\infty$ := $\E_{i}$\;
$\D^*$ := $\D^*\setminus\D_\infty$\;
$\T^*$ := $\T^*\cup\{E \subs\bot\mid E\usually F \in \D_\infty\}$\;}

\For{$j$ = $1$ to $i$}{
$\D_{j-1}$ := $\E_{j-1}\setminus\E_{j}$\;
$\R$ := $\R\cup\{\D_{j-1}\}$\;
}
\Return{$\tuple{\tuple{\T^*,\D^*},\R}$}
}
\end{procedure}
We start  by considering an ontology $\KB=\tuple{\T,\D}$. Lines 8-10 loop until we reach a  (possibly empty)  fixed-point of exceptional axioms.
Then, each axiom $C\usually D$ in the fixed point of the exceptionality function is eliminated from $\D^*$ (line 12) and we add $C\subs\bot$ to $\T^*$ (line 13). We repeat the loop in lines 5 - 13 until no exceptional axioms can be found anymore (\ie, $\E_i=\E_{i+1} = \emptyset$, for some $i\geq 0$).
\begin{remark}
Note that the loop  in between lines 5 - 13  allows us to move all the strict knowledge possibly `hidden' inside the DBox to the TBox. That is, there may be defeasible axioms in the DBox that are actually equivalent to classical axioms, and, thus, can be moved from the DBox to the TBox as classical inclusion axioms. Example~\ref{exinfty}  illustrates such a case.
\qed
\end{remark}
\nd Lines 15-17 determine the rank value of the remaining defeasible axioms not in $\D_\infty$. That is,  set $\D_{j-1}$ is the set of axioms of rank $j-1$ ($1 \leq j \leq i$), which are the axioms in $\E_{j-1} \setminus \E_{j}$.

The following can easily be shown.
\begin{proposition} \label{propcr}
Consider an ontology $\KB=\tuple{\T,\D}$. Then $\mathtt{ComputeRanking}(\KB)$ returns the ontology $\tuple{\T^{*}, \D^{*}}$, where $\D^*$ is partitioned into a sequence $\D_0,\ldots,\D_n$, where $\T^*$, $\D^*$ and all $\D_i$ are equal to the sets $\T^\bullet$, $\D^\bullet$ and $\D^\bullet_0,\ldots,\D^\bullet_n$ obtained via {\bf RC.Step 1}.
\qed
\end{proposition}
\nd Also, once we have applied the procedure $\mathtt{ComputeRanking}$, the proposition corresponding in the \ELbot~framework to Proposition~\ref{except2} holds.
\begin{proposition}\label{exceptB2}
Given an ontology $\KB^*=\tuple{\T^*,\D^*}$, obtained from the application of the procedure $\mathtt{ComputeRanking}$ to an ontology $\KB=\tuple{\T,\D}$, for every concept $C$,
\[
\T^*_{\delta_{\D^*}}\entails C\dlAnd \delta_{\D^*} \subs\bot \ , 
\]
if and only if  $C$ is exceptional \wrt~$\KB^*$.
\qed
\end{proposition}
\nd Next, we describe some examples that illustrate the behaviour of the ranking procedure. The following example shows a case in which there is non-defeasible knowledge `hidden' in a DBox and that  more than one cycle of the lines 4-14 in $\mathtt{ComputeRanking}$ is needed to extract this information.
\begin{example}\label{exinfty}
Let $\KB=\tuple{\T,\D}$ be an ontology with
\begin{eqnarray*}
\T &  =\{ & A\subs B, \\
&&  B\dlAnd D\subs\bot \ \ \} \\
\D & =\{ & B\usually C, \\
&& A\usually D, \\
&&  E\usually \exists r.A \ \ \} \ .
 \end{eqnarray*}
It can be verified that the execution of $\mathtt{ComputeRanking}(\KB)$ is as follows:
\begin{eqnarray*}
&& \T^* = \T,  \D^* = \D, \R=\emptyset\\
\mathtt{repeat} 1 & i=0  & \mathcal{E}_{0} = \D^*,  \mathcal{E}_{1} = \{ A\usually D \} \\
		        & i=1  & \mathcal{E}_{2} = \{ A\usually D \} \mbox{ (end while) } \\
                         &       & \D_\infty = \mathcal{E}_{2} = \{ A\usually D \} \\
                          &       & \D^* = \D^* \setminus  \{ A\usually D \}  = \{B\usually C, E\usually \exists r.A\}\\
                         &       & \T^* = \T^* \cup  \{ A \subs \bot \}  = \{A\subs B, B\dlAnd D\subs\bot , A \subs \bot\}\\\\
\end{eqnarray*}
\begin{eqnarray*}                                               
\mathtt{repeat} 2 & i=0  & \mathcal{E}_{0} = \D^*,  \mathcal{E}_{1} = \{ E\usually \exists r.A\} \\
                         & i=1  & \mathcal{E}_{2} = \{  E\usually \exists r.A\ \} \mbox{ (end while) } \\
                         &       & \D_\infty = \mathcal{E}_{2} = \{ E\usually \exists r.A \} \\
                         &       & \D^* = \D^* \setminus  \{ E\usually \exists r.A \ \}  = \{B\usually C\} \\
                         &       & \T^* = \T^* \cup  \{ E \subs \bot \}  = \{A\subs B, B\dlAnd D\subs\bot , A \subs \bot, E \subs \bot\} \\\\
\mathtt{repeat} 3 & i=0  & \mathcal{E}_{0} = \D^*,  \mathcal{E}_{1} = \emptyset \\
  			& i=1  & \mathcal{E}_{2} = \emptyset \mbox{ (end while) }\\
			&       & \D_\infty = \mathcal{E}_{2} \\
			&       & \D^* = \D^* \setminus \emptyset  = \{ B\usually C\} \\
			&       & \T^* = \T^* \cup \emptyset  = \{A\subs B, B\dlAnd D\subs\bot , A \subs \bot, E \subs \bot\} \mbox{ (end repeat) } \\ \\
\mathtt{for}	& j=1  & \D_{0} = \E_{0}\setminus \E_{1} = \{ B\usually C \} \\
			&       & \R = \R \cup\{\D_{0}\} = \{\D_{0} \}  \mbox{ (end for) }
\end{eqnarray*}
Therefore, $\mathtt{ComputeRanking}(\KB)$ terminates with
\begin{eqnarray*}
 \T^* & = & \{ \ \   A\subs B, B\dlAnd D\subs\bot , A \subs \bot, E \subs \bot \ \  \} \\
 \D^*& = & \{ \ \ B\usually C \ \ \} \\
  \R    & = &  \{ \ \ \D_{0} \ \ \} \\
  \D_0  & = &  \{ \ \  B\usually C \ \   \}  \ .
\end{eqnarray*}
The only defeasible axiom  in $\D^*$ is $B\usually C$, which has rank $0$. Axioms $A\usually D$ and $E\usually \exists r.A$ have rank $\infty$ instead, and so are substituted by the classical axioms $A \subs \bot$ and $E \subs \bot$. Note that we need to iterate the loop in lines 5-13 in procedure $\mathtt{ComputeRanking}$ more than once to determine such ranking values. In fact, in the first round we get $A \subs \bot$, while the second round we get also $E \subs \bot$.
\qed
\end{example}
\begin{example}\label{exranking}
Consider the ontology $\KB$ in Example \ref{ex02}. It can be verified that the execution of\newline $\mathtt{ComputeRanking}(\KB)$ is as follows:
\begin{eqnarray*}
&& \T^* = \T,  \D^* = \D, \R=\emptyset\\
\mathtt{repeat} 1 & i=0  & \mathcal{E}_{0} = \D^*,  \mathcal{E}_{1} = \{ \dl{MRBC}\usually \dl{NotN} \} \\
		        & i=1  & \mathcal{E}_{2} = \emptyset \\
			& i=2  & \mathcal{E}_{3} = \emptyset \mbox{ (end while) } \\
                         &       & \D_\infty = \mathcal{E}_{3} = \emptyset \\
                         &       & \D^* = \D^* \setminus  \emptyset  = \D \\
                         &       & \T^* = \T^* \cup  \emptyset  = \T \mbox{ (end repeat) } 
\end{eqnarray*}
\begin{eqnarray*}
\mathtt{for}	& j=1  & \D_{0} = \E_{0}\setminus \E_{1} = \{\dl{VRBC}\usually\exists\dl{hasCM}.\top,  \dl{VRBC}\usually\exists\dl{hasN}.\top \} \\
			&       & \R = \R \cup\{\D_{0}\} = \{\D_{0} \} \\
			& j=2  & \D_{1} = \E_{1}\setminus \E_{2} = \{ \dl{MRBC}\usually \dl{NotN} \} \\
			&       & \R = \R \cup\{\D_{1}\} = \{\D_{0}, \D_{1} \}  \mbox{ (end for) }
\end{eqnarray*}

\nd Therefore, $\mathtt{ComputeRanking}(\KB)$ terminates with
\begin{eqnarray*}
 \T^* & = & \{ \ \   \dl{BRBC}\subs\dl{MRBC}, \dl{ARBC}\subs\dl{VRBC}, \dl{MRBC}\subs\dl{VRBC}, \exists\dl{hasN}.\top\dlAnd \dl{NotN}\subs\bot \ \  \} \\
 \D^*& = & \{ \ \ \dl{VRBC}\usually\exists\dl{hasCM}.\top,  \dl{VRBC}\usually\exists\dl{hasN}.\top, \dl{MRBC}\usually \dl{NotN}  \ \ \} \\
  \R    & = &  \{ \ \ \D_{0}, \D_{1} \ \ \} \\
  \D_0  & = &  \{ \ \  \dl{VRBC}\usually\exists\dl{hasCM}.\top,  \dl{VRBC}\usually\exists\dl{hasN}.\top \ \   \}  \\
  \D_1  & = &  \{ \ \  \dl{MRBC}\usually \dl{NotN} \ \   \}  \ .
\end{eqnarray*}
\nd Defeasible axioms in $\D_0$ have rank value $0$, while the axiom in  $\D_1$ has rank value $1$.  
\qed
\end{example}
\begin{remark} \label{reelbotcla}
From Propositions~\ref{classgci} and~\ref{propcr}, and Corollary~\ref{consistency}, we  immediately also get decision procedures for \ELbot~to determine both whether  classical GCIs  are in the RC of an ontology and whether an ontology has a ranked model, so we do not address these decision problems further.
\qed
\end{remark}
%
%
\paragraph{The Subsumption Decision Procedure.} \label{procedurerc}
So far,  we have defined a procedure that determines the rank value of the axioms in a KB, which is based on a sequence of classical \ELbot~subsumption decision steps (those in the 
$\mathtt{Exceptional}$ procedure, line 4). As next, we illustrate how we are going to implement {\bf RC.Step 2} using \ELbot~subsumption tests only and, thus, get an algorithm to decide whether a defeasible axiom $C\usually D$ is in the RC of an \ELbot~KB.  Specifically, given a KB $\KB$, let us assume that we have applied to it the $\mathtt{ComputeRanking}$ procedure and, thus, the returned KB 
 does not have defeasible inclusion axioms with $\infty$ as rank value.

 In the following, given $\KB=\tuple{\T,\D}$ as the output of the $\mathtt{ComputeRanking}$ procedure, with $\D$ partitioned into $\D_0,\ldots,\D_n$, and given new atomic concepts $\delta_i$ ($0 \leq i \leq n$), we define $\T_{\delta_i}$ as
\begin{equation} \label{Tdeltai}
\T_{\delta_i} =\T\cup\{E \dlAnd\delta_i\subs F \mid E \usually F \in  \D_i \cup \ldots \cup \D_n \} \ .
\end{equation}
\begin{remark}\label{entrc}
The purpose of definition $\T_{\delta_i}$ is to encode the concepts $H_i$ in {\bf RC.Step 2.2} as \ELbot~GCIs. Specifically, the aim is to replace the subsumption tests
\begin{eqnarray}
\T^* & \not \models & H_i \subs \neg C \label{rct1}\\
\T^* & \models & C \dlAnd H_i \subs D  \label{rct2}
\end{eqnarray}
\nd  in Step 2.2 with the two equivalent \ELbot~subsumptions tests
\begin{eqnarray}
\T^*_{\delta_i} & \not\entails & C\dlAnd\delta_i\subs \bot  \label{rct3} \\
\T^*_{\delta_i} & \entails & C\dlAnd\delta_i\subs D  \label{rct4} \ .
\end{eqnarray}
\nd respectively.
\qed
\end{remark}
\begin{proposition}\label{lemeq}
By referring to Remark~\ref{entrc}, the subsumption test~(\ref{rct1}) (resp.~\ref{rct2}) is equivalent to the subsumption test~(\ref{rct3}) (resp.~\ref{rct4}).
\qed
\end{proposition}
\begin{example}[Example~\ref{exranking} cont.]\label{exrankingA}
From $\tuple{\tuple{\T^*,\D^*},\R}$ in Example~\ref{exranking}, we get by definition that
\begin{eqnarray*}
\T^*_{\delta_0} & = \T^*\cup\{ & \dl{VRBC}\dlAnd\delta_0\subs\exists\dl{hasCM}.\top,\\
&& \dl{VRBC}\dlAnd\delta_0\subs\exists\dl{hasN}.\top,\\
&& \dl{MRBC}\dlAnd\delta_0\subs \dl{NotN} \ \ \} \ .
\end{eqnarray*}
\nd and
\begin{eqnarray*}
\T^*_{\delta_1} & = \T^*\cup\{\dl{MRBC}\dlAnd\delta_1\subs \dl{NotN}\}  \ .
\end{eqnarray*}
Note that we get the following results:
\begin{eqnarray*}
\T^*_{\delta_0}  & \entails  & \dl{MRBC}\dlAnd\delta_0\subs\bot\\
\T^*_{\delta_0} & \not\entails & \dl{VRBC}\dlAnd\delta_0\subs\bot \\
\T^*_{\delta_1} & \not\entails & \dl{MRBC}\dlAnd\delta_1\subs\bot
\end{eqnarray*}
\qed
\end{example}
\nd Procedure  $\mathtt{RationalClosure}$ illustrates how we implement the subsumption decision procedure for \alc~under RC, using \ELbot~subsumption tests only. Note that essentially lines 1 - 7 implement {\bf RC.Step~1}, while lines 8 - 18 implement {\bf RC.Step 2}.
\begin{procedure}[h]
\caption{RationalClosure($\KB,\alpha$) \label{RCalg}}
{\small
\KwIn{Ontology $\KB$ and defeasible axiom $\alpha$ of the form  $C\usually D$}
\KwOut{$\mathtt{true}$ iff $C\usually D$ is in the Rational Closure of $\KB$}
$CL : = \T \entails C\subs D$ //Check if $\alpha$ holds classically\;
  \If{CL}{\Return{$CL$}}
  $\tuple{\tuple{\T^*,\D^*},\{\D_0,\ldots,\D_n\}}$ := $\mathtt{ComputeRanking}$($\KB$)\;
  $CL : = \T^* \entails C\subs D$ //Check if $\alpha$ holds classically, after finding strict knowledge in $\D$\;
  \If{CL}{\Return{$CL$}}
  //Compute $C$'s rank $i$\;
  $i$ :=  $0$; $\D_\R$ :=  $\D^*$\;
  $\T_{\delta_0}:=\T^*\cup\{E\dlAnd\delta_0\subs F\mid E\usually F\in\D_\R\}$, where  $\delta_0$ is a new atomic concept\;
  \While{$\T_{\delta_i}\entails C\dlAnd\delta_i\subs \bot$ {\bf and} $\D_\R\neq\emptyset$}{
    $\D_\R$ := $\D_\R\backslash\D_{i}$; $i$ := $i + 1$\;
    $\T_{\delta_i}:=\T^*\cup\{E\dlAnd\delta_i\subs F\mid E\usually F\in\D_\R\}$, where  $\delta_i$ is a new atomic concept\;
  }
  // Check now if $\alpha$ holds under RC\;
  \eIf{$\T_{\delta_i}\not\entails C\dlAnd\delta_i\subs \bot$}{\Return{$\T_{\delta_i}\entails C\dlAnd\delta_i\subs D$}}{\Return{$CL$}}
  }
\end{procedure}

\begin{example}[Example~\ref{exrankingA} cont.]\label{exdecision}
We want to decide whether the red blood cells of a bovine ($\dl{BRBC}$) should presumably have a nucleus, that is, whether $\dl{BRBC}\usually\exists\dl{hasN}.\top$,
$\dl{BRBC}\usually\dl{notN}$, or neither of them are in the RC of $\KB$.

First of all, we determine the rank of $\dl{BRBC}$, then we check whether we can conclude that the typical elements of $\dl{BRBC}$ are in $\dl{hasN}.\top$ (or in $\dl{notN}$).
We have that
\begin{eqnarray*}
\T^*_{\delta_0} & \entails & \dl{BRBC}\dlAnd\delta_0\subs\bot\\
\T^*_{\delta_1} & \not\entails & \dl{BRBC}\dlAnd\delta_1\subs\bot \ .
\end{eqnarray*}
\nd Hence $\dl{BRBC}$ is of rank $1$, and we can associate with it the defeasible information of rank $1$, that is, we can use the TBox $\T^*_{\delta_1}$.
\begin{eqnarray*}
\T^*_{\delta_1} &\not\entails & \dl{BRBC}\dlAnd\delta_1\subs\exists\dl{hasN}.\top\\
\T^*_{\delta_1} & \entails  & \dl{BRBC}\dlAnd\delta_1\subs\dl{notN} \ .
\end{eqnarray*}
\nd Therefore, we  conclude that it's not the case that the red blood cells of bovines presumably have a nucleus. That is,
$\mathtt{RationalClosure}$($\KB,\dl{BRBC}\usually\exists\dl{hasN}.\top$)  and
$\mathtt{RationalClosure}$($\KB,\dl{BRBC}\usually\dl{notN}$) return respectively $\mathtt{false}$ and $\mathtt{true}$.
\qed
\end{example}
\nd By Propositions~\ref{mainratio}, \ref{propcr}, and \ref{lemeq}, and by construction of the $\mathtt{RationalClosure}$ procedure, we immediately get the following result.
\begin{proposition}\label{completeness}
Consider an \ELbot~ontology $\KB=\tuple{\T,\D}$ and a defeasible GCI $C\usually D$. Then,  $C\usually D$ is in the RC of $\KB$ iff $\mathtt{RationalClosure}(\KB,C\usually D)$ returns $true$.
\qed
\end{proposition}

\subsection{Computational Complexity}\label{seccomplexity}

\nd Classical subsumption can be decided in polynomial time for  \ELbot~\cite{Baader05a}. We next show that our subsumption decision procedure under RC requires a polynomial number of classical subsumption test and, thus, is polynomial overall \wrt~the size of a KB.

As we have seen, the entire procedure can be reduced to a sequence of classical entailments tests, while all other operations are linearly bounded by the size of the KB. Therefore, in order to determine the computational complexity of our method, we have to check, given a KB $\KB=\tuple{\T,\D}$ as input, how many classical entailment tests are required in the worst case.

It is easily verified that   $\mathtt{Exceptional}(\T,\E)$ performs at most  $|\E| \in \mathcal{O}(|\D|)$ subsumption tests. 
Now, let us analyse
$\mathtt{ComputeRanking}(\KB)$. Line 7 requires $\mathcal{O}(|\D|)$ subsumption test. Lines 8 - 10 require at most $\mathcal{O}(|\D|^2)$ subsumption tests as at each round, $|\E_{i+1}|$ is $|\E_{i}| - 1$ in the worst case. At each \texttt{repeat} round $|\D^*|$ decreases in size (at line 12), and thus the \texttt{repeat} loop is iterated at most $\mathcal{O}(|\D|)$ times. Therefore, $\mathtt{ComputeRanking}$ requires at most $\mathcal{O}(|\D|^3)$ subsumption tests.
This gives us the following proposition:
%
\begin{proposition}\label{complexity1a}
Given a KB $\KB=\tuple{\T,\D}$,  procedure $\mathtt{ComputeRanking}$  runs in polynomial time \wrt~the size of $\KB$.
\qed
\end{proposition}
\nd Note that if the KB remains unchanged in between several defeasible subsumption tests, then the ranking procedure needs to be executed  only once.

Now consider  $\mathtt{RationalClosure}(\KB,\alpha)$. Lines 1 - 3 require one subsumption test.
In line 4, the value of $n$ is bounded by $|\D|$ and line 4 requires at most $\mathcal{O}(|\D|^3)$ subsumption tests.
Lines 5 - 7 require one subsumption test.
The loop in lines 11 - 13 is executed at most $|\D|$ times (as at each loop $|\D_{\R}|$ decreases), at each iteration we execute one subsumption test only, and there are at most two subsumption tests between lines 15 - 18. Hence, $\mathtt{RationalClosure}(\KB,\alpha)$ requires  at most $\mathcal{O}(|\D|^3 + |\D|)$ subsumption tests. 
Therefore,
\begin{proposition}\label{complexity1}
Procedure $\mathtt{RationalClosure}$, that decides whether the defeasible inclusion axiom $C\usually D$ is in the RC of $\KB=\tuple{\T,\D}$, runs in polynomial time \wrt~the size of $\KB$.
\qed
\end{proposition}

%
\subsection{Normal Form}\label{normalforms}
%
\nd Usually, a classical \ELbot~ontology is transformed into a normal form to which one then applies a subsumption decision procedure~\cite{Baader05a}.
In the following, we extend the notion of normal form to defeasible \ELbot~ontologies, and show that our subsumption decision procedure works fine for normalised ontologies as well.\footnote{KBs in normal form are also needed in Section~\ref{extend}.} That is, the normal form transformation of an ontology $\KB$ is a conservative extension of $\KB$ also \wrt~RC.
So, let us recap that a classical \ELbot~ontology is in \emph{normal form} if the axioms in it have the form:
\begin{itemize}
\item $C_1\subs D$
\item $C_1\dlAnd C_2\subs D$
\item $\exists r.C_1\subs D$
\item $C_1\subs \exists r.C_2$
\end{itemize}
\nd where $C_1,C_2\in \CN\cup\{\top\}$ and $D\in \CN\cup\{\top,\bot\}$. One may transform axioms into normal form by applying the following rules:
\begin{itemize}
\item[R1:] $C\dlAnd \hat{D}\subs E\mapsto  \hat{D}\subs A, C\dlAnd A\subs E$;
\item[R2:] $\exists r.\hat{C}\subs D\mapsto \hat{C}\subs A, \exists r.A\subs D$;
\item[R3:] $\bot\subs D\mapsto \emptyset$.
\end{itemize}
\begin{itemize}
\item[R4:] $\hat{C}\subs\hat{D}\mapsto \hat{C}\subs A, A\subs\hat{D}$;
\item[R5:] $B\subs\exists r.\hat{C}\mapsto B\subs \exists r.A, A\subs\hat{C}$;
\item[R6:] $B\subs C\dlAnd D\mapsto B\subs C,B\subs D$,
\end{itemize}
\nd where $\hat{C},\hat{D}\notin \CN\cup\{\top\}$, and $A$ is a new  atomic concept. Rules R1-R3 are applied first, then rules R4-R6 are applied, until no more rule can be applied. It is easily verified that the transformation is time polynomial and entailment preserving, \ie, given a TBox $\T$ and its normal form transformation $\T'$, then $\T\entails C\subs D$ iff $\T'\entails C\subs D$ for every $C,D$ not using any of the newly introduced atomic concepts $A$ by the rules $R1-R6$ above.

The above result is not sufficient to guarantee that we can apply this kind of transformation also to defeasible KBs. The problem is that the notion of \emph{logical equivalence} in a preferential setting, or \emph{rank equivalence} as we have called it up to now, follows rules that are slightly different from the ones characterising logical equivalence in classical reasoning. In particular, we are allowed to substitute a concept with a logically equivalent one on the left of a defeasible subsumption relation (LLE allows it) and on the right (a consequence of RW). However,  there are cases that are equivalent in their classical formulation but not equivalent in the preferential one.  For example, the two axioms  $C\subs D$ and $\top\subs\neg C\dlOr D$ are equivalent, while $C\usually D$ and $\top\usually\neg C\dlOr D$ are not rank equivalent (in fact, they convey different meanings, with the former indicating that all the most typical C's are D's, and the latter indicating that the most typical objects are either not C's or D's). Therefore, the normal form transformation rules have to preserve rank equivalence, which we are going to check next.

So, let $\KB=\tuple{\T,\D}$ be a defeasible ontology. We say that $\D$ is in \emph{normal form} if  each defeasible axiom in $\D$ is of the form
$A \usually B$, where $A,B \in \CN$. We say that $\KB$ is in \emph{normal form} if $\T$ and $\D$ are in normal form.
We next show how to transform  $\KB=\tuple{\T,\D}$ into normal form. First,  we replace every axiom $C\usually D\in\D$ with an axiom $A_C\usually A_D$ (with $A_C,A_D$ new atomic concepts), and add $A_C = C$ and $A_D = D$ to the TBox $\T$
(they are both  valid $\ELbot$ expressions); then we apply the classical $\ELbot$ normalisation steps to the axioms in $\T$.
In this way, we end up with a new knowledge base $\tuple{\T',\D'}$ that is in normal form. This transformation still remains, of course,  time polynomial.
\begin{example}\label{exnorm02}
A normal form of the KB in Example~\ref{ex02} is  $\KB=\tuple{\T,\D}$ with
\begin{eqnarray*}
\T  & = \{ &\dl{BRBC}\subs\dl{MRBC}, \\
&& \dl{ARBC}\subs\dl{VRBC}, \\
&& \dl{MRBC}\subs\dl{VRBC}, \\
&&  \exists\dl{hasN}.\top \subs\dl{A_1}, \\
&& \dl{A_1}\subs\exists\dl{hasN}.\top,\\
&&\exists\dl{hasCM}.\top  \subs \dl{A_2},\\
&& \dl{A_2}\subs\exists\dl{hasCM}.\top,\\
&&  \dl{A_1} \dlAnd \dl{NotN}\subs\bot \ \ \}  \\ \\
 \D & = \{ & \dl{VRBC}\usually\dl{A_2},\\
 &&  \dl{VRBC}\usually \dl{A_1},\\
 && \dl{MRBC}\usually \dl{NotN} \  \ \} \ .
\end{eqnarray*}
\qed
\end{example}
\nd We can prove that the ranking procedure gives back equivalent results whether we apply it to $\KB=\tuple{\T,\D}$ or to its normal form transformation  $\KB'=\tuple{\T',\D'}$. To this end, it suffices to show the following. 
\begin{proposition}\label{normalformrank}
Given an ontology $\KB=\tuple{\T,\D}$, $C \usually D \in \D$ and the corresponding ontology in normal form $\KB'=\tuple{\T',\D'}$, then $C$ is exceptional \wrt~$\KB$~iff $A_C$ is exceptional \wrt~$\KB'$, where $A_C$ is the new atomic concept introduced by the normalisation procedure to replace $C \usually D$ with $A_C \usually A_D$.
\qed
\end{proposition}
\nd Therefore, given an ontology $\KB=\tuple{\T,\D}$, $C \usually D \in \D$ and the corresponding ontology in normal form $\KB'=\tuple{\T',\D'}$,  by Proposition~\ref{normalformrank} and by construction of the  $\mathtt{ComputeRanking}$ procedure it follows that $C \usually D \in \D_i$ iff $A_C \usually A_D \in \D'_i$, where $\R =\{\D_0,\ldots,\D_n\}$ and $\R' =\{\D'_0,\ldots,\D'_n\}$ are the partitions computed by the ranking procedure applied to $\KB$ and $\KB'$, respectively. From this, we immediately get the following result.
\begin{proposition}\label{normalform}
Given an ontology $\KB=\tuple{\T,\D}$ and its normal form translation $\KB'=\tuple{\T',\D'}$, then for every pair of atomic \ELbot-concepts $A,B$ occurring in \KB, 
$\KB\ratent A\usually B$ iff $\KB'\ratent A\usually B$.
\qed
\end{proposition}
\nd Observe that, using a KB in normal form, all steps of our  RC decision procedure use classical \ELbot~TBoxes in normal form (as axioms $A_C\dlAnd\delta_i\subs A_D$ are  in normal form, too).

%
\section{Defeasible Inheritance-based Description Logics}\label{extend}
%
\nd So far, we have considered RC~\cite{KrausEtAl1990,Lehmann1995}: both the procedural and the semantic characterisations are well defined, and directly model a principle that is at the core of typicality reasoning: the \emph{presumption of typicality}. Moreover, RC is a \emph{syntactically independent} form of closure. That is, all the KBs that are rank equivalent generate the same RC (to the best of our knowledge, no other form of closure extending RC satisfies this property). 

\begin{proposition}\label{RC_syntactic_independent}
Let $\KB=\tuple{\T,\D}$ and $\KB'=\tuple{\T',\D'}$ be rank equivalent. For every $\ELbot$ defeasible GCI $C\usually D$, $C\usually D$ is in the RC of $\KB$ iff $C\usually D$ is in the RC of $\KB'$.
\end{proposition}

\nd Yet, it is well-known that the main limitation of RC, from an inferential point of view, is that an exceptional subclass cannot inherit \emph{any} of the typical properties of its superclass (also known as \emph{Drowning Effect})

%
\begin{example}\label{stopinherit}
Consider Examples~\ref{exranking} and \ref{exdecision}. The mammalian red blood cells are an exceptional subclass of the vertebrate red blood cells since they do not have a nucleus. So, the conflict determining the exceptionality of $\dl{MRBC}$ is determined by the axioms $\dl{VRBC}\usually\exists\dl{hasN}.\top$ and $\dl{MRBC}\usually \dl{NotN}$, while there is no conflict \wrt~the axiom $\dl{VRBC}\usually\exists\dl{hasCM}.\top$. Hence, it seems still reasonable to conclude that mammalian red blood cells have a cell membrane (\ie, $\dl{MRBC}\usually\exists\dl{hasCM}.\top$). However, as shown in Example~\ref{exranking}, $\dl{VRBC}\usually\exists\dl{hasCM}.\top\in\D_0$ and the rank of $\dl{MRBC}$ is $1$ and, thus, we cannot conclude $\dl{MRBC}\usually\exists\dl{hasCM}.\top$ under RC.
\qed
\end{example}
\nd In order to overcome such inferential limits, some closure operations extending RC have been proposed in DLs: for example, the defeasible Inheritance-based approach~\cite{CasiniStraccia13}, the Relevant Closure~\cite{CasiniEtAl2014}, and the Lexicographic Closure \cite{CasiniStraccia14}.  Here we consider a modification of RC based on the use of inheritance nets to identify the axioms taking part in each specific conflict~\cite{CasiniStraccia13}. As we are going to see, the Defeasible Inheritance-based approach is  interesting when considering the \ELbot~framework since both it allows to overcome the drowning effect and preserves computational tractability.\footnote{In~\cite[Appendix A]{CasiniStraccia13} it is also shown that the defeasible inheritance-based approach behaves well and better than RC \wrt~most of the ``benchmark'' examples illustrated there.}
\begin{remark}\label{remRelRC}
At the time of writing, we neither found  a tractable procedure to decide defeasible subsumption under  Relevant Closure nor for  
Lexicographic Closure~\cite{CasiniEtAl2014,CasiniStraccia14}. In fact, we conjecture that no tractable procedures exist for these cases.
\qed
\end{remark}


\newcommand{\strict}{\Rightarrow}
\newcommand{\defar}{\rightarrow}
\newcommand{\notdefar}{\not \rightarrow}
\newcommand{\negar}{\Leftrightarrow^{\neg}}
\newcommand{\andar}{\Leftrightarrow^{\wedge}}
\newcommand{\orar}{\Leftrightarrow^{\vee}}
\newcommand{\notLR}{\not\hspace*{-0.5mm}\Leftrightarrow}
\newcommand{\NK}{\N_\KB}
\newcommand{\Phik}{\Phi_\KB}
\newcommand{\Deltak}{\Delta_\KB}


\nd In the \emph{Defeasible Inheritance-based} approach~\cite{CasiniStraccia13} the axioms in the TBox and in the DBox are translated into an inheritance net~\cite{Horty94}. Such a construction allows us to apply the RC procedure locally, in such a way that if we want to decide whether $C\usually D$ holds, the exceptionality ranking and the RC are calculated considering only the information in the KB that has some connection to $C$ and $D$. For more details about the formalisation and inference properties of the approach, we refer the reader to~\cite{CasiniStraccia13}.

\paragraph{Basic notions on Inheritance Nets.} Nevertheless, for the sake of self-containedness, we briefly recap here some salient notions from~\cite{CasiniStraccia13}. 
In \emph{Defeasible Inheritance Nets}, or simply \emph{Inheritance Nets} (INs)~\cite{Horty94} there are classes (\emph{nodes}), a strict subsumption relation and a defeasible subsumption relation among such classes (\emph{links}). An IN is a pair $\N=\tuple{S,D}$, where $S$  is a set of \emph{strict links}, while $D$  is a set of \emph{defeasible links}. Every link in $\N$ is  a \emph{direct} link, and it can be strict or defeasible, positive or negative. Specifically,
\begin{enumerate}
\item  $C\strict D$: class $C$ is subsumed by class $Q$ [positive strict link];

\item $C\notLR D$: class $C$ and class $Q$ are disjoint [negative strict link];

\item $C\defar D$: an  element of the class $C$ is usually an  element of the class $D$ [positive defeasible link];

\item $C\notdefar D$: an  element of the class $C$ is usually not an  element of the class $Q$ [negative defeasible link].
\end{enumerate}


\begin{definition}[Course, Definition 3.1 in~\cite{CasiniStraccia13}]\label{defcourse}
\emph{Courses} are defined as follows (where $\star \in \{\strict,\notLR,\defar,\notdefar\}$):
\begin{enumerate}
\item every link $C\star D$ in $\N$ is a course $\pi=\tuple{C,D}$ in $\N$; and
\item if $\pi=\tuple{\sigma,C}$ is a course and  $C\star D$ is a link in $\N$ that does not already appear in $\pi$, then $\pi'=\tuple{\pi,D}$ is a course in $\N$.
\end{enumerate}
\end{definition}

\nd Roughly, courses are simply routes on the net following the direction of the arrows, without considering if each of them is a positive or a negative arrow.

In~\cite{CasiniStraccia13}, INs have been extended to, so-called \emph{Boolean Inheritance Nets} (BINs), that allow additionally to represent the negation, conjunction and disjunction of nodes. For what concerns us here, given nodes $C,D$ and $E$, a \emph{conjunction link} is of the form $C,D\andar E$ (read as the conjunction of $C$ and $D$ is equivalent to $E$). We will assume that inheritance nets containing such links are closed according to the following rule: if there is $C,D\andar E$ in a net, then there are also 
$E \strict C$ and $E \strict D$ in the net. Furthermore, the notion of courses is extended to BINs, calling them \emph{ducts}~\cite[Definition 3.2]{CasiniStraccia13}. That is, we consider not only `linear' routes from one node to another, but also `parallel' routes, in order to model the introduction of the conjunction. Roughly, 
\[
\pi=\tuple{C,\frac{\sigma}{\sigma'},D}
\]
\nd will indicate a duct $\pi$ that starts at node $C$ and develops through the ducts $\sigma$ and $\sigma'$, both reaching the node $D$.

\begin{definition}[Duct, Definition 3.2 in~\cite{CasiniStraccia13}]\label{duct}
\emph{Ducts} are defined as follows (where $\star \in \{\strict,\notLR,\defar,\notdefar\}$):
\begin{enumerate}
\item every link $C\star D$ in $\N$ is a duct $\pi=\tuple{C,D}$ in $\N$;
\item if $\pi=\tuple{C,\sigma,D}$ is a duct and  $D\star E$ is a link in $\N$ that does not already appear in $\sigma$, then $\pi'=\tuple{C,\sigma,D,E}$ is a duct in $\N$;
\item if $\pi=\tuple{C,\sigma,D}$ is a duct and  $E\star C$ is a link in $\N$ that does not already appear in $\sigma$, then $\pi'=\tuple{E,C,\sigma,D}$ is a duct in $\N$;
\item if $\tuple{C,\sigma,D}$ and $\tuple{C,\sigma',E}$ are ducts and  $D,E \andar F$ is a link in $\N$ that does not already appear in $\tuple{C,\sigma,D}$ and in $\tuple{C,\sigma',E}$, then $\tuple{C,\frac{\sigma,D}{\sigma',E},F}$ is a duct. \qed

\end{enumerate}
\end{definition}


\paragraph{A Decision procedure for INs-based \ELbot.} 
Now, the adaptation of the inheritance-based decision procedure~\cite[Section 5]{CasiniStraccia13} to $\ELbot$ can be formalised in the following way.
First of all, we assume that a KB $\tuple{\T,\D}$ has already been transformed into \emph{normal form}, as discussed in Section~\ref{normalforms}. Then, we create an inheritance net $\N_\KB$ representing the content of the KB. The procedure is essentially the one in~\cite{CasiniStraccia13}, just constrained to \ELbot. 

That is, given a KB $\tuple{\T,\D}$:
\begin{enumerate}
\item for every atomic concept appearing in the axioms in $\KB$, we create a corresponding node in the net;

\item for every axiom $A \usually B \in\D$, we add the defeasible link $A\defar B$ to the net;

\item for every axiom $A\dlAnd B\subs\bot\in\T$ we add the symmetric incompatibility link $A\notLR B$ to the net;

\item for all the remaining axioms (in normal form) $C\subs D\in\T$, we introduce (if not already present) two nodes in the net representing the concepts $C$ and $D$, respectively, and add the strict connection $C\strict D$ to the net;

\item we then complete the inheritance net $\N_\KB$ doing a total classification of the concepts appearing as nodes in the net. That is, for concepts $C,D$, where $C$ and concept $D$ are either atomic concepts or of the form $\exists r.F$, and $E$ (not being $\top)$, that have a node representation in $\N_\KB$, if $\T\entails C\dlAnd D\subs E$ holds then 
we add also a link 
$C,D \andar E$ to the net $\N_\KB$.
\end{enumerate}
\nd Note that nodes in $\N_\KB$ represent concepts that are either $\bot,\top$, atomic, of the form $\exists r.F$ or the conjunction of two atomic concepts.

%

Now, the  procedure for the closure of a KB $\tuple{\T,\D}$ using the Defeasible Inheritance-based approach, adapted to $\ELbot$, is as follows~\cite{CasiniStraccia13}:
\begin{description}
\item[Step 1.] Given $\KB=\tuple{\T,\D}$, check if $\KB$ has a ranked model  (see Corollary \ref{consistency} and Remark \ref{reelbotcla}). If yes, then define an inheritance net $\NK$ from $\KB$, as illustrated before.

\item[Step 2.] Set $\D_{in}=\D$.
For every pair of nodes $\tuple{C,D}$ such that $C$ and $D$ appear in the net $\NK$, do the following:

	\begin{itemize}
	
	\item determine the set $\Delta_{C,D}$ of defeasible links $E\defar F$ appearing in a duct from $C$ to $D$;
	
	\item consider the KB $\KB'=\tuple{\T,\D'}$, where $\D'\subseteq\D$ is the set of the defeasible axioms corresponding to the defeasible links in $\Delta_{C,D}$;
	
	\item if $C\usually D$ is in the RC of $\KB'$ (\ie~$\KB'\ratent C\usually D$), then add $C\usually D$ to $\D_{in}$.
	\end{itemize}

\item[Step 3.] Finally, let $\KB_{in}=\tuple{\T,\D_{in}}$. We can use the decision procedure for RC on $\KB_{in}$,  defining the non-monotonic inference relation $\vdash_{in}$
as
\[
\KB\vdash_{in}C\usually D\text{ iff } \KB_{in}\ratent C \usually D\ .
\]
\end{description}
\nd The above steps are implemented in procedure  $\mathtt{InheritanceBasedRationalClosure}$.

 \begin{procedure}[t]
\caption{InheritanceBasedRationalClosure($\KB,\alpha$) \label{INRCalg}}
{\small
\KwIn{Ontology $\KB$ and defeasible $\ELbot$ axiom $\alpha$}
\KwOut{$\mathtt{true}$ iff $\KB\vdash_{in} \alpha$}
$\tuple{\KB,\alpha}$:=$\mathtt{Normalise}(\KB,\alpha)$ //normalise both $\KB$ and $\alpha$\;
$CTD$ : = $\mathtt{RationalClosure}(\KB,\top \usually \bot)$ //Check if $\KB$ rank unsatisfiable\;
  \If{CTD}{\Return{$CTD$}}
$\NK$:= $\mathtt{BuildInheritanceNet}(\KB)$ //Build inheritance net $\NK$ from $\KB$\;
$\D_{in}$:= $\D$\;
 \ForEach{$C,D\in\NK$}{	
 $\Delta_{C,D}$:= $\{ E\star F \mid E\star F \mbox{ defeasible link occurring in a duct from $C$ to $D$}\}$\;
 $\D'$ := $\{E \usually F \in \D \mid E\defar F \in  \Delta_{C,D} \} \cup \{E \usually \neg F \in \D \mid E\notdefar F \in  \Delta_{C,D} \}$ \;
 $\KB'$ := $\tuple{\T,\D'}$\;
  \If{$\mathtt{RationalClosure}(\KB',C\usually D)$}{$\D_{in}$:= $\D_{in} \cup \{ C\usually D\}$}
 }
 // Check now if $\alpha$ holds under Inheritance-Based RC\;
 $\KB_{in}$ := $\tuple{\T,\D_{in}}$\;
 $IRC$ : = $\mathtt{RationalClosure}(\KB_{in},\alpha)$\;
\Return{$IRC$}\;
}
\end{procedure}

\begin{remark} \label{clar}
One may wonder whether the third item in Step 2 may be replaced with the simpler form
\begin{itemize}
\item[(*)] if $C\usually D$ is in $\D'$ then add $C\usually D$ to $\D_{in}$.
\end{itemize}

\nd Unfortunately, this choice does not provide the intended behaviour as the following variant of the ``penguin" example illustrates (see also~\cite[Example 3.1]{CasiniStraccia13}). Consider $\K = \tuple{\T,\D}$ with\footnote{$\dl{P,B,W,F}$ stand for Penguin, Bird, Wing and Fly, respectively.}
\begin{eqnarray*}
\T & = & \{\dl{F} \dlAnd \dl{NF} \subs \bot \\ 
\D & = & \{\dl{P} \usually \dl{B}, \dl{B} \usually  \dl{F}, \dl{P} \usually  \dl{NF}, \dl{P} \usually  \dl{W} \}
\end{eqnarray*}
\nd and corresponding inheritance net:

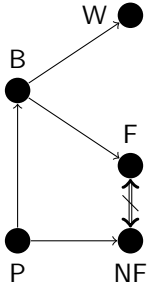
\begin{center}
\begin{tikzpicture}[transform shape]
\coordinate [circle, fill=black, label=above:$\dl{B}$] (b) at (1.5,1);
\coordinate [circle, fill=black, label=below:$\dl{P}$] (p) at (1.5,-1);
\coordinate [circle, fill=black, label=below:$\dl{NF}$] (n) at (3,-1);
\coordinate [circle, fill=black, label=above:$\dl{F}$] (f) at (3,0);
\coordinate [circle, fill=black, label=left:$\dl{W}$] (w) at (3,2);

\foreach \from/\to in {b/f, b/w, p/b, p/n}
\draw [->] (\from) -- (\to);
\draw[<->,strike through, double] ( n) -- ( f);
\end{tikzpicture}

\end{center}

%
%
%

\nd It then can be shown  (as desired) that
\[
\KB\vdash_{in} \dl{P}\usually \dl{W} \ ,
\]
\nd In fact, as $\Delta_{\dl{P} ,\dl{W}} = \{ \dl{P}\usually \dl{B}, \dl{B}\usually \dl{W}\}$, by item three in Step 2., we add $\dl{P}\usually \dl{W}$ to $\D_{in}$, which is not the case if $(*)$ is used instead. Therefore, as $\dl{P}$ is an exceptional subclass of $\dl{B}$, we have
\begin{eqnarray*}
\KB & \not\ratent  & \dl{P}\usually \dl{W} \\
\KB & \not \vdash_{in} &\dl{P}\usually \dl{W} \ , \mbox{ if using $(*)$ instead} \ .
\end{eqnarray*}
\end{remark}

\nd The correspondence of this procedure to the more general \ALC~procedure in~\cite{CasiniStraccia13} is guaranteed by Proposition~\ref{completeness}, proving that the procedure $\mathtt{RationalClosure}$ is correct and complete \wrt~RC, and the fact that the present definition of \emph{ducts} is just the \ELbot~restriction of the more general definition for \ALC~(\cite{CasiniStraccia13}, Definition 3.2).

\begin{example}[Example~\ref{exnorm02} cont.]\label{INex}
Consider Example \ref{exnorm02}. The  inheritance net $\NK$ built from $\KB$ is illustrated in Figure~\ref{net}.
\begin{figure}[h]
\begin{center}

\begin{tikzpicture}[transform shape]
\coordinate [circle, fill=black, label=above:$\dl{BRBC}$] (g) at (-3,0);
\coordinate [circle, fill=black, label=above:$\dl{MRBC}$] (a) at (-1,0);
\coordinate [circle, fill=black, label=below:$\dl{VRBC}$] (b) at (1,0);
\coordinate [circle, fill=black, label=above:$\dl{ARBC}$] (e) at (1,2);

\coordinate [circle, fill=black, label=above:$\exists\dl{hasCM}.\top$] (h) at (5.5,1.5);
\coordinate [circle, fill=black, label=below:$\exists\dl{hasN}.\top$] (i) at (5.5,-1.5);

\coordinate [circle, fill=black, label=above:$\dl{A_2}$] (d) at (3.5,1.5);
\coordinate [circle, fill=black, label=below:$\dl{A_1}$] (c) at (3.5,-1.5);

\coordinate [circle, fill=black, label=below:$\dl{NotN}$] (f) at (1.5,-1.5);

\foreach \from/\to in {b/c,b/d,a/f}
\draw [->] (\from) -- (\to);
\draw [<->, strike through,double] (f) -- (c);
\draw [<->,double] (d) -- (h);
\draw [<->,double] (c) -- (i);

\draw [->, double] (a) -- (b);
\draw [->, double] (e) -- (b);
\draw [->, double] (g) -- (a);
\end{tikzpicture}
\end{center}
\caption{Inheritance net built from Example~\ref{INex}.} \label{net}
\end{figure}
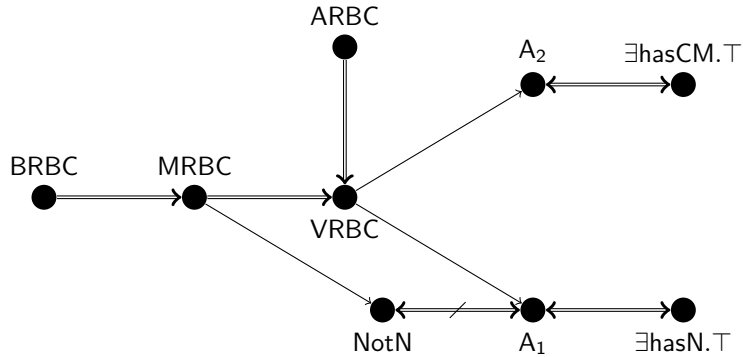
\nd We want to check whether $\KB\vdash_{in}\dl{MRBC}\usually\exists\dl{hasCM}.\top$ holds. To do so, we compute all  ducts between $\dl{MRBC}$ and $\dl{hasCM}.\top$.
There is only one and it includes the defeasible link representing $\dl{VRBC}\usually\dl{A_2}$.
Therefore, in  Step 10 of the $\mathtt{InheritanceBasedRationalClosure}$ procedure, we set
$\D'=\{ \dl{VRBC}\usually \dl{A_2}\}$.
It can be verified that $\dl{MRBC}\usually \dl{A_2}$ is in the RC of $\KB' = \tuple{\T,\D'}$ and, thus, by Step 12,
$\dl{MRBC}\usually \dl{A_2}$ is added  to $\D_{in}$, from which $\KB\vdash_{in} \dl{MRBC}\usually\exists\dl{hasCM}.\top$ follows.

Next, we want also to show that indeed $\KB \not \vdash_{in} \dl{MRBC}\usually\exists\dl{hasN}.\top$. Now, observe that for nodes
$C \in \{ \dl{MRBC} \}$ and $D \in\{\dl{A_1}, \dl{hasN}.\top\}$, $\D'$ in line 9 of the procedure is
\begin{eqnarray*}
\D' & = & \{\dl{MRBC} \usually \dl{hasN.\top}, \dl{MRBC} \usually  \dl{NotN} \} \ .
\end{eqnarray*}
\nd Therfore, for $D \in\{\dl{A_1}, \dl{hasN}.\top\}$
\begin{eqnarray*}
\KB' & \not \ratent & \dl{MRBC}\usually D \ ,
\end{eqnarray*}
\nd holds and, thus, $ \dl{MRBC}\usually D$ is not added to $\D_{in}$.  Consequently, for $D \in\{\dl{A_1}, \dl{hasN}.\top\}$
\begin{eqnarray*}
\KB & \not \vdash_{in} & \dl{MRBC}\usually D 
\end{eqnarray*}
\nd holds, as expected.
\qed
\end{example}
\paragraph{Computational Complexity.} Let us now address the computational complexity of the \\ $\mathtt{InheritanceBasedRationalClosure}$ procedure. To start with, note that the normalisation Step 1 can be done in linear time~\cite{Baader05a}, yielding a normalised KB whose size is linear in the size of $\KB$.
From the procedure described in Section~\ref{ratclosEL}, it follows that Step 2 of  $\mathtt{InheritanceBasedRationalClosure}$ requires at most $\mathcal{O}(|\D|^3 + |\D|)$ subsumption tests.

Let us now estimate the cost of Step 5.
The first step of the construction of $\NK$ creates as many nodes as there are atomic concepts in $\KB$ and, thus, both size and time bound are
$\mathcal{O}(|\KB|)$. The second step is obviously bounded both in time and size of number of edges by $\mathcal{O}(|\D|)$. Similarly, the third  and fourth step can be done together and are bounded both in time and size of number of new nodes and edges by $\mathcal{O}(|\T|)$. Therefore, Steps 1 - 4 can be done in time $\mathcal{O}(|\KB|)$, while for network $\NK$,  the number of both nodes and edges is bounded by $\mathcal{O}(|\KB|)$. Eventually, it can easily be verified that, including the fifth step of the construction of $\NK$, the overall time bound of  the construction of $\NK$ is time and size (number of nodes as well as edges) bounded by $\mathcal{O}(|\KB|^2)$.

The number of iterations of Steps 7 - 12 is bounded by $\mathcal{O}(|\KB|^4)$, as there are at most $\mathcal{O}(|\KB|^2)$ nodes in $\NK$. Step~8
is the same step as illustrated in~\cite{CasiniStraccia13},  in which it is shown that $\Delta_{C,D}$ can be computed in polynomial time \wrt~the size of $\KB$ 
\cite[Sections 3.1.6 and 3.2.5]{CasiniStraccia13}.
For ease of presentation, let us indicate with $\delta_\KB$ the time bound to compute $\Delta_{C,D}$.
Moreover, the size of both $\Delta_{C,D}$ and $\D'$ is bounded by $\mathcal{O}(|\D|)$ and, thus, the size of $\KB'$ is bounded by $\mathcal{O}(|\KB|)$. Therefore, Step 12 requires at most $\mathcal{O}(|\D|^3 + |\D|)$  subsumption tests. In summary, the computation time of iterations Steps 7 - 12 is bounded by $\mathcal{O}(\delta_\KB \cdot |\KB|^4)$ plus the time required to perform $\mathcal{O}(|\KB|^4| (|\D|^3 + |\D|))$ subsumption tests.

Finally, the size of $\D_{in}$ is bounded by $\mathcal{O}(|\KB|^2)$ and, thus, Step 15 may require at most $\mathcal{O}(|\KB|^6 + |\KB|^2)$ subsumption tests.
Therefore, we have the following result.
\begin{proposition}\label{complexityINRC}
Procedure $\mathtt{InheritanceBasedRationalClosure}$, that decides whether a defeasible inclusion axiom $C\usually D$ is in the Inheritance-based RC of $\KB=\tuple{\T,\D}$,
\ie~decides whether $\KB\vdash_{in} C \usually D$ holds, runs in polynomial time \wrt~the size of $\KB$.
\qed
\end{proposition}

\section{Rational Closure for Nominal Safe \ELObot}   \label{rcelo}
%
\nd In order to reason about individuals in a defeasible framework, more than one option is possible. Here we present an idea in which we use concept  nominals, \ie~concepts of the form $\{a\}$. There are other  options as well developed in a formal framework analogous to the one we are going to present here and compatible with it~\cite{BritzEtAl2015a,CasiniStraccia2010,CasiniStraccia13}. In~ \cite{Giordano13,Giordano15} other ways of reasoning about individuals in the preferential framework has been proposed, while in~\cite{Bonatti15} 
 a different kind of semantics is introduced.

%
Recall that by using nominals we can easily translate an assertion as a GCI as pointed out in Remark~\ref{aboxA}, and that the standard interpretation of nominals is a singleton from the domain.

Our approach to deal with nominals in defeasible \ELObot~starts with the possibility to express defeasible information about individuals by using axioms of the form $\{a\}\usually C$ (read as ``\emph{presumably,  individual $a$ is an instance of concept $C$}'') and $\{a\}\usually \exists r.\{b\}$ (``\emph{presumably, individual $a$ is connected via role $r$ to individual $b$}'').  
Since we are working in a defeasible framework, we would like to apply the \emph{presumption of typicality} (see Section \ref{alc}) also to  individuals. That is, we want to reason under the assumption that each individual behaves in the most typical possible way, compatibly with the information at our disposal. 

A main problem in enforcing the typicality of individuals is that role connections  may possibly create situations in which the defeasible information associated with an individual prevents it from associating some defeasible information to another individual, and vice-versa, as illustrated in the next example.
\begin{example}\label{multipleextensions}
Consider the following defeasible $\ELObot$ KB $\KB=\tuple{\T,\D}$ with 
\begin{eqnarray*}
\T & =  & \{\{a\}\subs \exists r.\{b\},C\dlAnd D\subs \bot\}\\
\D & = & \{\top\usually C,\exists r.C\usually D\} \ .
\end{eqnarray*}
\nd By applying the ranking procedure to $\KB$, it can be shown  that $\D$ is partitioned  into 
\begin{eqnarray*}
\D_0 & = & \{\top\usually C\} \\
\D_1 & = & \{\exists r.C\usually D\} \ . 
\end{eqnarray*}
\nd Now,  let's try to create a minimal model for this KB, in order to maximise the typicality also of the individuals (see Section \ref{alc}). If we position the interpretation of individual  $a$ in  layer $0$,\footnote{See the definition of \emph{layers} at page~\pageref{layers}.} then both $\top\usually C$ and $\exists r.C\usually D$ may be `applied' to $a$ and, thus, $\cass{a}{C}$ holds. But then we are forced to position $b$ on layer $1$ to impose that  $\cass{b}{C}$ does not hold, as otherwise we would have also $\cass{a}{\exists r.C}$ and $\cass{a}{D}$, and then we would be forced to conclude $\cass{a}{\bot}$ (since $C\dlAnd D\subs \bot$).  On the other hand, if we put $b$ on layer $0$ we would be forced to put $a$ on layer $1$ to impose that $\cass{a}{C}$ does not hold, again to avoid $\cass{a}{\bot}$ to hold. Therefore, we have at least two (minimal) ranked models.

Note that if we would apply  the so far developed RC procedure  to $\ELObot$ (see page~\pageref{algoRCalc}) we would infer instead that both 
$\{a\} \usually C$ and $\{b\} \usually C$ hold, and, thus, $\{a\} \usually D$ would follow, from which we would conclude $\{a\} \usually \bot$. I.e., we get an inconsistency. 

Eventually, note that we face with the same problem with the following KB $\KB'=\tuple{\T',\D'}$:
\begin{eqnarray*}
\T' & =  & \{A = \{a\}, B = \{b\}, A \subs \exists r.B, C\dlAnd D\subs \bot\}\\
\D' & = & \{\top\usually C,\exists r.C\usually D\} \ .
\end{eqnarray*}
\nd That is, not only individual $a$ (resp.~$b$) does not have the same rank among all   ranked models of $\KB'$, but the same applies to atomic concept $A$ (resp.~$B$).
\qed
\end{example}
\nd The multiplicity of incompatible minimal configurations  allows us to make a distinction between two distinct kinds of presumptive reasoning.
 Roughly, on one hand we have \emph{presumptive} reasoning, that can be modelled taking under consideration \emph{all} the most expected situations: an individual $a$ presumably satisfies a property $C$ iff $a^R\in C^R$ in \emph{all} the minimal models $R$ (of an ontology $\KB$).  On the other hand we have \emph{prototypical} reasoning, where we associate a property $C$ to an individual $a$ if all the preferred instantiations of $a$ fall under $C$; that is, 
a property $C$ belongs to the \emph{prototype} of an individual $a$ if $a^R\in C^R$ for every minimal model $R$ of $\KB$ in which $a^R$ is in the lowest possible rank.\footnote{A similar distinction was also made by Lehmann in the framework of propositional logic~\cite{Lehmann1995}. } 
We refer the reader to \ref{skept-cred} (Definitions~\ref{def7}, \ref{def8} and~\ref{def_nominalsafe_ent}) for a formal distinction among the two above mentioned approaches.
 
%
%
%
%


While the presumptive reasoning solution can be modelled following previous proposals~\cite{BritzEtAl2015a,CasiniStraccia13}, no polynomial time algorithm of its application to \ELObot~seems to exist so far (see also Section~\ref{related}).

In this work,  we rather investigate the second option. In order to model this kind of reasoning,  we propose to introduce a change in the interpretation of nominals. Our approach not only will allow us to model prototypical reasoning, but provides us also with a simple way to use the decision procedure presented so far. An important side effect is that we can preserve the polynomial time result for the subsumption problem. 

Specifically, our intuition is that  reasoning about individuals in a defeasible framework means to assign to them properties that we only \emph{presume} to hold. 
In order to model such a semantics, we interpret each individual (\ie, each nominal) not as a singleton, but as a \emph{set of objects}: each of them represents a possible instantiation of the individual, considering, on the  one hand, the information at our disposal, and, on the other, different possible levels of typicality. In this latter case, this corresponds to consider different rankings of the individual (kind of `set of prototypes of an individual).

\paragraph{The DL \defELObot.} Formally, we introduce a new kind of nominal, referred to as \emph{defeasible nominal}, and denoted  with $\defnom{a}$, where $a\in\ON$. The interpretation function $\cdot^\I$ of an interpretation $\I$ is extended to defeasible nominals by mapping them into sets, \ie
\begin{itemize}
\item $\defnom{a}^\I\subseteq\Delta^\I$ for every individual $a\in\ON$.
\end{itemize}
The defeasible version of $\ELObot$, referred to as $\defELObot$, is $\ELbot$ extended with defeasible nominals:
\[
C,D \rightarrow A\mid\top\mid\bot\mid C\dlAnd D\mid \exists  r.C \mid \defnom{a} \ .
\]

\nd Our reasoning mechanism for $\defELObot$ under RC remains  the same as \ELbot~under RC (and the extension of RC we have investigated in this work),  simply by \emph{treating defeasible nominals as new classical atomic concepts in our procedures}.\footnote{The interpretation of defeasible nominals as classical atomic concepts implies that some results required here and valid for \ELbot  extend to $\defELObot$ too: namely, the FMP and FCMP properties defined for defeasible \ALC; also, the procedures $\mathtt{Exceptional}$, 
$\mathtt{ComputeRanking}$,$\mathtt{RationalClosure}$ and $\mathtt{InheritanceBasedRationalClosure}$ remain the same. } 
Therefore, we will continue to use the entailment notation $\ratent$ also for the extension to defeasible nominals of our entailment relation for RC. 

The following example illustrates the behaviour of the RC procedure applied to \defELObot.

\begin{example}
Consider the ontology in Example \ref{ex02} extended with the axioms $\defnom{a}\subs \dl{BRBC}$ and $\defnom{b}\subs \dl{ARBC}$, dictating that we know that a specific cell $a$ is a red blood cell from a cow, while $b$ refers to a red blood cell from a bird. Now, it can be shown that, by applying the ranking procedure, we end up with the same partition of $\D$ as in Example \ref{exranking}. That is
\begin{eqnarray*}
 \T^* & = & \{ \ \   \defnom{a}\subs \dl{BRBC}, \defnom{b}\subs \dl{ARBC}, \dl{BRBC}\subs\dl{MRBC},\\
  &&\dl{ARBC}\subs\dl{VRBC}, \dl{MRBC}\subs\dl{VRBC}, \exists\dl{hasN}.\top\dlAnd \dl{NotN}\subs\bot \ \  \} \\
  \D_0  & = &  \{ \ \  \dl{VRBC}\usually\exists\dl{hasCM}.\top,  \dl{VRBC}\usually\exists\dl{hasN}.\top \ \   \}  \\
  \D_1  & = &  \{ \ \  \dl{MRBC}\usually \dl{NotN} \ \   \}  \ .
\end{eqnarray*}
Moreover, as for Example \ref{exrankingA}, we end up with the following entailments:
\begin{eqnarray*}
\T^*_{\delta_0}  & \entails  & \defnom{a}\dlAnd\delta_0\subs\bot\\
\T^*_{\delta_0} & \not\entails & \defnom{b}\dlAnd\delta_0\subs\bot \\
\T^*_{\delta_1} & \not\entails & \defnom{a}\dlAnd\delta_1\subs\bot \ .
\end{eqnarray*}
\nd Hence we have that defeasible nominal $\defnom{a}$ has rank 1 (the most prototypical instantiation of $a$ must be of rank 1), while $\defnom{b}$ has rank 0 (it is conceivably an instantiation of $b$ in the lowest layer). 

We can now check whether $\defnom{a}$ and $\defnom{b}$ have a nucleus, following the same procedure as in Example \ref{exdecision}. Since
\begin{eqnarray*}
\T^*_{\delta_1} &\not\entails & \defnom{a}\dlAnd\delta_1\subs\exists\dl{hasN}.\top\\
\T^*_{\delta_1} & \entails  & \defnom{a}\dlAnd\delta_1\subs\dl{notN} \\
\T^*_{\delta_0} & \entails  & \defnom{b}\dlAnd\delta_0\subs\exists\dl{hasN}.\top \ ,
\end{eqnarray*}
\nd we can conclude now that presumably the cell $a$ does not have a nucleus (\ie, $\defnom{a}\usually\dl{notN}$), while $b$ does (\ie, $\defnom{b}\usually\exists\dl{hasN}.\top$).
\qed
\end{example}

\nd Such an approach carries the risk to be incompatible with the classical part of our reasoning, where nominals are assumed to be interpreted as singletons. Yet, we will show that the results in~\cite{Kazakov14}, reported here as Propositions \ref{propelo} and \ref{propeloB}, justify our  reasoning about nominals by interpreting them as concepts, provided the \emph{safeness} condition is satisfied. That is, for nominal safe ontologies, the strict part of our knowledge behaves exactly as if we would have interpreted the nominals as individuals (see Proposition~\ref{use_concept} later on).  

To start with, given an $\defELObot$ concept (resp.~GCI, ontology) $\alpha$, we indicate with $\alpha_{\{\}}$ the $\ELObot$ concept (resp.~GCI, ontology) obtained by substituting every occurrence of each defeasible nominal $\defnom{a}$ with a nominal $\{a\}$. Vice-versa,  given an $\ELObot$ concept (resp.~GCI, ontology)
$\alpha$, we indicate with $\alpha_{\tiny\defnom{}}$ the $\mathcal{EL\defnom{O}}_\bot$ concept (resp.~GCI, ontology) obtained by substituting every nominal $\{a\}$ by $\defnom{a}$. Now, it is easy to see that one transformation is the inverse of the other. That is, for every $\defELObot$ expression $\alpha$, $\alpha=(\alpha_{\{\}})_{\tiny \defnom{}}$, and for every $\mathcal{EL{O}}_\bot$ expression $\alpha$, $\alpha=(\alpha_{\tiny \defnom{}})_{\{\}}$. 

The notion of \emph{safeness} (see Section \ref{EL}) is naturally extended to $\defELObot$:
\begin{itemize}
\item $C$ is \emph{safe} iff $C_{\{\}}$ is safe;
\item $C$ is \emph{n-safe} iff $C_{\{\}}$ is n-safe;
\item $C\subs D$ is \emph{safe} iff $(C\subs D)_{\{\}}$ is safe;
\item $C\usually D$ is \emph{safe} iff $C$ is n-safe and $D$ is safe. 
\end{itemize}
\nd Finally, a defeasible $\defELObot$ ontology $\KB=\tuple{\T,\D}$ is \emph{nominal safe} if it contains only safe axioms. 

Next, we address two aspects of using defeasible nominals that prove that using them instead of standard nominals does not change the classical part of our reasoning, provided that the ontology is nominal safe. 

The first observation is about the ranking procedure. With respect to this point, we can naturally extend the notion of exceptionality to  nominals as follows.
\begin{itemize}
\item A nominal $\{a\}$ is \emph{exceptional} \wrt~an $\mathcal{EL{O}}_\bot$ ontology $\KB=\tuple{\T,\D}$ iff, for every ranked model $R\in\R^\KB$, $\{a\}^{R} \cap\min_{\prec^R}(\Delta^{R}) = \emptyset$.

\item A defeasible nominal $\defnom{a}$ is \emph{exceptional} \wrt~an $\defELObot$ ontology $\KB=\tuple{\T,\D}$ iff, for every ranked model $R\in\R^\KB$, $\defnom{a}^{R}\cap\min_{\prec^R}(\Delta^{R})=\emptyset$.
\end{itemize}
\nd The respective notion of ranking $r_\KB$ follows the same definition from Section \ref{alc}. 

Now, we prove the correctness of the ranking procedure $\mathtt{ComputeRanking}$ extended to $\defELObot$ by treating defeasible nominals as atomic concepts. In particular, we prove that the axioms that have infinite rank \wrt~an $\defELObot$~ontology correspond to the axioms that have infinite rank \wrt~the correspondent $\mathcal{EL{O}}_\bot$~ontology. 
Specifically, the following propositions can be shown.\footnote{See the RC ranking procedure at page~\pageref{classicalRC} for the definition of the rank $r_{\KB}(C)$ of a concept $C$ \wrt~a knowledge base $\KB$, while see page~\pageref{defN} for the definition of $N(\cdot)$.}
\begin{proposition}\label{ranking_defnom}
For every  n-safe $\mathcal{EL{O}}_\bot$ concept $C$ and every nominal safe $\mathcal{EL{O}}_\bot$ ontology $\KB=\tuple{\T,\D}$, the following holds: $r_\KB(C)=i$ iff $r_{\KB_{\tiny \defnom{}}}(C_{\tiny \defnom{}})=i$. 
\qed
\end{proposition}
\begin{proposition}\label{ranking_defnomA}
For every  n-safe $\defELObot$ concept $C$ and every nominal safe $\defELObot$ ontology $\KB=\tuple{\T,\D}$, 
$r_{\KB}(C)=i$ iff $r_{N(\KB_{\tiny \{\}})}(N(C_{\tiny \{\}}))=i$. 
\qed
\end{proposition}

\nd From Propositions~\ref{ranking_defnom} and~\ref{ranking_defnomA} it follows that the information of infinite rank \wrt~an $\mathcal{EL{O}}_\bot$ defeasible ontology $\KB$ corresponds exactly to the information of infinite rank \wrt~$\KB_{\tiny \defnom{}}$. 
That is, 
\begin{proposition}\label{use_conceptA}
Let $\KB=\tuple{\T,\D}$ be a nominal safe $\defELObot$ ontology, with $C\usually D\in \D$.  Then, the following are equivalent:
\begin{enumerate}
\item  $r_\KB(C\usually D)=\infty$; 
\item $r_{\KB_{\{\}}}(C_{\{\}}\usually D_{\{\}})=\infty$; 
\item  $r_{N(\KB_{\{\}})}(N(C_{\{\}}\usually D_{\{\}}))=\infty$. \qed
\end{enumerate}
\end{proposition}

\nd An immediate consequence is that we can safely use 
$\mathtt{ComputeRanking}$ over $\mathcal{EL\defnom{O}_\bot}$ ontologies to move the axioms with infinite rank into the TBox.
\begin{corollary}\label{use_concept_coroll}
Let $\KB=\tuple{\T,\D}$ be a nominal safe $\defELObot$ ontology, with $C\usually D\in \D$. Let  $\T^\bullet$ be  the TBox 
 obtained from $\KB$ after applying {\bf RC.Step 1}, while let $(N(\T_{\{\}}))^*$ be the TBox obtained from $N(\KB_{\{\}})$ using the
$\mathtt{ComputeRanking}$ procedure. Then, the following are equivalent:
\begin{enumerate}
\item   $C\usually \bot\in \T^\bullet$; 
\item $N(C_{\{\}}\usually \bot)\in (N(\T_{\{\}}))^*$; 
\item $r_{\KB_{\{\}}}(C_{\{\}}\usually D_{\{\}})=\infty$. \qed
\end{enumerate}
\end{corollary}
\nd Propositions~\ref{ranking_defnom}-\ref{use_conceptA} guarantee that the ranking of concepts (and nominals) done using an $\defELObot$ ontology corresponds to the ranking over the correspondent 
$\mathcal{EL{O}}_\bot$ ontology, while Corollary~\ref{use_concept_coroll} guarantees that we are correct in using  the $\mathtt{ComputeRanking}$ procedure to distinguish the classical part of an ontology from the defeasible part (provided nominal safeness is assumed). 

The next step is now to prove that  the classical conclusions that we derive using defeasible nominals correspond to the conclusions we would derive using classical nominals. So, let $\KB_{\{\}}=\tuple{\T_{\{\}},\D_{\{\}}}$ be an $\mathcal{ELO}_\bot$ ontology and let $\D^{\subs}_{\{\},\infty}:=\{C\subs D\mid C\usually D\in\D_{\{\}} \text{ and } r_{\KB_{\{\}}}(C\usually D)=\infty\}$; hence ${\T}_{\{\}}\cup\D^{\subs}_{\{\},\infty}$ defines the classical part of the information contained in $\KB_{\{\}}$. The following holds.
\begin{proposition}\label{use_concept}
Let $\KB=\tuple{\T,\D}$ be a nominal safe $\defELObot$ ontology. 
For every \emph{safe}  GCI $C\subs D$, we have that the following are equivalent:
\begin{enumerate}
\item  $C\subs D$ is in the RC of $\KB$; 
\item $N(C_{\{\}} \subs D_{\{\}})$ is in the RC of $N(\KB_{\{\}})$;
\item ${\T}_{\{\}}\cup\D^{\subs}_{\{\},\infty}\entails C_{\{\}}\subs D_{\{\}}$. \qed
\end{enumerate}
\end{proposition}
\nd This proves that, under the nominal safeness assumption (i) entailments about  classical GCIs in  \defELObot are the same as for   $\mathcal{EL{O}_\bot}$; (ii) all the decision procedures we have introduced in the previous sections can be applied without any change to $\defELObot$ ontologies, \emph{provided that we replace defeasible nominals with fresh concept names}; and  (iii) the results connected to the computational complexity  shown in the previous sections remain unchanged. That is, Propositions \ref{complexity1a}, \ref{complexity1},  and \ref{complexityINRC} 
remain valid for nominal safe \defELObot ontologies as well.

In relation to the problematic Example~\ref{multipleextensions}, $\defnom{a}\usually C$ is interpreted as `the most prototypical instantiation of $a$ is an instance of concept $C$', while $\defnom{b}\usually C$ is interpreted as `the most prototypical instantiation of $b$ is instance of $C$'. That is, we are not inferring that it is possible to have a  model where we have at the same time a unique object representing the prototypical $a$ and a unique object representing the prototypical $b$ (\ie, they are both in the lowest layer of the model), but rather that both the most prototypical interpretation of $a$ and the most prototypical interpretation of $b$ are in the layer $0$. The following example illustrates the details.
\begin{example}\label{abox_nominals}
Consider the ontology in Example~\ref{multipleextensions}. Its translation in $\defELObot$ is $\KB_{\tiny \defnom{}}=\tuple{\T_{\tiny \defnom{}},\D_{\tiny \defnom{}}}$ and, thus,
\begin{eqnarray*}
 \T_{\tiny \defnom{}} & = & \{\defnom{a}\subs \exists r.\defnom{b},C\dlAnd D\subs \bot\} \\
\D_{\tiny \defnom{}} & = & \D  =\{\top\usually C,\exists r.C\usually D\} \ .
\end{eqnarray*}
\nd As for $\KB$, $\D_{\tiny \defnom{}}$ is partitioned by the ranking procedure into $\D_0=\{\top\usually C\}$ and $\D_1=\{\exists r.C\usually D\}$. Now, both $\defnom{a}$ and $\defnom{b}$ have rank $0$ as 
\begin{eqnarray*}
\T^*_{\tiny \defnom{},\delta_0}  &\not\entails  & \defnom{a}\dlAnd\delta_0\subs\bot\\
\T^*_{\tiny \defnom{},\delta_0} & \not\entails & \defnom{b}\dlAnd\delta_0\subs\bot 
\end{eqnarray*}
\nd hold. Furthermore, as
\begin{eqnarray*}
\T^*_{\tiny \defnom{},\delta_0}  &\entails  & \defnom{a}\dlAnd\delta_0\subs C\\
\T^*_{\tiny \defnom{},\delta_0} & \entails & \defnom{b}\dlAnd\delta_0\subs C 
\end{eqnarray*}
\nd hold, we can infer both  $\defnom{a}\usually C$ and $\defnom{b}\usually C$. But now, as defeasible nominals are interpreted as atomic concepts, we cannot infer  $\defnom{a}\usually D$ anymore and, consequently, we do not derive $\defnom{a}\usually \bot$. 

Essentially,  since defeasible nominals are interpreted as sets of individuals, we conclude that the most prototypical interpretation of individual $a$ is an instance of $C$, and the most prototypical interpretation of $b$ is instance of $C$ as well. However,  we are not forced to relate through role $r$ the most prototypical interpretation of $a$ to the most prototypical interpretation of $b$: in fact, 
 as $\defnom{a}\subs \exists r.\defnom{b}$ holds, the most prototypical elements of $\defnom{a}$ can be related to exceptional interpretations of $b$. That is, to individuals of $\defnom{b}$ with a rank higher than $0$ and not being an instance of $C$.
 \qed
\end{example}

\nd For further insights, we refer the reader to~\ref{skept-cred} and, specifically, to Proposition~\ref{defind}.

\section{Related Work}\label{related}
%
While there has been extensive work by now related to non-monotonic DLs, 
such as~
\cite{Baader95e,Bonatti09a,Bonatti15,BritzEtAl2008,BritzEtAl2011c,CasiniStraccia2010,CasiniStraccia13,DoniniEtAl2002,Giordano10,Giordano13,Giordano15,Grimm09,Straccia93},
we will focus here on low-complexity non-monotonic DLs extensions only. We refer the reader to \eg~\cite{Bonatti15,Giordano13}, in order to avoid an unnecessary replication, for a general in-depth discussion on non-monotone DLs and their characteristics.
Also, but somewhat less related, are approaches  based on some form of integration of non-monotone logic programs with DLs, 
as~\cite{Eiter11,Ivanov13,Kaminski15,Knorr11,Lukasiewicz07i,Lukasiewicz08b,Lukasiewicz09,Motik10}, although we will not address them here, except for those cases involving tractable computational complexity.


%
We believe that our results are significant. Indeed, let us recap that  non-monotonic DLs extensions involving low-complexity DLs, such as \cite{Bonatti10,Bonatti11b,Bonatti11a,Bonatti15b,BritzEtAl2015a,CasiniStraccia2010,CasiniStraccia13,CasiniEtAl2014,CasiniStraccia14,DoniniEtAl2002,Giordano09,Giordano11,Giordano11a}, only in  few cases the tractability of the underlying DL is preserved~\cite{Bonatti10,Bonatti11b,Bonatti11a,Bonatti15,Giordano18}.
%

Specifically, \cite{Bonatti11a} considers circumscription in low-complexity DLs such as DL-Lite~\cite{Artale09} and $\EL$
and identifies few fragments whose computational complexity of the decision problem at hand is in {\sc PTime}. That is, deciding subsumption, instance checking and concept consistency decision problems are in {\sc PTime} for $Circ_{fix}(LL~\EL)$, \ie ~\emph{Left Local} \ELbot\footnote{Roughly, left-hand sides of GCIs contain no qualified existential restrictions and defeasible axioms are syntactically of the form  $A \usually \exists r.B$.} under circumscription with \emph{fixed}  atomic concepts.
The same holds for $Circ_{var}(\EL)$, where all atomic concepts are \emph{varying} as circumscription collapses to classical \EL~reasoning.
The work~\cite{Bonatti10} follows the same spirit of~\cite{Bonatti11a}, by considering \EL~with \emph{default attributes}, \ie~defeasible GCIs of the form $A \usually \exists r.B$ under circumscription with \emph{fixed}  atomic concepts, but where KBs are assumed to be \emph{conflict safe} and without assertions. In that case, \cite{Bonatti10} shows that one may determine the set of normalised concepts subsuming a given concept in polynomial time. These results have been extended in~\cite{Bonatti11b}, where $\EL^{++}$~\cite{Baader05a} has been considered in place of $\ELbot$, but requiring a further adaption of the conflict safe notion.
In~\cite{Bonatti15}  a non-monotone extension of DLs is proposed based on a notion of overriding and supporting
normality concepts of the form $\mathsf{N}C$ that denote the prototypical/normal instances of a concept $C$. Defeasible GCIs are of the form
$C \usually D$, with a priority relation $\prec$ among such GCIs and intended meaning: ``normally the instances of $C$ are instances of $D$, unless stated otherwise" meaning that higher-priority defeasible GCI may override $C \usually D$. A remarkable feature of this proposal is that the reasoning  problems are tractable for \emph{any}  underlying tractable classical DL, as the reasoning algorithm, as for our case, is built up on a sequence of DL subsumption tests. 


%

In~\cite{Giordano18},  the DL $\mathcal{SROEL}(\dlAnd, \times)^\mathbf{R}\mathbf{T}$ is proposed, \ie~the
 DL $\mathcal{SROEL}(\dlAnd, \times)$~\cite{Kroetzsch10} extended with \emph{typicality} concepts $\mathbf{T}(C)$, whose instances are intended to be the typical $C$ elements. A defeasible GCI   $C \usually D$ can be expressed in  $\mathcal{SROEL}(\dlAnd, \times)^\mathbf{R}\mathbf{T}$  as $\mathbf{T}(C) \subs D$, meaning that ``the typical C elements are Ds", but additionally,  $\mathcal{SROEL}(\dlAnd, \times)^\mathbf{R}\mathbf{T}$ allows to express also  \eg~$D \subs \mathbf{T}(C)$ (``the Cs  are typical elements of Ds").
 %
In summary, in this work, the authors show that 
\begin{enumerate}
\item \label{i1} the instance checking problem under minimal entailment is $\Pi_2^p$-hard for $\mathcal{SROEL}(\dlAnd, \times)^\mathbf{R}\mathbf{T}$. 


\item \label{i2} the instance checking problem under rational entailment and rational closure\footnote{In this case, a knowledge base has to be \emph{simple}, \ie~the typicality operator $\mathbf{T}$ occurs on the left of a GCI only.} is in  {\sc PTime} for $\mathcal{SROEL}(\dlAnd, \times)^\mathbf{R}\mathbf{T}$. The authors propose to extend the Datalog set of rules for  $\mathcal{SROEL}(\dlAnd, \times)$~\cite{Kroetzsch10} with an appropriate set of stratified Datalog rules with negation to deal with the typicality operator.
Additionally, the authors point to the fact that the correspondence between the rational closure construction and the canonical minimal model semantics in~\cite{Giordano15} is lost if \eg~nominals occur in an ontology and, thus, a semantical characterisation is missing so far. 
\end{enumerate}

\nd Let us note that a major difference with our approach is that we may rely entirely on an $\ELbot$ reasoner as black box, while this is not true for the Datalog encoding of $\mathcal{SROEL}(\dlAnd, \times)^\mathbf{R}\mathbf{T}$, as in the latter case there are interacting rules between rules dealing with defeasible knowledge and the set of rules for $\mathcal{SROEL}(\dlAnd, \times)$~\cite{Kroetzsch10}.

%

%
%


%
%

In~\cite{PenselEtAl2017}  an extension of  RC and Relevant Closure~\cite{CasiniEtAl2014} for $\ELbot$ has been proposed, in order to maximise also the typicality of the concepts that are under the scope of an existential quantification. The proposal is in line with the expressivity of the proposal in~\cite{Bonatti15}, uses the ranking procedure defined for \ALC~as presented in Section \ref{alc} and, thus, the procedure does not run in polynomial time.\footnote{By personal communication, we were informed that in~\cite{Pensel17b} a similar construction as the one proposed in Equation~(\ref{Tdelta}) (and, thus, like the one proposed in~\cite{CasiniEtAl2015TR}) to determine exceptionality for  \ELbot~has been described.} 


Concerning the somewhat related `hybrid' approaches, \ie~those that in some way combine non-monotone logic programs with (low-complexity) DLs, there are some results that still preserve tractability (but only \wrt~data complexity), such as~\cite{Eiter11,Ivanov13,Kaminski15,Knorr11,Lukasiewicz07i,Lukasiewicz08b,Lukasiewicz09}.
Very roughly, in these works data complexity tractability is preserved for some form of hybrid normal logic programs under some well-founded semantics, provided that the underlying DL has a tractable instance checking procedure.

%
\section{Conclusions}\label{conclusions}
%
\paragraph{Contribution.} We have shown that the subsumption problem under RC in $\ELbot$  can be decided in polynomial time, as is the case for the monotone  counterpart, via a polynomial number of classical $\ELbot$ subsumption tests and, thus, can be implemented by relying on any current reasoners for $\ELbot$. Furthermore, we have adapted the algorithm also to one refinement of RC for DLs, Defeasible Inheritance-based Closure, that overcomes the drowning effect, a main RC's inferential limit, 
showing that the subsumption problem remains polynomial. We also have presented a novel, polynomial time decision procedure to deal with nominals. 
To this end,  we propose the introduction of the notion of \emph{defeasible nominals}, which is compatible with monotone classical nominal safe \ELObot. The main result here is that  we may still use the RC procedure presented for $\ELbot$ (and its extension)  to deal with nominal safe \ELObot~without an increase in the computational complexity of the subsumption problem.

\paragraph{Future Work.} Concerning  future work, one of our aims  is to determine whether the subsumption problem under RC still remains polynomial for various other tractable DLs such as those of the  \emph{DL-Lite} family~\cite{Artale09}
and those of the so-called Horn-DL family~\cite{Kroetzsch13}, which are at the core of the OWL 2 profiles OWL QL and OWL RL, respectively. Another direction is to implement our method and test it in a similar way as done in~\cite{BritzEtAl2015a, Moodley15}, where the latter confirmed that the use of DL reasoners as a black box is scalable in practice.




\section*{Acknowledgements}
\nd The work of Giovanni Casini and of Thomas Meyer has received funding from the European
Unions Horizon 2020 research and innovation programme under
the Marie Sk\l{}odowska-Curie grant agreement No. 690974 (MIREL project).  The work of Thomas Meyer has been supported in part by the National Research Foundation of South Africa (grant No. UID 98019).



\newpage

\appendix

\section*{Appendix}
 
\section{Characteristic DL model for Rational Closure}\label{canmod}
%

\newcommand{\rumin}{R^{\cup_{\min}}_\KB}
\newcommand{\ruminc}{R^{\cup_{\mathfrak{c},\min}}_\KB}

\nd In~\cite[Section 5]{BritzEtAl2015a}  the following DL model characterising rational closure for \ALC~is introduced. Let $\KB=\tuple{\T,\D}$ be a defeasible ontology and  $\R^\KB$ be the set of all the ranked models of $\KB$. Let $\Delta$ be any countable infinite domain, and $\R^\KB_\Delta$ be the set of ranked models of $\KB$ that have $\Delta$ as domain. 
We can use the set of models $\R^\KB_\Delta$ to define a single model characterising rational closure. Let $\ru=\tuple{\Delta^{\ru},\cdot^{\ru},\prec^{\ru}}$ be a ranked model obtained in the following way. 
\begin{itemize}
\item For the domain $\Delta^{\ru}$, we consider in $\Delta^{\ru}$ one copy of $\Delta$ for each model in $\R^\KB_\Delta$. Specifically, given $\Delta=\{x,y\ldots\}$, we indicate as $\Delta_R=\{x_R,y_R,\ldots\}$ a copy of the domain $\Delta$ associated with an interpretation $R\in\R^\KB_\Delta$ and define 
\[
\Delta^{\ru}=\bigcup_{R\in\R^\KB_\Delta} \Delta_R \ .
\]
\item The interpretation function and the preferential relation are defined referring directly to the models in $\R^\KB_\Delta$. That is, for every $x_R, y_{R'}\in\Delta^{\ru}$, every atomic concept $A$ and every role $r$,
\begin{itemize}
\item $x_R\in A^{\ru}$ iff $x\in A^{R}$;
\item $(x_R, y_{R'})\in r^{\ru}$ iff $R=R'$ and $(x,y)\in r^{R}$;
\item $x_R\prec^{\ru} y_{R'}$ iff $h_R(x)< h_{R'}(y)$. 
\end{itemize}
It is easy to check by induction on the construction of the concepts that, for every $x_R\in\Delta^{\ru}$ and every  concept $C$,
\begin{itemize}
\item $x_R\in C^{\ru}$ iff $x\in C^{R}$;
\end{itemize}
and that every individual $x_R\in\Delta^{\ru}$ preserves its original height, that is,
\begin{itemize}
\item $h_{\ru}(x_R)=h_{R}(x)$.
\end{itemize}
\end{itemize}
$\ru$ is a ranked model characterising the Rational Closure of $\KB$ \cite[Theorem 4 and Corollary 1]{BritzEtAl2015a}. That is, given a KB $\KB=\tuple{\T,\D}$, for every concept inclusion $C\usually D$,

\begin{equation}\label{RC_model}
C\usually D\text{ is in the RC of }\KB\text{ iff }\ru\sat C\usually D.
\end{equation}

\begin{remark}
From now on, given a set of ranked models $S$, we will indicate the model obtained using a procedure like the one used here above as the \emph{ranked merge} of the models in $S$. That is, $\ru$ is the \emph{ranked merge} of the models in $\R^\KB_\Delta$.
\end{remark}

In \cite[Theorem 4]{BritzEtAl2015a} the correspondence of $\ratent$ with the model $\ru$ has been proved. That is, for every \ALC~KB $\KB=\tuple{\T,\D}$ and every  inclusion $C\usually D$ (resp. $C\subs D$), $\KB\ratent C\usually D$ (resp., $\KB\ratent C\subs D$) iff $\ru\sat C\usually D$ (resp., $\ru\sat C\subs D$).

\nd In what follows we are going to prove that the above characterisation of RC using the model $\ru$ is equivalent to the one introduced in~\cite{Giordano15}, which we will use then in~\ref{skept-cred}.
The following results are valid for  DL languages more expressive than $\ELbot$, but not containing nominals, as \eg~for~\ALC.

First we restrict our attention to a subset of the models in $\R^\KB_\Delta$. Specifically, those that are:

\begin{enumerate}
\item \emph{minimal} w.r.t. the relation  $<$, as defined in \cite{Giordano15} and reported here as Definition \ref{def_minimal_giordano}. 

\item \emph{canonical}, as defined in \cite{Giordano15}, which we report below for the sake of completeness.

\begin{definition}[\cite{Giordano15}, Definition 24]\label{def_canonical_giordano}
Let $\KB=\tuple{\T,\D}$ be a knowledge base in defeasible \ALC, and $S$ be the set of concepts appearing in $\KB$. A ranked model $R=\tuple{\Delta^R,\cdot^R,\prec^R}$ of $\KB$ is a \emph{canonical} model if, for every set of concepts $\{C_1,\ldots,C_n\}\subseteq S$ consistent with $\KB$ there is at least an $x\in\Delta^R$ s.t. $x\in(C_1\dlAnd\ldots\dlAnd C_n)^R$.
\end{definition}

\end{enumerate}

\nd Let $\min^{\mathfrak{c}}_{<}(\R^\KB_\Delta)$ (resp.  $\min^{\mathfrak{c}}_{<}(\R^\KB)$) be the set of the minimal models in $\R^\KB_\Delta$ (resp. $\R^\KB$) \wrt~$<$ that are also canonical. The following proposition shows that, for every concept $C$, every canonical minimal model in $\R^\KB$ associates to $C$ the same height. 

\begin{proposition}\label{rumin_prop}
For every $R,R'$ in $\min^{\mathfrak{c}}_{<}(\R^\KB)$ and every concept $C$,
\[
h_{R'}(C)=h_R(C)\  .
\]
\qed
\end{proposition}

\begin{proof}
Given two ranked interpretations $R=\tuple{\Delta^{R},\cdot^{R},\prec^{R}}$ and $R'=\tuple{\Delta^{R'},\cdot^{R'},\prec^{R'}}$ and an ontology $\KB=\tuple{\T,\D}$, let $\Gamma_{\KB}$ be the set of the subconcepts appearing in $\KB$ and, for $o\in\Delta^R$, let $[o_{\KB}^R]_{R'}=\{o'\in\Delta^{R'}\mid \text{ for every concept }C\in\Gamma,\ o\in C^{R}\text{ iff } o'\in C^{R'}\}$. Also, let $h_{R'}([o_{\KB}^R]_{R'})=\min\{h_{R'}(o')\mid o'\in [o_{\KB}^R]_{R'}\}$.

Assume to the contrary that there are two $R,R'\in\min^{\mathfrak{c}}_{<}(\R^\KB)$ and a concept $C$ s.t. $h_{R'}(C)<h_R(C)$. 
Since $R,R'$ are canonical, this implies that for some $o\in C^R$ there is an object $o'\in [o_{\KB}^R]_{R'}$ s.t. $h_{R'}(o')<h_R(o)$.

Let $R^*=\tuple{\Delta^{R^*},\cdot^{R^*},\prec^{R^*}}$ be s.t. $\Delta^{R^*}=\Delta^{R}$, $\cdot^{R^*}=\cdot^{R}$, and, for every $o\in\Delta^{R^*}$ $h_{R^*}(o)=\min\{h_R(o),h_{R'}( [o_{\KB}^R]_{R'})\}$. It is easy to check that $R^*$ is a model of $\KB$: first check by induction on the construction of the concepts that for every $o\in\Delta^{R^*}$ and for every concept $C$, $o\in C^{R^*}$ iff $o\in C^R$; then it follows that for every strict inclusion $C\subs D$, $R^*\sat C\subs D$ iff $R\sat C\subs D$; eventually, using the definition of $h_{R^*}(o)$, we can prove that for every $C\usually D \in\D$, $h_{R^*}(C\dlAnd D)<h_{R^*}(C\dlAnd \neg D)$  or $h_{R^*}(C)=\infty$. So $R^*$ is a model of $\KB$ and $R^*< R$, against the hypothesis that $R\in\min^{\mathfrak{c}}_{<}(\R^\KB)$.
%
%
%
%
%
%
%
%
%
%
%
%
%
\end{proof}

\nd We can prove easily that also the following lemma holds. Let $\ruminc$ be the ranked merge of $\min^{\mathfrak{c}}_{<}(\R^\KB_\Delta)$.

\begin{lemma}\label{lemma_height_cup}
For every KB $\KB=\tuple{\T,\D}$ and every concept $C$, 
\[
h_{\ru}(C)=h_{\ruminc}(C) \ .
\]
\qed
\end{lemma}

\begin{proof}
For every concept $C$, $h_{\ru}(C)=h_{R}(C)$ for some $R\in \min_{<}(\R^\KB_\Delta)$, that is, some $R\in \R^\KB_\Delta$ that is minimal w.r.t. $<$. We have to prove that if there is a model $R\in \min_{<}(\R^\KB_\Delta)$ s.t. $h_{\ru}(C)=h_{R}(C)$, then there must be a model $R'\in \min^{\mathfrak{c}}_{<}(\R^\KB_\Delta)$ s.t. $h_{\ru}(C)=h_{R'}(C)$.

Let $R^*$ be a model in $\min^{\mathfrak{c}}_{<}(\R^\KB_\Delta)$, and $R^{**}$ be the ranked merge of $R$ and $R^*$; it is easy to prove that $R^{**}$ is still a model of $\KB$, and in particular it must be a minimal canonical model of $\KB$: it must be minimal, otherwise $R$ and $R^*$ would not be minimal, and it must be canonical, since $R^*$ is canonical. Hence, for every concept $C$, $h_{R^{**}}(C)=\min\{h_{R}(C), h_{R^{*}}(C)\}\leq h_{\ru}(C)$, and, since $R^*$ is merged into $\ru$, $h_{R^{**}}(C)=h_{\ru}(C)$. Also, the domain of $R^{**}$, $\Delta^{R^{**}}$, must be countably infinite, since it is obtained by merging two countably infinite domains.

Since $\Delta^{R^{**}}$ is countably infinite, we can define a bijection $b$ between $\Delta^{R^{**}}$ and $\Delta$, and using $b$  we can define a model $R^{**}_\Delta$ corresponding to $R^{**}$ but having $\Delta$ as domain  (that is, for every $C\usually D$, $R^{**}\sat C\usually D$ iff $R^{**}_\Delta \sat C\usually D$); it is sufficient to define $\cdot^{R^{**}_\Delta}$ as:

\begin{itemize}
\item for  every $A\in\CN$, $A^{R^{**}_\Delta}=\{o\in\Delta\mid b^-(o)\in A^{R^{**}}\}$;
\item for  every $r\in\RN$, $r^{R^{**}_\Delta}=\{(o,o')\in\Delta\times\Delta\mid (b^-(o),b^-(o'))\in r^{R^{**}}\}$;
\end{itemize}

\nd and for every $o\in\Delta$, set $h_{R^{**}_\Delta}(o)=h_{R^{**}}(b^{-}(o))$. $R^{**}_\Delta$ must be in $\min^{\mathfrak{c}}_{<}(\R^\KB_\Delta)$ and $h_{R^{**}_\Delta}(C)=h_{\ru}(C)$. 

So, for every concept $C$ there is a model $R'\in \min^{\mathfrak{c}}_{<}(\R^\KB_\Delta)$ \st~$h_{R'}(C)=h_{\ru}(C)$. By Proposition \ref{rumin_prop}, that implies that $h_R(C)=h_{\ru}(C)$ for every $R\in \min^{\mathfrak{c}}_{<}(\R^\KB_\Delta)$, that in turn implies that $h_{\ruminc}(C)=h_{\ru}(C)$ for every concept $C$.
\end{proof}

\nd From Lemma \ref{lemma_height_cup} and Definition \ref{rankedsat} it follows that RC can be characterised merging the minimal canonical models of $\KB$ having $\Delta$ as domain.

\begin{proposition}\label{equiv_1}
For every inclusion  $C\usually D$,
\[\ru\sat C\usually D\text{ iff }\ruminc\sat C\usually D\]
\qed
\end{proposition}

\nd Expression \ref{giordano_eq} here below corresponds to  the semantic characterisation of RC for \ALC~presented in~\cite{Giordano13,Giordano15}. 
 For every KB $\KB=\tuple{\T,\D}$ and every inclusion $C\usually D$,

\begin{equation}\label{giordano_eq}
C\usually D\text{ is in the RC of }\KB\text{ iff }R\sat C\usually D, \text{ for every  } R\in\mbox{$\min^{\mathfrak{c}}_{<}$}(\R^\KB) \ .
\end{equation}

\nd In particular, Expression~\ref{giordano_eq} is a rephrasing in our framework of Theorem 8 in \cite{Giordano15}. We want to prove that, indeed, our characterisation of RC (Expression \ref{RC_model}) is equivalent to the one presented in  \cite{Giordano15} (Expression~\ref{giordano_eq}). First, we  prove that, starting from the 
Expression~\ref{giordano_eq}, we can restrict our attention only to the minimal canonical models in $\R^\KB_\Delta$ (Lemma \ref{giordano_delta_lemma}); then  (Proposition \ref{giordano_delta_prop}) we are going to prove that such a characterisation corresponds to the one using $\ruminc$.

\begin{lemma}\label{giordano_delta_lemma}
For every canonical model $R$ of a KB $\KB=\tuple{\T,\D}$, there is a canonical model $R_\Delta$ of $\KB$ that has a countably infinite domain $\Delta$ and s.t. for every inclusion axiom $C\usually D$, $R\sat C\usually D$ iff $R_\Delta\sat C\usually D$.
\qed
\end{lemma}

\begin{proof}

Let $R=\tuple{\Delta^R,\cdot^R,\prec^R}$ be a canonical model of $\KB$. The possible cases are three.

\begin{enumerate}
\item $\Delta^R$ is countably infinite. Then we can build an equivalent model having $\Delta$ as domain and preserving the height of each concept (see the proof of Lemma \ref{lemma_height_cup}).
\item $\Delta^R$ is finite. Then we can build an equivalent model with a countably infinite domain: just make a countably infinite number of copies of $R$ and rank-merge them in a single model (that preserves the height of each concept). Then the result follows from case 1.
\item $\Delta^R$ is uncountably infinite. Also, in this case we can build an equivalent model with a countably infinite domain.
To build it we can use the Finite Model Property (FMP, \cite[Theorem 7]{BritzEtAl2015a}) and the Finite Counter-Model Property (FCMP, \cite[Proposition 15]{BritzEtAl2015a}) that hold for defeasible \ALC. The latter indicates that, if there is a model of an ontology $\KB$ that does not satisfy an inclusion $C\usually D$, then there is a finite model of $\KB$ that does not satisfy $C\usually D$.

FMP and FCMP are based on the following construction \cite[pp. 78-79]{BritzEtAl2015a}: let $\Gamma$ be a set of concepts $\{C_1,\ldots,C_n\}$ \st~$\Gamma$ is obtained closing under sub-concepts and negation the concepts appearing in \KB~(plus the concepts $C$ and $D$, if we are looking for a model of $\KB$ that also falsifies $C\usually D$).
 Now we define the equivalence relation $\sim_\Gamma$ as
 \[
 \forall x,y\in\Delta^R, x\sim_\Gamma y\text{ iff }\forall C\in\Gamma,x\in C^R\text{ iff }y\in C^R \ . 
 \]
\nd  We indicate with $[x]_\Gamma$ the equivalence class of the individuals that are related to an individual $x$ through $\sim_\Gamma$:
 \[
 [x]_\Gamma=\{y\in\Delta^R\mid x\sim_\Gamma y\} \ .
 \]
\nd  We introduce a new model $R'=\tuple{\Delta^{R'},\cdot^{R'}, \prec^{R'}}$, defined as:

 \begin{itemize}
 \item $\Delta^{R'}=\{[x]_\Gamma\mid x\in\Delta^R\}$;
 \item for every $A\in\CN\cap\Gamma$, $A^{R'}=\{[x]_\Gamma\mid x\in A^R\}$;
 \item for every $A\notin\CN\cap\Gamma$, $A^{R'}=\emptyset$; 
 \item for every $r\in \RN$, $r^{R'}=\{\tuple{[x]_\Gamma,[y]_\Gamma}\mid \tuple{x,y}\in r^R\}$;
 \item 
 For every $[x]_\Gamma,[y]_\Gamma\in \Delta^{R'}$, $[x]_\Gamma\prec^{R'} [y]_\Gamma$ if there is an object $z\in [x]_\Gamma$ \st~for all $v\in [y]_\Gamma$, $z\prec^R v$; 
 \item for every $a\in\Names$, $a^{R'}=[x]_\Gamma$ iff $a^{R}=x$.
 \end{itemize}

It is easy to see that this construction gives back a ranked interpretation that preserves the relative height of the concepts, that is, for every $C, D\in\Gamma$, $h_{R'}(C)\leq h_{R'}(D)$ iff $h_{R}(C)\leq h_{R}(D)$. Also, since the set $S$ in Definition \ref{def_canonical_giordano} is a subset of $\Gamma$, also canonicity is preserved.

Now, given the model $R$, let $\usually^-_R=\{C\usually D\mid R\not\sat C\usually D\}$. Since we have assumed a finitely generated vocabulary (see Section \ref{EL}), we will have a countably infinite number of inclusion statements that can be generated (this can be proven using some G\"odel-style numerical encoding of the inclusions); hence $\usually^-_R$ must be countably infinite too (it cannot be a finite set, since it is easy to check that for every inclusion there is in our language an infinite number of inclusions that are logically equivalent to it).

For each inclusion $C\usually D$ in $\usually^-_R$, take a finite model $R_{C\usually D}$ of $\KB$ that falsifies $C\usually D$. Let $R^*$ be the interpretation obtained by merging all the models $R_{C\usually D}$, with $C\usually D\in \usually^-_R$. It is easy to check that $R^*$ is a model of $\KB$ that satisfies the desired constraints and has a countably infinite domain (it is obtained by merging a countably infinite number of finite models). Then, the result follows from case 1.
\end{enumerate}
%
%
%
%
\end{proof}

\nd We can now prove that in Expression \ref{giordano_eq} we can restrict our attention from all the minimal canonical models of an ontology $\KB $ (that is, the minimal canonical models in $\R^\KB$) to the minimal canonical models of $\KB $ having $\Delta$ as domain (that is, the minimal canonical models in $\R^\KB_\Delta$). Indeed,  the following proposition follows immediately from Lemma \ref{giordano_delta_lemma}.

\begin{proposition}\label{giordano_delta_prop}
For every KB $\KB=\tuple{\T,\D}$ and every inclusion $C\usually D$, $R\sat C\usually D$ for every model $R\in\min^{\mathfrak{c}}_{<}(\R^\KB)$ iff $R'\sat C\usually D$ for every model $R'\in\min^{\mathfrak{c}}_{<}(\R^\KB_\Delta)$.
\end{proposition}

%
%
%

\nd Now, the correspondence between our semantics for RC (Expression~\ref{RC_model}) and the one presented in \cite{Giordano15} (Expression~\ref{giordano_eq}) follows.

\begin{proposition}\label{prop_correspondence}
model $R\in\min^{\mathfrak{c}}_{<}(\R^\KB)$ if and only if $\ru\sat C\usually D$.
Expression~\ref{RC_model} and Expression~\ref{giordano_eq} are equivalent.
\qed
\end{proposition}

\begin{proof}
The following holds:
\begin{eqnarray*}
C\usually D\text{ is in the RC of }\KB & \mbox{iff} & \ru\sat C\usually D~(\mbox{Expression~\ref{RC_model}}) \\
& \mbox{iff} &  \ruminc\sat C\usually D~(\mbox{Proposition~\ref{equiv_1}}) \\
& \mbox{iff} & R\sat C\usually D \mbox{ for every } R\in\mbox{$\min^{\mathfrak{c}}_{<}$}(\R^\KB_\Delta)~(\mbox{Proposition~\ref{rumin_prop}}) \\ 
& \mbox{iff} & R\sat C\usually D \mbox{ for every } R\in\mbox{$\min^{\mathfrak{c}}_{<}$}(\R^\KB)~(\mbox{Proposition~\ref{giordano_delta_prop}}) \\ 
& \mbox{iff} & C\usually D\text{ is in the RC of }\KB~\mbox{according to Expression~\ref{giordano_eq}}) \ ,
\end{eqnarray*}
\nd which concludes.
%
%
%
\end{proof}

\section{Prototypical Reasoning over the individuals}\label{skept-cred}
%

\nd As shown in Example \ref{multipleextensions}, once we introduce presumptive reasoning involving individuals, we are faced with the possibility of multiple minimal configurations. Such an observation allows us to distinguish two kinds of reasoning concerning the individuals (see Section \ref{rcelo}): one is a presumptive reasoning, modelled by a skeptical approach that takes under consideration what holds in \emph{all} the most typical situations satisfying our KB; the second possible approach is to take under consideration only the most typical instantiations of the individual we are interested in. That is, we attribute presumptively a property $C$ to an individual $a$ if $a$ falls under $C$ in all the  models in which $a$ is interpreted in the most prototypical way. 
Given a nominal safe $\ELObot$ ontology $\KB=\tuple{\T,\D}$ and an individual $a\in \ON$, the latter kind of reasoning would correspond to taking under consideration only those minimal models of $\KB$ in which also the interpretation of $a$ is minimal. 

In~\ref{canmod} we have modelled our reasoning using $\ru$, $\ruminc$, or, equivalently,  $\min^{\mathfrak{c}}_{<}(\R^{\KB}_\Delta)$, that are equivalent options  to characterise the RC of a KB, if we do not allow nominals. 


We can extend the definition of $<$, and hence  the definition of $\min^{\mathfrak{c}}_{<}(\R^{\KB}_\Delta)$, also to languages containing nominals. As Example \ref{multipleextensions} shows, Proposition \ref{rumin_prop} does not hold for the minimal models of ontologies that have also nominals: in such an example we would have that in some minimal models $\{a\}$ is interpreted on layer 0 and $\{b\}$ on layer 1, and in other minimal models the other way around.
Therefore, in presence of nominals, we have two options regarding the kind of minimality we want to use. On one hand we can continue to use a notion of entailment relation 
$\aboxent$ built using the entire set of models in $\min^{\mathfrak{c}}_{<}(\R^{\KB}_\Delta)$ as done in~\cite{BritzEtAl2015a,CasiniStraccia13}. 
That is:

\begin{definition} \label{def7}
$\KB\aboxent C\usually D$  iff $R \models C\usually D$ for every $R\in\mbox{$\min^{\mathfrak{c}}_{<}$}(\R^{\KB}_\Delta)$.
\qed
\end{definition}

\nd Another option, and which is what we are going to analyse here, is more appropriate for modelling prototypical reasoning. If we are investigating about which properties are to be associated to the prototype of a certain individual $a$, we consider only the minimal models in which $a$ is interpreted in the lowest possible layer. 
%
%
That is, if we are wondering whether some inclusion $\{a\}\usually C$ holds, we do not refer to all the models in $\min^{\mathfrak{c}}_{<}(\R^{\KB}_\Delta)$, but  only to those models in which the individual $a$ is interpreted at its minimal height. Specifically, let us define $R\in\min^{\mathfrak{c},a}_{<}(\R^{\KB}_\Delta)$ iff $R\in\min^{\mathfrak{c}}_{<}(\R^{\KB}_\Delta)$ and $h_R(a)\leq h_{R'}(a)$ for every $R'\in\min^{\mathfrak{c}}_{<}(\R^{\KB}_\Delta)$. Based on this notion, we can now define a different consequence relation $\aboxentsec$ to reason about nominals.




\begin{definition}\label{def8}
$\KB\aboxentsec\{a\}\usually C$ iff $a^R\in C^R$ for every $R\in\min^{\mathfrak{c},a}_{<}(\R^{\KB}_\Delta)$. 
\qed
\end{definition}
\nd To illustrate the difference between $\aboxentsec$ and $\aboxent$, consider the knowledge base in Example \ref{multipleextensions}. According to $\aboxentsec$ we consider only the models in which $a$ is at rank $0$ if we are asking something about $a$ (for example, whether the prototype of $a$ satisfies $C$), while we consider only the models in which $b$ is at rank $0$, if we are asking something about $b$ (for example, if the prototype of $b$ satisfies $C$). According to $\aboxent$, we neither can conclude that presumably $a$ satisfies $C$ nor that presumably $b$ satisfies $C$ (each of them is falsified in some typical situations), while according to $\aboxentsec$ we can conclude that  the prototypes of both $a$ and $b$ satisfy $C$ (in the preferred situations w.r.t. $a$, $a$ falls under $C$, and analogously for $b$).

We can  generalise the definition of $\aboxentsec$ to the entire language of nominal-safe $\ELObot$. Specifically, given a nominal safe ontology $\KB=\tuple{\T,\D}$, a ranked model $R$ of $\KB$, and an n-safe concept $C$, let us define $R\in\min^{\mathfrak{c},C}_{<}(\R^{\KB}_\Delta)$ iff $R\in\min^{\mathfrak{c}}_{<}(\R^{\KB}_\Delta)$ and for every model $R'\in \min^{\mathfrak{c}}_{<}(\R^{\KB}_\Delta)$, $h_R(C)\leq h_{R'}(C)$. Then, Definition \ref{def8} can be generalised as follows.

\begin{definition}\label{def_nominalsafe_ent}
Given a  nominal safe $\ELObot$ ontology $\KB=\tuple{\T,\D}$ and a  safe inclusion $C\usually D$, $\KB\aboxentsec C\usually D$ iff $\min_{\prec^R}(C^R)\subseteq D^R$ for every $R\in\min^{\mathfrak{c},C}_{<}(\R^{\KB}_\Delta)$. 
\qed
\end{definition}


\nd Note that if we are not using nominals, that is, we are using $\ELbot$, there is no difference between $\aboxent$ and $\aboxentsec$, since they both correspond to  RC entailment.

\begin{proposition}\label{no_nominals_rc}
Given a defeasible  $\ELbot$ ontology $\KB=\tuple{\T,\D}$ and an $\ELbot$ inclusion $C\usually D$, then $\KB\aboxentsec C\usually D$ iff $\KB\aboxent C\usually D$ iff $\KB\ratent C\usually D$. 
\qed
\end{proposition}

\begin{proof}
The correspondence between $\aboxentsec$ and $\aboxent$ in $\ELbot$ is an immediate consequence of Proposition \ref{rumin_prop}. The correspondence between $\ratent$ and Expression \ref{RC_model} has been proved in \cite[Theorem 4]{BritzEtAl2015a}, Proposition \ref{prop_correspondence} proves the correspondence between Expression \ref{RC_model} and Expression \ref{giordano_eq}, and Expression \ref{giordano_eq} corresponds to $\aboxent$ by Definition \ref{def7}.

\end{proof}

\nd We can now prove that the decision procedure defined in Section \ref{rcelo} using defeasible nominals corresponds to prototypical reasoning over the individuals, that is, it corresponds exactly to $\aboxentsec$.

\begin{proposition} \label{defind}
Let $\KB=\tuple{\T,\D}$ be a nominal safe $\ELObot$ ontology. For every safe $C\usually D$, 
$\KB\aboxentsec C\usually D$  iff $\KB_{\tiny \defnom{}}\ratent C_{\tiny \defnom{}}\usually D_{\tiny \defnom{}}$.

\qed
\end{proposition}

\begin{proof}

Since $\KB_{\tiny \defnom{}}$ and $ C_{\tiny \defnom{}}\usually D_{\tiny \defnom{}}$ are \ELbot~expressions, by Proposition \ref{no_nominals_rc}, $\KB_{\tiny \defnom{}}\ratent C_{\tiny \defnom{}}\usually D_{\tiny \defnom{}}$ iff $\KB_{\tiny \defnom{}}\aboxentsec C_{\tiny \defnom{}}\usually D_{\tiny \defnom{}}$. 
So, it is sufficient to prove that for every safe $C\usually D$, 
$\KB\aboxentsec C\usually D$  iff $\KB_{\tiny \defnom{}}\aboxentsec C_{\tiny \defnom{}}\usually D_{\tiny \defnom{}}$.

$(\Leftarrow)$. Assume $\KB_{\tiny \defnom{}}\aboxentsec C_{\tiny \defnom{}}\usually D_{\tiny \defnom{}}$. Consider each model $R$ in $\min^{\mathfrak{c},C}_{<}(\R^{\KB})$ and extend its interpretation function $\cdot^R$ by imposing that, for every $a\in\ON$,  $\defnom{a}^R=\{a\}^R$: $R$ becomes also a  model of $\KB_{\tiny \defnom{}}$ and it must hold that $R\in \min^{\mathfrak{c},C_{\tiny \defnom{}}}_{<}(\R^{\KB_{\tiny \defnom{}}})$. Assume  the latter is not the case: then there is an $R^*\in \min^{\mathfrak{c}, C_{\tiny \defnom{}}}_{<}(\R^{\KB_{\tiny \defnom{}}})$ s.t. $R^* < R$; extend $\cdot^{R^*}$ \st, for every $a\in\ON$,  $\{a\}^{R^*} = \defnom{a}^{R^*}$ (Since $R^*<R$, it must be $\defnom{a}^{R^*}=\defnom{a}^{R}$, hence also in $R^*$ every defeasible nominal is interpreted into a single object), and then we have 
$R^*\in \R^{\KB}$ with $R^* < R$, contrary to the assumption $R$ in $\min^{\mathfrak{c},C}_{<}(\R^{\KB})$.


Therefore, since every  $R$ in $\min^{\mathfrak{c},C}_{<}(\R^{\KB})$ is also in $ \min^{\mathfrak{c},C_{\tiny \defnom{}}}_{<}(\R^{\KB_{\tiny \defnom{}}})$, $\KB_{\tiny \defnom{}}\aboxentsec C_{\tiny \defnom{}}\usually D_{\tiny \defnom{}}$ implies that $R \models C_{\tiny \defnom{}}\usually D_{\tiny \defnom{}}$  and, since in every such $R$, for every $a\in\ON$,  $\defnom{a}^R=\{a\}^R$, $R \models C\usually D$ follows and we can conclude $\KB\aboxentsec C\usually D$.


$(\Rightarrow)$. Assume $\KB_{\tiny \defnom{}}\not\aboxentsec C_{\tiny \defnom{}}\usually D_{\tiny \defnom{}}$. We need to prove that $\KB\not\aboxentsec C\usually D$.

$\KB_{\tiny \defnom{}}\not\aboxentsec C_{\tiny \defnom{}}\usually D_{\tiny \defnom{}}$ implies that  there  is a  model $R=\tuple{\Delta^R,\cdot^R,\prec^R}$ in $\min^{\mathfrak{c},C_{\tiny \defnom{}}}_{<}(\R^{\KB_{\tiny \defnom{}}})$ s.t. $R \not \models C_{\tiny \defnom{}}\usually D_{\tiny \defnom{}}$, and we need to prove that there  is a  model $R^*$ in $\min^{\mathfrak{c},C}_{<}(\R^{\KB})$ s.t. $R^* \not \models C\usually D$.

From $R$ we define a ranked interpretation $R'=\tuple{\Delta^{R'}, \cdot^{R'},\prec^{R'}}$ using the procedure explained in the proof of Lemma~\ref{exceptional_defnom}.  Specifically,\footnote{ Recall that we introduce $x_{\typ{E}}$  as an object that represents a typical instance of the concept $E$, while  $x_{{E}}$ represents an atypical instance of the concept $E$.}

\begin{itemize}
\item $\Delta^{R'}=\{x_{\typ{E}},x_E\mid R\not\models E\subs\bot, \mbox{ where $E$ is a safe $\mathcal{EL\defnom{O}}_\bot$ concept}\}\cup\{x_{\defnom{a}}\mid a\in \ON\}$;
\item $\defnom{a}^{R'}=\{x_{\defnom{a}}\in \Delta^{R'}\}$, for every defeasible nominal $\defnom{a}$; 
\item $A^{R'}=\{x_E\in \Delta^{R'}\mid R\models E\subs A\}\cup\{x_{\typ{E}}\in \Delta^{R'}\mid R\models E\usually A\}\cup\{x_{\defnom{a}}\in \Delta^{R'}\mid R\models \defnom{a}\usually A\}$, for every  safe $\mathcal{EL\defnom{O}}_\bot$ concept $E$, every defeasible nominal $\defnom{a}$, and every atomic $\mathcal{EL_\bot}$ concept $A$; 
\item $r^{R'}=\{\tuple{x_E,x_F}\in \Delta^{R'}\times\Delta^{R'}\mid R\models E\subs \exists r. F\}\cup\{\tuple{x_{\typ{E}},x_F}\in \Delta^{R'}\times\Delta^{R'}\mid R\models E\usually \exists r. F\}\cup \{\tuple{x_{\defnom{a}},x_F}\in \Delta^{R'}\times\Delta^{R'}\mid R\models \defnom{a}\usually \exists r. F\}$, for every  safe $\mathcal{EL\defnom{O}}_\bot$ concepts $E$ and $F$, and every defeasible nominal $\defnom{a}$;
\item $h_{R'}(x_{\typ{E}})=h_R(E)$, for every $x_{\typ{E}}\in\Delta^{R'}$;
\item $h_{R'}(x_{\defnom{a}})=h_R(\defnom{a})$, for every defeasible nominal $\defnom{a}$; 
\item $x_{\typ{F}}\prec^{R'}x_{E}$ for every $x_{\typ{F}},x_E\in\Delta^{R'}$.
\end{itemize}
\nd We have proved in  the proof of Lemma~\ref{exceptional_defnom} that $R'$ is still a model of  $\KB_{\tiny \defnom{}}$. Specifically, for every pair of safe $\mathcal{EL\defnom{O}}_\bot$ concepts $E,F$, we have 
\begin{enumerate}
\item $R\models E\subs F$  and $R \not \models E \subs \bot$ iff $x_E\in F^{R'}$;
\item $R\models E\usually F$ and $R \not \models E \subs \bot$ iff $x_{\typ{E}}\in F^{R'}$.
\end{enumerate}
\nd And, for every pair  $\defnom{a},E$ where $a\in\ON$  and $E$ is \emph{safe}, we have, 
\begin{enumerate}
\item[3.] $R\models \defnom{a}\usually E$ iff $x_{\defnom{a}}\in E^{R'}$.
\end{enumerate}

This implies that for every safe concept $E$ such that $R \not \models E \subs \bot$, $x_{\typ E},x_{E}\in E^{R'}$, and for every $\defnom{a}$, $ \defnom{a}^{R'}=\{x_{\defnom{a}}\}$.
Moreover, by construction of $R'$,  for every safe concept $E_{\tiny \defnom{}}$, $x_E\in\min_\prec^{R'}(E_{\tiny \defnom{}}^{R'})$ (otherwise it is easy to obtain a contradiction using 2. and 3.) and, consequently, $h_{R'}(E_{\tiny \defnom{}}) = h_R(E_{\tiny \defnom{}})$. We can conclude that  
$R'$  is a model of $\KB_{\tiny \defnom{}}$ s.t. $R'\not\sat C_{\tiny \defnom{}}\usually D_{\tiny \defnom{}}$. Since in $R'$ every defeasible nominal $\defnom{a}$ is interpreted into a single object, extending $\cdot^{R'}$ by imposing that, for every $a\in\ON$,  $\{a\}^{R'}=\defnom{a}^{R'}=\{x_{\defnom{a}}\}$, we can conclude that $R'$  is a model of $\KB$ s.t. $R'\not\sat C\usually D$, that is, there is an object $o$ s.t. $o\in\min_{\prec^{R'}}(C^{R'})$ and $o\notin D^{R'}$. 

If $R'$  is in $\min^{\mathfrak{c},C}_{<}(\R^{\KB})$, we conclude our proof. It could be that $R'\notin\min^{\mathfrak{c},C}_{<}(\R^{\KB})$, and, in case,  we need to prove that the existence of $R'$ implies the existence of another model  $R^*$ s.t. $R^* \not \models C\usually D$ with $R^*$ in $\min^{\mathfrak{c},C}_{<}(\R^{\KB})$.

So, let $R^*$ be a model in $\min^{\mathfrak{c},C}_{<}(\R^{\KB})$ and such that $R^*< R'$; we have to prove that $R^* \not \models C\usually D$.

As above, extend $\cdot^{R^*}$ for defeasible nominals by imposing that, for every $a\in\ON$,  $\{a\}^{R'}=\defnom{a}^{R'}=\{x_{\defnom{a}}\}=\{x_{\defnom{a}}\}$. $R^*$ is a model of $\KB_{\tiny \defnom{}}$ and  $R^*\in \min^{\mathfrak{c},C}_{<}(\R^{\KB})$ implies $R^*\in \min^{\mathfrak{c},C_{\tiny \defnom{}}}_{<}(\R^{\KB_{\tiny \defnom{}}})$ (as shown in the $(\Leftarrow)$-part of this proof).

By Proposition~\ref{rumin_prop}, since $R$ and $R^*$ are both in $\min^{\mathfrak{c},C_{\tiny \defnom{}}}_{<}(\R^{\KB_{\tiny \defnom{}}})$, for every safe concept $E_{\tiny \defnom{}}$ it holds  $h_{R^*}(E_{\tiny \defnom{}}) = h_R(E_{\tiny \defnom{}})$. We have already seen that for every safe concept $E_{\tiny \defnom{}}$, $h_{R'}(E_{\tiny \defnom{}}) = h_R(E_{\tiny \defnom{}})$.

Hence we can conclude that, for every safe concept $E_{\tiny \defnom{}}$,  $h_{R^*}(E_{\tiny \defnom{}}) = h_{R'}(E_{\tiny \defnom{}})$, that is, for every safe concept $E$,  $h_{R^*}(E) = h_{R'}(E)$. Remember that object $o$, introduced above, is an object s.t. $o\in\min_{\prec^{R'}}(C^{R'})$ (that is, $h_{R'}(o)=h_{R'}(C)$) and $o\notin D^{R'}$. We have the following:

\begin{itemize}
\item $h_{R^*}(C) = h_{R'}(C)$;
\item since $R^*<R'$, $o\in C^{R^*}$, $o\notin D^{R^*}$, and $h_{R^*}(o) \leq h_{R'}(o)$.
\end{itemize}

So, it must be $h_{R^*}(o) = h_{R'}(o)=h_{R'}(C)=h_{R^*}(C)$, that is, also in $R^*$ it is the case that $o\in\min_{\prec^{R^*}}(C^{R^*})$ and $o\notin D^{R^*}$.

Therefore, we can conclude that $R^*\in\min^{\mathfrak{c},C}_{<}(\R^{\KB})$ with $R^* \not \models C\usually D$, and, thus, $\KB\not\aboxentsec C\usually D$.

\end{proof}

\section{Various Proofs}\label{proofs} 

\nd Proposition \ref{except} and \ref{except2} can be derived from the result~\cite[Proposition 7]{BritzEtAl2015a}. However, since there are no proper statements and proofs of them, we add them here for convenience.

\vspace{0.2cm}

 \nd\textbf{Proposition \ref{except}}
\emph{For every concept $C$ and every ontology $\KB=\tuple{\T,\D}$,~if
\begin{equation}
\T \models \bigsqcap \{ \neg E \dlOr F \mid E \usually F \in \D\} \subs \neg C \ 
\end{equation}
then $C$ is exceptional \wrt~$\KB$. 
}\qed
%
\begin{proof}
Let $\KB=\tuple{\T,\D}$ be an ontology, and assume that $\T \models \bigsqcap \{ \neg E \dlOr F \mid E \usually F \in \D\} \subs \neg C$, but  $C$ is not exceptional \wrt~$\KB$. The latter means that there is a ranked interpretation $R$ such that it is a model of $\KB$ and there is an object $o$ in the lower layer $L^R_0$ of $R$ \st~$o\in C^R$. However, since $R$ is a model of $\KB$ and $o$ is in the lower layer, for every axiom $E\usually F\in\D$ it must be that either $o\notin E^R$ or $o\in F^R$. That is , $o\in (\bigsqcap \{ \neg E \dlOr F \mid E \usually F \in \D\}\dlAnd C)$, against $\T \models \bigsqcap \{ \neg E \dlOr F \mid E \usually F \in \D\} \subs \neg C$.
\end{proof}
\nd\textbf{Proposition \ref{except2}}
\emph{Given an ontology $\KB^\bullet=\tuple{\T^\bullet,\D^\bullet}$, 
obtained from the application of {\bf RC.Step 1}  to an ontology $\KB=\tuple{\T,\D}$, for every concept $C$, then
\[\T^\bullet \models \bigsqcap \{ \neg E \dlOr F \mid E \usually F \in \D^\bullet\} \subs \neg C \ \]
if and only if  $C$ is exceptional \wrt~$\KB^\bullet$.
}\qed
\begin{proof}
Let $\KB^\bullet=\tuple{\T^\bullet,\D^\bullet}$ be  specified as in the statement. One half of the statement is valid due to Proposition \ref{except}. We have to prove the other half.

Assume $C$ is exceptional w.r.t. $\KB^\bullet$, but $\T^\bullet \not\models \bigsqcap \{ \neg E \dlOr F \mid E \usually F \in \D^\bullet\} \subs \neg C$. Then there is a classic DL interpretation $M=\tuple{\Delta^M,\cdot^M}$ that is a model of $\T^\bullet$ and such that there is an object $o\in ( \bigsqcap \{ \neg E \dlOr F \mid E \usually F \in \D^\bullet\} \dlAnd  C)^M$.

Now consider the canonical model of the RC of $\KB^\bullet$, $\rub=\tuple{\Delta^{\rub},\cdot^{\rub},\prec^{\rub}}$, built as described in~\ref{canmod} (see also~\cite{BritzEtAl2015a}). 
Define a ranked interpretation $R'=\tuple{\Delta^{R'},\cdot^{R'},\prec^{R'}}$ in the following way:

\begin{itemize}
\item $\Delta^{R'}=\Delta^{\rub}\cup \Delta^M$;
\item for every atomic concept $A$, $A^{R'}=A^{\rub}\cup A^M$;
\item for every role $r$, $r^{R'}=r^{\rub}\cup r^M$;
\item for every $x\in\Delta^{\rub}$, $h_{R'}(x)=h_{\rub}(x)$;
\item for every $x\in\Delta^M$, $h_{R'}(x)=i$, if $i$ is the lowest value s.t. $x\in (H_i)^M$,
\end{itemize}
\nd where $\D^\bullet_0,\ldots,\D^\bullet_n$ is the partition of $\D^\bullet$ obtained by  {\bf RC.Step 1}  and $H_i  =  \bigsqcap \{\neg E \dlOr F \mid E \usually F \in  \D^\bullet_i \cup \ldots \cup \D^\bullet_n \}$, as defined in {\bf RC.Step 2}.

By induction on the construction of the concepts, it is easy to see that for every concept $C$ and for every $x\in \Delta^{R'}$, $x\in C^{R'}$ iff either $x\in \Delta^{M}$ and $x\in C^{M}$ or $x\in \Delta^{\rub}$ and $x\in C^{\rub}$. Hence $R'$ is a model of $\T^\bullet$.
Now we show  that it is also a model of $\D^\bullet$.
Assume that there is an axiom $E\usually F\in \D^\bullet$ s.t. $R'\not\models E\usually F$; in particular, let $E\usually F\in \D^\bullet_i$ for some $i$. That is, there is an object $p\in\Delta^{R'}$ s.t. $p\in\min_{\prec_{R'}}(E^{R'})$ and $p\in (E\dlAnd \neg F)^{R'}$. As $\rub$ is a model of $\KB^\bullet$, it must be that $p\in\Delta^M$. Due to our construction of $R'$, it must be $h_{R'}(p)>i$.
As $\rub$ is a minimal canonical model of the RC of $\KB^\bullet$ and $E\usually F\in \D^\bullet_i$, it must be the case that $h_{\rub}(E)=i$ and $h_{\rub}(E\dlAnd F)=i$ (it follows immediately from \cite[Proposition 7]{BritzEtAl2015a}, which by construction of $R'$ implies $h_{R'}(E)=i$ and $h_{R'}(E\dlAnd F)=i$. Hence it cannot be the case that $p\in\min_{\prec_{R'}}(E^{R'})$. Therefore, $R'$ must be a model of $\KB^\bullet$.

As a consequence, we end up with an interpretation $R'$ that is a ranked model of $\KB^\bullet$, and such that there is an object $o\in ( \bigsqcap \{ \neg E \dlOr F \mid E \usually F \in \D^\bullet\} \dlAnd  C)^{R'}$. Due to our construction of $R'$, the fact that $o$ satisfies $\bigsqcap \{ \neg E \dlOr F \mid E \usually F \in \D^\bullet\}$, that is, $H_0$, implies that $h_{R'}(o)=0$. Therefore, $R'$ is a ranked model of $\KB^\bullet$ that has an object satisfying $C$ in the lower layer. Hence $C$ cannot be exceptional w.r.t. $\KB^\bullet$, and we have a contradiction.
\end{proof}
\nd\textbf{Proposition \ref{classgci}}
\emph{A classical GCI $C\subs D$ is in the RC of $\KB= \tuple{\T,\D}$ iff $\T^\bullet \models C \subs D$, where $\T^\bullet$ has been computed using 
{\bf RC.Step~1}.}\qed
\begin{proof}
$C\subs D$ is in the RC of $\KB= \tuple{\T,\D}$ iff also $ C\dlAnd \neg D\subs\bot$ is in it, that is, iff $C\dlAnd \neg D\usually\bot$ is in the RC of $\KB= \tuple{\T,\D}$.

By {\bf RC.Step~2},  $C\dlAnd \neg D\usually\bot$ is in the RC of $\KB= \tuple{\T,\D}$ iff one of two following conditions are satisfied:

\begin{itemize}
\item there is a concept $H_i$, as defined in {\bf RC.Step 2} in Section \ref{alc}, such that
\begin{eqnarray*}
\T^\bullet  & \not \models & H_i \subs \neg (C\dlAnd \neg D)\text{ and }\\
\T^\bullet & \models & (C \dlAnd \neg D) \dlAnd H_i \subs \bot,
\end{eqnarray*}
that is an impossible occurrence;
\item $H_{C\dlAnd \neg D} = \top$ and, thus, $\T^\bullet \models (C\dlAnd \neg D) \subs \bot$, and this must be the case.
\end{itemize} 
\end{proof}
\nd The proofs of Propositions~\ref{exceptB}, \ref{propcr}, \ref{exceptB2}, and~\ref{lemeq} are immediate once we prove the following lemma.

\begin{lemma}\label{lemma_main}
Consider any DL that is closed under boolean operators. Given a finite set of defeasible axioms $\D$ and a concept $C$, the following two subsumption tests are equivalent:
\begin{eqnarray}
&& \T  \entails  \bigsqcap \{ \neg E \dlOr F \mid E \usually F \in \D\} \subs \neg C \label{b1} \\
&& \T \cup \{E \dlAnd \delta_\D \subs F \mid E \usually F \in\D\}  \entails   C \dlAnd \delta_\D \subs\bot \ . \label{b2}
\end{eqnarray}
where $\delta_\D$ is a new atomic concept.
\qed
\end{lemma}
\begin{proof}
\nd It is easily verified that (\ref{b1}) is equivalent to 
\begin{eqnarray}
&& \T  \entails C \dlAnd  \bigsqcap_{E \usually F \in \D} ( \neg E \dlOr F  ) \subs \bot \ . \label{bb1}
\end{eqnarray}
\nd Therefore,  it suffices to show the equivalence among (\ref{bb1}) and (\ref{b2}).

So, assume (\ref{bb1}) holds and assume to the contrary that (\ref{b2}) does not hold. Then there is an interpretation $\I$ such that
$\I \models \T$, $\I \models E \dlAnd \delta_\D \subs F$, for all $E \usually F \in \D$ and $o \in (C \dlAnd \delta_\D)^\I$ for some $o \in \Delta^\I$, \ie, $o \in C^\I$ and $o \in \delta_\D^\I$. But, 
$\I \models E \dlAnd \delta_\D \subs F$ means that  $\Delta^\I = (\neg E \dlOr \neg \delta_\D \dlOr F)^\I$. As $o \in \delta_\D^\I$, $o \in (\neg E \dlOr  F)^\I$ follows. It follows that 
\begin{equation} \label{eqo}
o \in (C \dlAnd  \bigsqcap_{E \usually F \in \D} ( \neg E \dlOr F  ))^\I \ ,
\end{equation} 
\nd which is absurd as (\ref{bb1}) holds by assumption

Conversely, assume that (\ref{b2}) holds and assume to the contrary that (\ref{bb1}) does not hold. Therefore, there is an interpretation $\I$ such that
$\I \models \T$ and (\ref{eqo}) holds for some $o \in \Delta^\I$. Now, extend $\I$ to $\delta_\D$ by setting
\[
\delta^\I_\D =  (\bigsqcap_{E \usually F \in \D} ( \neg E \dlOr F  ))^\I \ .
\]
\nd Note that $o \in \delta^\I_\D$. Then, by construction $\I \models E \dlAnd \delta_\D \subs F$, for all $E \usually F \in \D$ holds. 
Therefore, $\I$ is a model of the antecedent in (\ref{b2}) with $o \in (C \dlAnd \delta_\D)^\I$, which is absurd as (\ref{b2}) holds by assumption.
\end{proof}
 
 \nd\textbf{Proposition \ref{exceptB}}
\emph{For every concept $C$ and every ontology $\KB=\tuple{\T,\D}$,~if
\begin{equation}
\T_{\delta_\D}\entails C\dlAnd \delta_\D \subs\bot \ ,
\end{equation}
\nd where $\delta_\D$ is a new atomic concept, then $C$ is exceptional \wrt~$\KB$.
}\qed
%
\begin{proof}
It is an immediate consequence of Lemma \ref{lemma_main} and Proposition \ref{except}. 
\end{proof}
\nd\textbf{Proposition \ref{propcr}}
\emph{Consider an ontology $\KB=\tuple{\T,\D}$. Then $\mathtt{ComputeRanking}(\KB)$ returns the ontology $\tuple{\T^{*}, \D^{*}}$, where $\D^*$ is partitioned into a sequence $\D_0,\ldots,\D_n$, where $\T^*$, $\D^*$ and all $\D_i$ are equal to the sets $\T^\bullet$, $\D^\bullet$ and $\D^\bullet_0,\ldots,\D^\bullet_n$ obtained via {\bf RC.Step 1}.}\qed
\begin{proof}
By Lemma \ref{lemma_main}, $e(\T',\E')$ is the same set retuned by $\mathtt{Exceptional}(\T',\E')$, for any $\T'$ and $\E'$.
Now, the only difference between  {\bf RC.Step 1} and  $\mathtt{ComputeRanking}$ procedures is in the way to compute the exceptional concepts. But, since these are the same, 
$\mathtt{ComputeRanking}(\KB)$ returns the ontology $\tuple{\T^{*}, \D^{*}}$, where $\D^*$ is partitioned into a sequence $\D_0,\ldots,\D_n$, where $\T^*$, $\D^*$ and all $\D_i$ are equal to the sets $\T^\bullet$, $\D^\bullet$ and $\D^\bullet_0,\ldots,\D^\bullet_n$ obtained via {\bf RC.Step 1}. 
\end{proof}
\nd\textbf{Proposition \ref{exceptB2}}
\emph{Given an ontology $\KB^*=\tuple{\T^*,\D^*}$, obtained from the application of the procedure $\mathtt{ComputeRanking}$ to an ontology $\KB=\tuple{\T,\D}$, for every concept $C$,
\[\T^*_{\delta_{\D^*}}\entails C\dlAnd \delta_{\D^*} \subs\bot \ , \]
if and only if  $C$ is exceptional \wrt~$\KB^*$.
}\qed
\begin{proof}
It is an immediate consequence of Lemma \ref{lemma_main}, Proposition \ref{except2} and Proposition \ref{propcr}. 
\end{proof}

\nd\textbf{Proposition \ref{lemeq}}
\emph{
By referring to Remark~\ref{entrc}, the subsumption test~(\ref{rct1}) (resp.~\ref{rct2}) is equivalent to the subsumption test~(\ref{rct3}) (resp.~\ref{rct4}).}
\qed
\begin{proof}
As in Sections \ref{alc} and \ref{ratclosEL}, $\E_i  =  \D_i \cup \ldots \cup \D_n$ and $H_i=\bigsqcap\{\neg E\dlOr F\mid E\usually F\in \E_i\}$. We have to show that both the two subsumption tests
\begin{eqnarray}
&& \T^* \not \models  H_i 
\subs \neg C \label{rct1A} \\
&& \T^* \cup\{E \dlAnd\delta_i\subs F \mid E\usually F \in \E_i \} \not\entails  C \dlAnd\delta_i\subs \bot   \label{rct3A}
\end{eqnarray}
\nd are equivalent, and that
\begin{eqnarray}
&&\T^*  \models  C \dlAnd  H_i
\subs D  \label{rct2A} \\
&& \T^* \cup\{E \dlAnd\delta_i\subs F \mid E \usually F \in \E_i  \} \entails  C\dlAnd\delta_i\subs D  \label{rct4A}
\end{eqnarray}
\nd are equivalent. The proof is essentially the same as for Lemma~\ref{lemma_main}. For illustrative purpose,
we show that (\ref{rct1A}) and (\ref{rct3A}) are equivalent. The other case can be shown in a similar way.  

Assume (\ref{rct1A}) holds. Let us show that (\ref{rct3A}) holds as well. By (\ref{rct1A}), there is a model $\I$ of $\T^*$ with 
\begin{equation} \label{eqoA}
o \in (C \dlAnd  \bigsqcap_{E {\tiny \usually} F \in \E_i} ( \neg E \dlOr F  ))^\I \ ,
\end{equation} 
\nd for some  $o \in \Delta^\I$. Now, let us extend $\I$ to $\delta_i$ by defining
\[
\delta_i^\I =  (\bigsqcap_{E {\tiny \usually} F \in \E_i} ( \neg E \dlOr F  ))^\I \ .
\]
\nd Note that by definition $o \in C^\I$ and $o \in \delta_i^\I$.
By construction, $\I \models E \dlAnd \delta_i \subs F$, for all $E \usually F \in \E_i$.
Therefore, $\I$ is a model of the antecedent in (\ref{rct3A}) with $o \in (C \dlAnd \delta_i)^\I$ and, thus, (\ref{rct3A}) holds.

Vice-versa, assume (\ref{rct3A}) holds. Let us show that (\ref{rct1A}) holds as well. By (\ref{rct3A}), there is a model $\I$ of $\T^*$ such that
$\I \models E \dlAnd \delta_i \subs F$, for all $E \usually F \in \E_i$  and $o \in (C \dlAnd \delta_i)^\I$ for some $o \in \Delta^\I$, \ie, $o \in C^\I$ and $o \in \delta_i^\I$. 
But,  $\I \models E \dlAnd \delta_i \subs F$ means that  $\Delta^\I = (\neg E \dlOr \neg \delta_i \dlOr F)^\I$. As $o \in \delta_i^\I$, $o \in (\neg E \dlOr  F)^\I$ follows. Therefore,
\begin{equation} \label{eqoB}
o \in (\bigsqcap_{E \usually F \in \E_i} ( \neg E \dlOr F  ))^\I 
\end{equation} 
\nd with $o \in C^\I$ and, thus,  (\ref{rct1A}) holds. 
\end{proof}
%
%
 \nd\textbf{Proposition \ref{normalformrank}}
\emph{
Given an ontology $\KB=\tuple{\T,\D}$, $C \usually D \in \D$ and the corresponding ontology in normal form $\KB'=\tuple{\T',\D'}$,  then $C$ is exceptional \wrt~$\KB$~iff $A_C$ is exceptional \wrt~$\KB'$, where $A_C$ is the new atomic concept introduced by the normalisation procedure to replace $C \usually D$ with $A_C \usually A_D$.}\qed
 %
 \begin{proof}
By Proposition~\ref{exceptB} we have that $C$ is exceptional \wrt~$\KB$~iff
$\T_{\delta_\D}\entails C\dlAnd \delta_\D \subs\bot$, where $\delta_\D$ is a new atomic concept. From this it suffices to show that 
$\T_{\delta_\D}\entails C\dlAnd \delta_\D \subs\bot$ iff $\T'_{\delta_\D}\entails A_C\dlAnd \delta_\D \subs\bot$.

So, let $\tuple{\T^+,\D^+}$ the KB obtained adding to $\T$ the axioms $\A_E= E$   and $\A_F= F$ for every axiom $E\usually F\in\D$ ($A_E, A_F$ being new atomic concepts), while $\D^+$ is obtained by replacing in $\D$ every axiom $E\usually F$ with $A_E\usually A_F$. Now, let us prove that $\T_{\delta_\D}\entails C\dlAnd\delta_\D\subs\bot$ iff $\T^+_{\delta_\D} \entails A_C\dlAnd\delta_\D\subs\bot$. 

The proof from left to right is immediate: assume $\I \models \T^+_{\delta_\D}$. Then $\I\models A_C  = C$, $\I \models \T_{\delta_\D}$ and, thus, $\I \models C\dlAnd\delta_\D\subs\bot$ and $\I \models A_C\dlAnd\delta_\D\subs\bot$.

From right to left, if $\T_{\delta_\D}\not\entails C\dlAnd\delta_\D\subs\bot$ then there is a model $\I$ of $\T_{\delta_\D}$ (that interprets only the  concepts appearing in $\T_{\delta_\D}$) \st~$(C\dlAnd\delta_0)^\I\neq\emptyset$. Now, extend the interpretation function $\cdot^\I$ of $\I$ in such a way that $A_E^\I=E^\I$ for all $E \usually F \in \D$; Then $\I\models \T^+_{\delta_\D}$ with $(A_C \dlAnd\delta_\D)^\I\neq\emptyset$ and, thus,
$\T^+_{\delta_\D}\not\entails A_C \dlAnd\delta_\D\subs\bot$.

Since $\T'_{\delta_\D}$ is the normal form of $\T^+_{\delta_\D}$, that preserves satisfaction, it follows that for every $C\usually D\in \D$, $\T_{\delta_\D} \entails C\dlAnd\delta_0\subs\bot$ iff $\T'_{\delta_\D} \entails A_C\dlAnd\delta_0\subs\bot$, which concludes the proof.
\end{proof}

 \nd\textbf{Proposition \ref{RC_syntactic_independent}}
\emph{Let $\KB=\tuple{\T,\D}$ and $\KB'=\tuple{\T',\D'}$ be rank equivalent. for every $\ELbot$ defeasible GCI $C\usually D$, $C\usually D$ is in the RC of $\KB$ iff $C\usually D$ is in the RC of $\KB'$.} \qed

\begin{proof}
In the following, we refer to the semantic construction in \ref{canmod}. Now, $\KB=\tuple{\T,\D}$ and $\KB'=\tuple{\T',\D'}$ being rank equivalent means that, for every ranked interpretation $R$, $R$ is a model of $\KB$ iff it is a model of $\KB'$ (see Section \ref{prel}). In turn, this implies that, given a countably infinite domain $\Delta$, the set of models of $\KB$ having $\Delta$ as domain corresponds to the  set of models of $\KB'$ having $\Delta$ as domain ($\R^\KB_\Delta=\R^{\KB'}_\Delta$). The latter implies that the characteristic model of the RC of $\KB$, $\ru$, is also a characteristic model of the RC of $\KB'$, which suffices to conclude. 
\end{proof}


\nd In order to prove Proposition \ref{ranking_defnom}, we need the following lemma.
\begin{lemma}\label{exceptional_defnom}
An n-safe $\mathcal{EL{O}}_\bot$ concept $C$ is exceptional \wrt~a safe $\mathcal{EL{O}}_\bot$ ontology $\KB=\tuple{\T,\D}$ iff the n-safe concept $C_{\tiny \defnom{}}$ is exceptional \wrt~the nominal safe $\mathcal{EL\defnom{O}}_\bot$ ontology $\KB_{\tiny \defnom{}}=\tuple{\T_{\tiny \defnom{}},\D_{\tiny \defnom{}}}$.
\qed
\end{lemma}
\begin{proof}
Assume $C$ is not exceptional \wrt~$\KB$. So, there is a ranked model $R$ of $\KB$ \st~$C^{R}\cap\min_{\prec^R}(\Delta^{R}) \neq \emptyset$. Now, extend the interpretation function $\cdot^R$ of $R$ imposing that for every $a\in\ON$, $\defnom{a}^R=\{a\}^R$; then $R$ is model of $\KB_{\tiny \defnom{}}$ and $C_{\tiny \defnom{}}$ is not exceptional, \ie~$C_{\tiny \defnom{}}^{R}\cap\min_{\prec^R}(\Delta^{R}) \neq \emptyset$.

For the other direction, assume $C_{\tiny \defnom{}}$ is not exceptional \wrt~$\KB_{\tiny \defnom{}}$. So,  there is a model $R$ of $\KB_{\tiny \defnom{}}$ 
\st~$C_{\tiny \defnom{}}^{R}\cap\min_{\prec^R}(\Delta^{R}) \neq\emptyset$. Now, we build a ranked interpretation $R'=\tuple{\Delta^{R'}, \cdot^{R'},\prec^{R'}}$ in the following way:

\begin{itemize}
\item $\Delta^{R'}=\{x_{\typ{D}},x_D\mid R\not\models D\subs\bot, \mbox{ where $D$ is a safe $\mathcal{EL\defnom{O}}_\bot$ concept}\}\cup\{x_{\defnom{a}}\mid a\in \ON\}$;
\item $\defnom{a}^{R'}=\{x_{\defnom{a}}\in \Delta^{R'}\}$, for every defeasible nominal $\defnom{a}$; 
\item $A^{R'}=\{x_D\in \Delta^{R'}\mid R\models D\subs A\}\cup\{x_{\typ{D}}\in \Delta^{R'}\mid R\models D\usually A\}\cup\{x_{\defnom{a}}\in \Delta^{R'}\mid R\models \defnom{a}\usually A\}$, for every  safe $\mathcal{EL\defnom{O}}_\bot$ concept $D$, every defeasible nominal $\defnom{a}$, and every atomic $\mathcal{EL_\bot}$ concept $A$; 
\item $r^{R'}=\{\tuple{x_D,x_E}\in \Delta^{R'}\times\Delta^{R'}\mid R\models D\subs \exists r. E\}\cup\{\tuple{x_{\typ{D}},x_E}\in \Delta^{R'}\times\Delta^{R'}\mid R\models D\usually \exists r. E\}\cup \{\tuple{x_{\defnom{a}},x_E}\in \Delta^{R'}\times\Delta^{R'}\mid R\models \defnom{a}\usually \exists r. E\}$, for every  safe $\mathcal{EL\defnom{O}}_\bot$ concepts $D$ and $E$, and every defeasible nominal $\defnom{a}$;
\item $h_{R'}(x_{\typ{D}})=h_R(D)$, for every $x_{\typ{D}}\in\Delta^{R'}$;
\item $h_{R'}(x_{\defnom{a}})=h_R(\defnom{a})$, for every defeasible nominal $\defnom{a}$; 
\item $x_{\typ{D}}\prec^{R'}x_{E}$ for every $x_{\typ{D}},x_E\in\Delta^{R'}$.
\end{itemize}

\emph{Note:} using $x_{\typ{D}}$ we want to introduce an object that represents a typical instance of a concept $D$, while with $x_{{D}}$ we want to represent an atypical instance of a concept $D$. Only for defeasible nominals we introduce a single object representing them. What we need to impose in the interpretation $R'$ is that  every object $x_{\typ{D}}$ is more typical (is positioned in a lower layer) than any atypical objects $x_{{E}}$ satisfying the concept $D$; such a condition is easily satisfied if we impose the last constraint in the above definition of $R'$: for every pair of concepts $D$ and $E$, $x_{\typ{D}}\prec^{R'}x_{E}$; that is, we move all the objects $x_{E}$, representing an atypical occurrence of some concept $E$, to the upper layer. 


As next, we prove that $R'$ is still a model of $\KB_{\tiny \defnom{}}$. The first step is to prove that for every pair of concepts $D,E$ where $D$  and $E$ are \emph{safe}, we have, 
\begin{enumerate}
\item $R\models D\subs E$  and $R \not \models D \subs \bot$ iff $x_D\in E^{R'}$;
\item $R\models D\usually E$ and $R \not \models D \subs \bot$ iff $x_{\typ{D}}\in E^{R'}$.
\end{enumerate}
\nd And, for every pair  $\defnom{a},E$ where $a\in\ON$  and $E$ is \emph{safe}, we have, 
\begin{enumerate}
\item[3.] $R\models \defnom{a}\usually E$ iff $x_{\defnom{a}}\in E^{R'}$.
\end{enumerate}
\nd The proofs are by induction on the construction of $E$. The only relevant steps are  $E=\exists r.F$ (case $a$) and $E=\exists r.\defnom{b}$ (case $b$).
\begin{enumerate}
\item 

	\begin{itemize}
	\item[$a.$] If $R\models D\subs \exists r.F$ and $R \not \models D \subs \bot$ then 
	$\tuple{x_D,x_{F}} \in r^{R'}$.
	 By induction hypothesis, $x_F\in F^{R'}$, hence $x_D\in (\exists r.F)^{R'}$. 
	Vice-versa, if $x_D\in (\exists r.F)^{R'}$, then  
	$\tuple{x_D,x_{G}} \in r^{R'}$
	for some $x_{G}\in F^{R'}$. By construction, $R \not \models D \subs \bot$, $R\models D\subs \exists r.G$, and
	by induction hypothesis, $R\models G\subs F$, hence $R\models D\subs \exists r.F$.

	\item[$b.$] If $R\models D\subs \exists r.\defnom{b}$ and $R \not \models D \subs \bot$ then 
	$\tuple{x_D, x_{\defnom{b}} } \in r^{R'}$. 
	Clearly $x_{\defnom{b}}\in (\defnom{b})^{R'}$, and $x_D\in (\exists r.{\defnom{b}})^{R'}$. 
	Vice-versa, if $x_D\in (\exists r.\defnom{b})^{R'}$, then  
	$\tuple{x_D, x_{\defnom{b}} } \in r^{R'}$ 
	 and, by construction, $R\models D\subs \exists r.\defnom{b}$ and $R \not \models D \subs \bot$.

	\end{itemize}
\item 
	\begin{itemize}
	\item[$a.$] If $R\models D\usually \exists r.F$ and $R \not \models D \subs \bot$ then 
	$\tuple{x_{\typ{D}}, x_{F} } \in r^{R'}$. 	
	By induction point 1, $x_F\in F^{R'}$, and consequently $x_D\in (\exists r.F)^{R'}$. 
	Vice-versa, if $x_{\typ{D}}\in (\exists r.F)^{R'}$, then  
	$\tuple{x_{\typ{D}}, x_{G} } \in r^{R'}$ 
	for some $x_{G}\in F^{R'}$. By construction, $R \not \models D \subs \bot$, $R\models D\usually \exists r.G$, and by point 1, $R\models G\subs F$, hence $R\models D\usually \exists r.F$.

	\item[$b.$] If $R\models C\usually \exists r.\defnom{b}$ and $R \not \models D \subs \bot$ then 
	$\tuple{x_{\typ{D}}, x_{\defnom{b}} } \in r^{R'}$. 
	By construction, $x_{\defnom{b}}\in (\defnom{b})^{R'}$, hence $x_D\in (\exists r.\defnom{b})^{R'}$. 
	Vice-versa, if $x_{\typ{D}}\in (\exists r.\defnom{b})^{R'}$, then  
	$\tuple{x_{\typ{D}}, x_{\defnom{b}} } \in r^{R'}$,	
	and, by construction, $R\models D\usually \exists r.\defnom{b}$ and $R \not \models D \subs \bot$.
	\end{itemize}
	
\item Analogously to point 2.

\end{enumerate}
Now we prove that $R'$ satisfies the safe  $\mathcal{EL\defnom{O}}_\bot$ ontology $\tuple{\T_{\tiny \defnom{}},\D_{\tiny \defnom{}}}$:
\begin{itemize}
\item Case $D\subs E\in \T_{\tiny \defnom{}}$.
If $x_F\in D^{R'}$, with $F$ a safe concept, by  point 1, $R\models F\subs D$, that, combined with $R\models D\subs E$, implies $R\models F\subs E$, and, again by point 1, $x_F\in E^{R'}$.
If $x_{\typ{F}}\in D^{R'}$, with $F$ a safe concept, by  point 2, $R\models F\usually D$, that, combined with $R\models D\subs E$, implies $R\models F\usually E$, and, again by point 2, $x_F\in E^{R'}$.
Eventually, if $x_{\defnom{a}}\in D^{R'}$, by  point 3, $R\models \defnom{a}\usually D$, that, combined with $R\models D\subs E$, implies $R\models \defnom{a}\usually E$, and, again by point 3, $x_{\defnom{a}}\in E^{R'}$.

\item Case $D\usually E\in \D_{\tiny \defnom{}}$.
By construction of $R'$,  for any $F$, it cannot be $x_F\in \min_{\prec^{R'}}(D^{R'})$: only an object of kind $x_{\typ{F}}$ or $x_{\defnom{a}}$ can be in $\min_{\prec^{R'}}(D^{R'})$. 
If $x_{\typ{F}}\in \min_{\prec^{R'}}(D^{R'})$ then $h_{R'}(x_{\typ{F}})=h_{R'}(x_{\typ{D}})$, and by construction $h_R(F)=h_R(D)$. Also, by point 2 $R\models F\usually D$. It can then be verified that $h_R(F)=h_R(D)$, $R\models F\usually D$ and $R\models D\usually E$ together imply that $R\models F\usually E$, that by point 2 implies $x_{\typ{F}}\in  E^{R'}$.
Analogously, if  $x_{\defnom{a}}\in \min_{\prec^{R'}}(D^{R'})$ then $h_{R'}(x_{\defnom{a}})=h_{R'}(x_{\typ{D}})$, and by construction $h_R(\defnom{a})=h_R(D)$. Also, by point 3 $R\models \defnom{a}\usually D$. Then it can be verified that $h_R(\defnom{a})=h_R(D)$, $R\models \defnom{a}\usually D$ and $R\models D\usually E$ together imply that $R\models \defnom{a}\usually E$, that by point 3 implies $x_{\defnom{a}}\in  E^{R'}$.
\end{itemize}

\nd Therefore, the ranked interpretation $R'$ is a model of $\tuple{\T_{\tiny \defnom{}},\D_{\tiny \defnom{}}}$. Note that since we imposed that $C_{\tiny \defnom{}}$ is not exceptional in $R$, in $R'$ we have the object $x_{C_{\tiny \defnom{}}}$ in layer $0$, with $x_{C_{\tiny \defnom{}}}\in (C_{\tiny \defnom{}})^{R'}$; hence $C_{\tiny \defnom{}}$ is not exceptional also in $R'$.

Now we can trivially extend the interpretation $R'$ to classical nominals:
\begin{itemize}
\item $\{a\}^{R'}=\defnom{a}^{R'}=\{x_{\defnom{a}}\}$ for every $a\in \ON$.
\end{itemize}
It can easily be proved now that $R'$ is also a model of $\KB =\tuple{\T,\D}$. To prove it, it is sufficient to check that for every \emph{n-safe} concept $D$ and every object $x$, 
$x \in D^{R'}$ iff $x \in (D_{\{\}})^{R'}$. This can be proved by induction on the construction of the concept $D$; the only relevant steps are
\begin{itemize}
\item $D=\defnom{a}$. By construction, $x\in \defnom{a}^{R'}$ iff $x\in \{a\}^{R'}$.
\item $D=\exists r.\defnom{a}$. $x\in (\exists r.\defnom{a})^{R'}$ iff 
$\tuple{x, x_{\defnom{a}} } \in r^{R'}$ 
iff $x\in (\exists r.\{a\})^{R'}$.
\end{itemize}
Hence $R'$ is a model of $\KB$. Since in Layer $0$ we have $x_{C_{\tiny \defnom{}}}$, and $x_{C_{\tiny \defnom{}}}\in (C_{\tiny \defnom{}})^{R'}$, we  also have $x_{C_{\tiny \defnom{}}}\in (C)^{R'}$. That is, 
$C^{R'}\cap\min_{\prec^{R'}}(\Delta^{R'}) \neq \emptyset$, \ie~$C$ is not exceptional in \wrt~$\KB$.  
\end{proof}

\nd Now we can easily prove Proposition \ref{ranking_defnom}.

\vspace{2ex}
\nd {\bf Proposition \ref{ranking_defnom}.}
\emph{For every  n-safe $\mathcal{EL{O}}_\bot$ concept $C$ and every nominal safe $\mathcal{EL{O}}_\bot$ ontology $\KB=\tuple{\T,\D}$, $r_\KB(C)=i$ iff $r_{\KB_{\tiny \defnom{}}}(C_{\tiny \defnom{}})=i$. }\qed
%
\begin{proof}
Consider a nominal safe $\mathcal{EL{O}}_\bot$ ontology $\KB=\tuple{\T,\D}$ and the nominal safe $\mathcal{EL\defnom{O}}_\bot$ ontology $\KB_{\tiny \defnom{}}$. Lemma \ref{exceptional_defnom} tells us that for every safe $\mathcal{EL{O}}_\bot$ GCI $C\usually D$, $C\usually D$ is exceptional \wrt~$\KB$ iff $C_{\tiny \defnom{}}\usually D_{\tiny \defnom{}}$ is exceptional \wrt~$\KB_{\tiny \defnom{}}$. Then we can easily prove the proposition by induction on the construction of the ranking. 
\end{proof}
\nd Next, we address Proposition~\ref{ranking_defnomA}.

\vspace{2ex}
\nd {\bf Proposition~\ref{ranking_defnomA}.}
\emph{For every  n-safe $\defELObot$ concept $C$ and every nominal safe $\defELObot$ ontology $\KB=\tuple{\T,\D}$ the following holds:
$r_{\KB}(C)=i$ iff $r_{N(\KB_{\tiny \{\}})}(N(C_{\tiny \{\}}))=i$. } \qed
\begin{proof}
The proof of this proposition is quite immediate. Assume that there is a ranked interpretation $R$ that is a model of $\KB$ and such that there is an object in layer $0$ satisfying $C$. Now define an $\mathcal{EL_\bot}$ ranked interpretation $R'$ that is identical to $R$, by imposing that for every atomic concept $N_a$, ${N_a}^{R'}=\defnom{a}^R$. It is easy to check by induction on the construction of a concept $D$ that, for every object $o\in\Delta^R$, $o\in D^R$ iff $o\in (N(D_{\{\}}))^{R'}$. Hence $R'$ is a model of $N(\KB_{\{\}})$ and $N(C_{\{\}})$ is not exceptional in $R'$. Repeating exactly the same transformation and reasoning in the other direction, we can conclude that $C$ is exceptional in $\KB$ iff $N(C_{\{\}})$ is exceptional in $N(\KB_{\{\}})$. This property can be immediately extended to the definition of the rankings (Section \ref{alc}), proving the proposition. 
\end{proof}

\nd {\bf Corollary~\ref{use_concept_coroll}.}
\emph{Let $\KB=\tuple{\T,\D}$ be a nominal safe $\defELObot$ ontology, with $C\usually D\in \D$. Let  $\T^\bullet$ be  the TBox 
 obtained from $\KB$ after applying {\bf RC.Step 1}, while let $(N(\T_{\{\}}))^*$ be the TBox obtained from $N(\KB_{\{\}})$ using the
$\mathtt{ComputeRanking}$ procedure. Then, the following are equivalent:}
\begin{enumerate}
\item   $C\usually \bot\in \T^\bullet$; 
\item $N(C_{\{\}}\usually \bot)\in (N(\T_{\{\}}))^*$; 
\item $r_{\KB_{\{\}}}(C_{\{\}}\usually D_{\{\}})=\infty$.
\qed
\end{enumerate}
\begin{proof}
The extension of Proposition~\ref{propcr} from $\mathcal{EL_\bot}$ ontologies to $\mathcal{EL\defnom{O}_\bot}$ ontologies is immediate: since defeasible nominals behave exactly like atomic concepts, it is sufficient to go through the proof of Proposition~\ref{propcr} by taking into account also defeasible nominals. 
This, combined with Proposition~\ref{ranking_defnomA}, guarantees that the application of the above-mentioned procedures gives back the same results for an $\mathcal{EL\defnom{O}_\bot}$ ontology $\KB$ and the correspondent $\mathcal{EL{O}_\bot}$ ontology $N(\KB_{\{\}})$.  
\end{proof}

\nd The following Lemma is needed to prove Proposition~\ref{use_concept}. 

\begin{lemma}\label{strict_nominals}
Let $\KB=\tuple{\T,\D}$ be a nominal safe $\defELObot$ ontology. Then
$N(\T_{\{\}}\cup\D^{\subs}_{\{\},\infty})$ is logically equivalent to $(N(\T_{\{\}}))^*$ (where $(N(\T_{\{\}}))^*$ is defined as in Corollary~\ref{use_concept_coroll}).
\qed
\end{lemma}
\begin{proof}
From Proposition~\ref{ranking_defnomA} we know that $C\usually D\in \D_{\{\},\infty}$ iff $N(C\usually D)\in (N(\D_{\{\}}))_\infty$ (that is, the axioms in $\D_{\{\}}$ that have infinite rank \wrt~$\KB_{\{\}}$ correspond to the axioms in $N(\D_{\{\}})$ that have infinite rank \wrt~$N(\KB_{\{\}})$), hence $N(\T_{\{\}}\cup \D^\subs_{\{\},\infty})$ is the same as to $N(\T_{\{\}})\cup (N(\D_{\{\}}))^\subs_\infty$. 

Now, we have to prove that $N(\T_{\{\}})\cup (N(\D_{\{\}}))^\subs_\infty$ is logically equivalent to $(N(\T_{\{\}}))^*$. To this purpose, it is sufficient to prove that every axiom in $(N(\D_{\{\}}))^\subs_\infty$ is derivable from $(N(\T_{\{\}}))^*$ and every axiom in $(N(\T_{\{\}}))^*\setminus N(\T_{\{\}})$ is derivable from $N(\T_{\{\}})\cup (N(\D_{\{\}}))^\subs_\infty$.

So, assume  $C\subs D\in (N(\D_{\{\}}))^\subs_\infty$. Then $C\usually D\in (N(\D_{\{\}}))_\infty$. By  procedure $\mathtt{ComputeRanking}$, an axiom $C\subs \bot$ has been added to $(N(\T_{\{\}}))^*\setminus N(\T_{\{\}})$ (line 13 of the procedure).  If $C\subs \bot \in (N(\T_{\{\}}))^*$, then $(N(\T_{\{\}}))^*\models C\subs D$.

Vice-versa, assume $C\subs \bot\in (N(\T_{\{\}}))^*\setminus N(\T_{\{\}})$ (due to line 13 of procedure $\mathtt{ComputeRanking}$, only axioms with the  form $C\subs \bot$ can be in $(N(\T_{\{\}}))^*\setminus N(\T_{\{\}})$). This implies that an axiom $C\usually D$ is in $(N(\D_{\{\}}))_\infty$ (due to the construction of $N(\T_{\{\}}))^*$ in the line 13 of  procedure $\mathtt{ComputeRanking}$), and so $C\subs D\in (N(\D_{\{\}}))^\subs_\infty$. By Proposition \ref{propcr}, 
we know that $C\usually D\in (N(\D_{\{\}}))_\infty$ iff it is in the fixed point of the exceptionality procedure: that is, $C\usually D\in (N(\D_{\{\}}))$ and
\[
(N(\T_{\{\}}))\models \bigsqcap\{\neg E\dlOr F\mid E\usually F\in (N(\D_{\{\}}))_\infty\}\subs \neg C
\]
\nd Since 
\[
N(\T_{\{\}})\cup (N(\D_{\{\}}))^\subs_\infty\models \top\subs \bigsqcap\{\neg E\dlOr F\mid E\usually F\in (N(\D_{\{\}}))_\infty\}
\]
we end up with $N(\T_{\{\}})\cup (N(\D_{\{\}}))^\subs_\infty\models \top\subs \neg C$, that is,
\[
N(\T_{\{\}})\cup (N(\D_{\{\}}))^\subs_\infty\models C\subs \bot
\]
for every $C\subs \bot\in (N(\T_{\{\}}))^*\setminus N(\T_{\{\}})$, which concludes. 
\end{proof}

\nd {\bf Proposition~\ref{use_concept}.}
\emph{Let $\KB=\tuple{\T,\D}$ be a nominal safe $\defELObot$ ontology. 
For every \emph{safe}  GCI $C\subs D$, we have that the following are equivalent:}
\begin{enumerate}
\item  $C\subs D$ is in the RC of $\KB$; 
\item $N(C_{\{\}} \subs D_{\{\}})$ is in the RC of $N(\KB_{\{\}})$;
\item ${\T}_{\{\}}\cup\D^\subs_{\{\},\infty}\entails C_{\{\}}\subs D_{\{\}}$.
\qed
\end{enumerate}
\begin{proof}
That $C\subs D$ is in the RC of $\KB$ iff
$N(C_{\{\}} \subs D_{\{\}})$ is in the RC of $N(\KB_{\{\}})$ follows easily from Proposition~\ref{ranking_defnomA} and from the fact that $\mathtt{RationalClosure}$ is invariant to the substitution of every $\defnom{a}$ with an atomic concept $N_a$ (it can be proven semantically in a similar way as in the proof of Proposition~\ref{ranking_defnomA}). From Proposition~\ref{classgci} we know that $N(C_{\{\}} \subs D_{\{\}})$ is in the RC of $N(\KB_{\{\}})$ iff $(N(\T_{\{\}}))^*\entails N(C_{\{\}} \subs D_{\{\}})$, that, by Propositions \ref{propelo} and \ref{propeloB}  and Lemma \ref{strict_nominals} holds iff ${\T}_{\{\}}\cup\D^\subs_{\{\},\infty}\entails C_{\{\}} \subs D_{\{\}}$. 
\end{proof}

 \end{document}